\documentclass{article}

\RequirePackage{amsthm,amsmath,amsfonts,amssymb}
\RequirePackage{natbib}
\RequirePackage[colorlinks,citecolor=blue,urlcolor=blue]{hyperref}
\RequirePackage{graphicx}

\usepackage[utf8]{inputenc}
\usepackage[utf8]{inputenc}
\usepackage[T1]{fontenc}
\usepackage{graphicx}
\usepackage{mathtools}
\usepackage{booktabs}
\usepackage{bm}
\usepackage{tcolorbox}
\usepackage{multicol,multirow}
\usepackage{comment}

\newcommand{\Tcal}{\mathcal{T}}
\usepackage{booktabs}
\newcommand{\ie}{\textit{i.e.}}

\usepackage{array}
\newcommand{\RNum}[1]{\uppercase\expandafter{\romannumeral #1\relax}}
\usepackage{arydshln}

\newcommand{\Jcal}{\mathcal{J}}

\newcommand{\Eg}{\textit{E.g.}}
\usepackage{enumerate}   
\usepackage{amsmath}
\usepackage{multirow}
\usepackage{bbm}
\usepackage{tikz}

\definecolor{mygreen}{rgb}{0,0.6,0}
\usetikzlibrary{positioning,arrows,decorations,decorations.markings,shadows,positioning,arrows.meta,matrix,fit}

\usetikzlibrary{positioning,arrows}
\usepackage{wrapfig}
\DeclareMathOperator*{\argmax}{arg\,max}
\DeclareMathOperator*{\argmin}{arg\,min}

\usepackage{amsmath}
\usepackage{appendix}
\usepackage{amssymb}

\usepackage{amsthm}
\usepackage{hyperref}
\usepackage[capitalise]{cleveref}
\Crefname{assumption}{Assumption}{Assumptions}
\crefname{lemma}{lemma}{lemmas}
\Crefname{lemma}{Lemma}{Lemmas}

\makeatletter
\def\cref@getref#1#2{%
  \xdef\@lastusedlabel{#1}%
  \expandafter\let\expandafter#2\csname r@#1@cref\endcsname%
  \expandafter\expandafter\expandafter\def%
    \expandafter\expandafter\expandafter#2%
    \expandafter\expandafter\expandafter{%
      \expandafter\@firstoftwo#2}}%

\newcommand\enameref[1]{%
  \expandafter\ifx\csname r@#1\endcsname\relax
  ??\typeout{^^JLaTeX Warning: Reference #1 undefined on input line \the\inputlineno}%
  \else
    \expandafter\expandafter\expandafter\@thirdoffive\csname r@#1\endcsname
  \fi
}    

\Crefformat{lemma}{%
    \IfStrEq{\enameref{\@lastusedlabel}}{}{%
        \Cref@lemma@name~#1%
    }{%
        \nameref{\@lastusedlabel} (#1)%
    }%
}

\crefformat{lemma}{%
    \IfStrEq{\enameref{\@lastusedlabel}}{}{%
        \cref@lemma@name~#1%
    }{%
        \nameref{\@lastusedlabel} (#1)%
    }%
}
\makeatother

\usepackage{graphicx}

\newcommand{\rea}{realizability\,}

\newcommand{\comp}{completeness\,}

\def\Holder{{H\"{o}lder}}

\newcommand{\MDL}{\mathrm{mdl}}
\newcommand{\MIL}{\mathrm{mil}}

\newcommand{\bN}{\mathcal{N}}

\newcommand{\bigO}{\mathcal{O}} %

\newcommand{\ts}{\textstyle}

\newcommand{\hP}{\mathbb{P}_n}

\newcommand{\op}{\mathrm{o}_{p}}
\newcommand{\E}{\mathbb{E}}

\newcommand{\pa}{\mathrm{\pa}}

\newcommand{\rI}{\mathrm{I}}

\newcommand{\var}{\mathrm{var}}

\newcommand{\rE}{\mathrm{E}}

\newcommand{\rd}{\mathrm{d}}

\newcommand{\prns}[1]{\left(#1\right)}
\newcommand{\braces}[1]{\left\{#1\right\}}
\newcommand{\bracks}[1]{\left[#1\right]}

\newcommand{\abs}[1]{\left|#1\right|}
\newcommand{\Rl}{\mathbb{R}}

\newcommand{\epol}{\pi^\mathrm{e}}
\newcommand{\bpol}{\pi^{\mathrm{b}}}

\newcommand{\Fcal}{\mathcal{F}}

\newcommand{\Scal}{\mathcal{S}}
\newcommand{\Acal}{\mathcal{A}}
\newcommand{\Bcal}{\mathcal{B}}

\newcommand{\Rbbb}{\mathbb{R}}
\newcommand{\Wbbb}{\mathbb{W}}
\newcommand{\Qbbb}{\mathbb{Q}}
\newcommand{\Vbbb}{\mathbb{V}}

\newcommand{\Xcal}{\mathcal{X}}
\newcommand{\Wcal}{\mathcal{W}}
\newcommand{\RR}{\mathbb{R}}

\newcommand{\Dcal}{\mathcal{D}}
\newcommand{\Gcal}{\mathcal{G}}
\newcommand{\Rcal}{\mathcal{R}}

\newcommand{\Ncal}{\mathcal{N}}
\newcommand{\Error}{\mathrm{Err}}

\renewcommand{\hat}{\widehat}
\renewcommand{\tilde}{\widetilde}

\renewcommand{\eqref}[1]{(\ref{#1})}

\newcommand{\RN}[1]{%
  \textup{\uppercase\expandafter{\romannumeral#1}}%
}

\usepackage{color,graphicx,lscape,longtable}

\def\boxit#1{\vbox{\hrule\hbox{\vrule\kern6pt\vbox{\kern6pt#1\kern6pt}\kern6pt\vrule}\hrule}}

\usepackage{xcolor}
\newcount\Comments  
\newcommand{\kibitz}[2]{\ifnum\Comments=1\textcolor{#1}{#2}\fi}

\newcommand{\Pcal}{\mathcal{P}}

\usepackage{booktabs}
\usepackage{graphicx}
\usepackage{authblk}
\usepackage{comment}
\usepackage{xcolor}
\usepackage{minitoc}
\usepackage{multirow}

\newtheorem{theorem}{Theorem}[section]
\newtheorem{lemma}[theorem]{Lemma}
\newtheorem{corollary}[theorem]{Corollarry}
\theoremstyle{remark}
\newtheorem{example}{Example}
\theoremstyle{definition}
\newtheorem{assumption}{Assumption}
\newtheorem{remark}{Remark}

\title{{\Large Finite Sample Analysis of Minimax Offline Reinforcement Learning:\\ Completeness, Fast Rates and First-Order Efficiency}}

\author[1]{Masatoshi Uehara \footnote{ \href{mailto:mu223@cornell.edu}{mu223@cornell.edu}}}
\author[2]{Masaaki Imaizumi }
\author[3]{Nan Jiang}   
\author[1]{Nathan Kallus}
\author[1]{Wen Sun}
\author[3]{Tengyang Xie}

\affil[1]{Cornell University  }
\affil[2]{The University of Tokyo}
\affil[3]{University of Illinois at Urbana-Champaign }
\date{}

\begin{document}
%

\maketitle

\begin{abstract}%
We offer a theoretical characterization of off-policy evaluation (OPE) in reinforcement learning using function approximation for marginal importance weights and $q$-functions when these are estimated using recent minimax methods. Under various combinations of realizability and completeness assumptions, we show that the minimax approach enables us to achieve a fast rate of convergence for weights and quality functions, characterized by the critical inequality \citep{bartlett2005}. Based on this result, we analyze convergence rates for OPE. In particular, we introduce novel alternative completeness conditions under which OPE is feasible and we present the first finite-sample result with first-order efficiency in non-tabular environments, i.e., having the minimal coefficient in the leading term.
\end{abstract}






%

\maketitle

\section{Introduction} \label{sec:intro}


 Off-policy evaluation (OPE) is the problem of estimating the expected return in an unknown Markov decision process (MDP) of a given decision policy, known as the evaluation policy, using historical data generated by another policy, known as the behavior policy \citep{Precup2000,pmlr-v37-thomas15,pmlr-v97-bibaut19a,LevineSergey2020ORLT,NamkoongHongseok2020OPEF,ShiChengchun2021DOIE,MaCong2021MOEf}. OPE is especially important in applications where online experimentation may be costly or dangerous, such as in medicine \citep{LuckettDanielJ.2018EDTR,2019PM,chakraborty2013statistical}, recommendation systems \citep{ChenMinmin2019TOCf}, and education \citep{Mandel2014}. OPE is also utilized as an intermediate step for 
 off-policy 
 policy optimization, where the goal is finding a well-performing policy using the historical data \citep{pmlr-v37-thomas15,LazaricAlessandro2016AoCP}. 
 
Our goal is to develop estimators for OPE and study their convergence rates.
To give a high-level overview of our results, we briefly introduce our problem setting here, with formal definitions deferred to \cref{sec:preparation}. We consider an MDP represented by a tuple $M=\{\Scal,\Acal,P_{S'|S,A},P_{R|S,A},\gamma,d_0\}$, where $\Scal$ is a state space, $\Acal$ is an action space, $P_{S'\mid S,A}:\Scal\times \Acal\to \Delta(\Scal)$ is an \emph{unknown} transition distribution, $P_{R\mid S,A}:\Scal\times \Acal\to \Delta(\RR)$ is an \emph{unknown} reward distribution, $\gamma \in [0,1)$ is a discount factor, and $d_0\in\Delta(\Scal)$ is an initial state distribution on $\Scal$. A policy  $\pi:\Scal\to \Delta(\Acal)$ maps from state to a distribution over $\Acal$. In OPE, we consider two policies: an evaluation policy $\epol$ and a behavior policy $\bpol$. We define the policy value of the evaluation policy $\epol$ as
\begin{align}
    J := (1-\gamma) \E_{\epol} \left[ \sum_{t=0}^\infty \gamma^t r_t \right], \label{def:policy_value}
\end{align}
where the expectation is taken under the MDP and evaluation policy $\epol$. In other words, the above expectation is taken with respect to the distribution over trajectories $s_0,a_0,r_0,s_1,a_1,\cdots$ given by the product $d_0(s_0)\epol(a_0| s_0)P_{R|S,A}(r_0|s_0,a_0)P_{S'|S,A}(s_1|s_0,a_0)\epol(a_1| s_1)\cdots$.

We wish to estimate $J$ based on $n$ independent and identically distributed transition observations
generated by the behavior policy $\bpol$ under the same MDP:
\begin{align}
    \{(s_i,a_i,r_i,s_i')\}_{i=1}^{n}\sim P_{S,A,R,S'} (s,a,r,s') =  P_{S}(s)\bpol(a\mid s)P_{S'|S,A}(s'\mid s,a)P_{R|S,A}(r \mid s,a). \label{def:data_generating_process} 
\end{align}
where $P_{S}$ is some state distribution, possibly different from $d_0$. 

When $\gamma=0$ (and understanding $0^0=1$), the OPE problem is reduced to the problem of treatment effect estimation from unconfounded data in causal inference \citep{ImbensGuido2015Cifs,ChernozhukovVictor2018DMLo,robins94}. However, the reinforcement learning (RL) problem ($\gamma>0$) is more complicated in that we need to deal with the sequential nature of the problem and leverage its Markovianity.  
 
Standard approaches to estimating the policy value $J$ use the $q$-function, the policy value when starting at a given state-action pair, or the $w$-function, the density ratio between the evaluation policy's discounted state-action average occupancy distribution and the state-action distribution of the observed data (see \cref{sec:preparation} for precise definitions).
These approaches use the fact that for the true $q$-function, $q_{\epol} : \Scal\times \Acal \to \RR$, and for the true $w$-function, $w_{\epol} : \Scal\times \Acal \to \RR$, we can express $J$ as
\begin{align}
    J&=(1-\gamma) \E_{s_0\sim d_0,a\sim \epol(\cdot \mid s_0)}[q_{\epol}(s_0,a)]=\E_{s \sim P_S, a \sim \epol(\cdot \mid s), r\sim P_{R\mid S,A}(\cdot \mid s,a)} [w_{\epol}(s,a)r]\notag\\
    &=\label{eq:Jforms}
    (1-\gamma) \E_{s_0\sim d_0,a\sim \epol(\cdot \mid s_0)}[q_{\epol}(s_0,a)]
    \\&\phantom{=}+
    \E_{s \sim P_S, a \sim \epol(\cdot \mid s), r\sim P_{R\mid S,A}(\cdot \mid s,a)} [w_{\epol}(s,a)(r-q_{\epol}(s,a)+\gamma \E_{a' \sim \epol(\cdot \mid s')}[q_{\epol}(s',a')])]
    .\notag
\end{align}
Therefore, we can estimate $J$ by estimating $q_{\epol}$ and/or $w_{\epol}$ from the data and then plugging these in above and replacing true expectations with empirical averages.
Plugging into the reformulation in the first equality yields to the so-called direct method of estimation, the second yields importance sampling \citep{Liu2018}, and the third yields the doubly robust method \citep{KallusNathan2019EBtC}.

\textit{Minimax estimators} have recently attracted significant attention for estimating $q$-functions \citep{antos2008learning,FengYihao2019AKLf,UeharaMasatoshi2019MWaQ} and $w$-functions \citep{NachumOfir2019APGf,zhang2019gendice,UeharaMasatoshi2019MWaQ}.
These minimax estimators are defined as solutions of zero-sum games between two function classes, the function to be learned and a critic function. Minimax estimators are easily amenable to function approximation as one can choose any function class, and models for the MDP itself are never needed.
In our work, we consider a unified formulation for minimax $q$- and $w$-estimation that incorporates the possibility of a \emph{stabilizer} term, which penalizes the norm of the critic function (distinguished from a regularizer, which penalizes the learned function). This unified formulation, in particular, also recovers as special cases many existing estimators, including modified RBM \citep{antos2008learning}, the DICE family of estimators \citep{NachumOfir2019APGf,zhang2019gendice}, MWL and MQL \citep{UeharaMasatoshi2019MWaQ}.
Minimax $q$-estimators can be seen as one-step alternative to the multi-step Fitted Q-Iteration (FQI) method \citep{ernst2005tree,le2019batch}, which uses least-squares regression to repeatedly apply backward recursion of values.
For $w$-estimation with general function approximation, the minimax approach is the predominant method.

In this paper, we study minimax estimators of $q$- and $w$-functions and their resulting OPE estimators. 
We first answer the following question: 

\vspace{1ex}
~\hfill\parbox{0.9\textwidth}{\it Under what sufficient 
conditions can we obtain (non-asymptotic) statistical rates with polynomial dependence on problem-dependent quantities for learning $q$-functions, $w$-functions, and policy value using minimax methods with general function approximation, and how do the rates change with different conditions?}\hfill~
\vspace{1ex}

\noindent The problem-dependent quantities mentioned here include the effective horizon $1/(1-\gamma)$, an upper bound on the reward $r$, and an upper bound on $w_{\epol}$. We emphasize that a number of standard methods for OPE utilizing importance sampling based on sequential density ratios \citep{jiang,Precup2000,robins99,thomas2016,pmlr-v97-bibaut19a,zhang2013robust,jiang2019entropy} do \emph{not} satisfy the polynomial dependence, since their error is exponential in the effective horizon.

%
%
In our analyses, two types of expressiveness conditions on the function classes play important roles: \textit{realizability} and \textit{completeness}.  
Given hypothesis classes $\Qbbb$ and $\Wbbb$ for the $q$- and $w$-functions, respectively, realizability refers to the assumption that they contain the true functions, namely, $q_{\epol} \in \Qbbb$ and $w_{\epol} \in \Wbbb$, respectively.
Completeness for $q$-functions states that the hypothesis class is closed under the Bellman operator $\Bcal$ (precisely defined in Section \ref{sec:preparation}), i.e., that $\{\Bcal q: q \in  \Qbbb\}\subset\Qbbb$ holds \citep{ChenJinglin2019ICiB,munos2008finite}. 
Completeness for $w$-functions is a new analogous condition we introduce with respect to the \emph{backward} Bellman operator.
We also introduce new adjoint versions of both $q$- and $w$-completeness. We show how different combinations of these realizability, completeness, and adjoint completeness assumptions provide sufficient conditions and how the rates differ in each case and with or without stabilizer terms.

For $q$-estimation, is is well-known that $q$-realizability and -completeness together yield finite-sample guarantees for FQI \citep{munos2008finite,FanJianqing2019ATAo}. Completeness plays a crucial role in handling the dynamics of the problem, and when we only assume realizability we generally suffer exponential dependence on the effective horizon \citep{edsarx.2010.1189520200101}, which is undesirable in RL \citep{Liu2018,KallusUehara2019}.
Unlike FQI, which only applies to $q$-estimation, statistical properties of minimax estimators are neither well understood nor tightly characterized, despite of their recent popularity. In particular, no work currently derives finite-sample convergence rates for $w$-estimation with general function approximation. For minimax $q$-estimators, although there is some work that derives finite-sample errors \citep{antos2008learning,ChenJinglin2019ICiB,JMLR:v17:13-016}, their results are either limited to finite hypothesis classes or obtain slow rates of $O(n^{-1/4})$ even in settings where $O(n^{-1/2})$ rates can be attained, as we will do.
Finally, we obtain rates of convergence under completely new combinations of realizability and completeness conditions, showing that estimation without bad horizon dependence is possible in new regimes.

To tackle this first question, under a variety of realizability and completeness condition combinations,
we establish \emph{fast} rates of convergence of minimax $q$- and $w$-estimators in terms of Bellman residual errors and L2 errors (precisely defined in \cref{sec:without_sta_rate}). The fast rates are characterized by the critical radius of a localized Rademacher complexities of the function classes \citep{bartlett2005}. This same quantity is often used to tightly characterize fast rates for regression and density estimation \citep{WainwrightMartinJ2019HS:A}. The results also satisfy our requirement of polynomial dependence on the problem-dependent quantities. In particular, we obtain the \emph{first} finite-sample result with a fast rate for any $w$-estimator. We also obtain the \emph{first} finite-sample result with a fast rate for minimax $q$-estimators with general function approximation.



A second question we tackle is

\vspace{1ex}
~\hfill\parbox{0.9\textwidth}{\it Under what sufficient 
conditions can we obtain (non-asymptotic) statistical rates with first-order efficiency for minimax policy value estimation?}\hfill~
\vspace{1ex}

For OPE in the infinite-horizon setting we study,
\citep{KallusNathan2019EBtC} derived the semiparametric efficiency bound, which is essentially the smallest-possible value of the limit of $\sqrt{n}$-scaled mean squared error (for more detail on semiparametric efficiency bounds see \citep{VaartA.W.vander1998As,TsiatisAnastasiosA2006STaM,KosorokMichaelR2008ItEP}).
Representing this quantity by $\mathrm{EB}$, we say that an error bound with probability $1-\delta$ is first-order efficient if it has the form $\sqrt{2(\mathrm{EB})\log(1/\delta)/n} + o(1/\sqrt{n})$, that is, if the constant on the leading term is the smallest-possible instance-dependent constant for a sub-Gaussian tail.
We provide the first first-order efficient finite-sample guarantees for OPE in general function approximation settings.
In comparison, \citep{KallusNathan2019EBtC} only showed that the efficiency bound is achievable asymptotically, hiding many important and possibly exponential terms in the higher-order terms, and they rely on the estimation of $q$- and $w$-functions at some rates as a basic assumption, where it is not clear what sufficient conditions may make such estimation possible.
Aside from this, the only first-order efficient guarantees are in the time-varying tabular setting (i.e., finite state and action spaces with time-dependent transition and reward distributions) using model-based estimates \citep{YinMing2020AEOE}.
Under \rea and completeness of both $w$- and $q$-functions, we obtain the
\emph{first} finite-sample result with first-order efficiency for OPE in general environments with general function approximation. 
We next give a brief overview of the paper and the results. We summarize some of their implications in \cref{tab:summary}.
\begin{itemize}
   \item (\cref{sec:without_sta}) Under $q$- and $w$-realizability, we prove general finite-sample error bounds for estimating policy value using minimax estimators (without stabilizers) by refining the analysis of \cite{UeharaMasatoshi2019MWaQ}. We derive specific rates of convergence depending on the types of hypothesis classes, such as the VC classes, Sobolev balls, and neural networks. 
   \item (\cref{sec:fast_slow,sec:without_sta_rate}) 
   We relax one of the realizability conditions and impose instead one of a variety of completeness-type conditions. We show minimax estimators (with stabilizers) can achieve a fast rate of convergence in estimating the $q$- and $w$-functions themselves in either Bellman residual error or L2-error. 
   The rates depends on the specific hypothesis class as above.
   We apply these results to derive the rate of policy value estimators, which is generally slower than the rate in \cref{sec:without_sta}.  
   \item (\cref{sec:efficiency}) Under both \rea and completeness of both $w$- and $q$-functions, we show that the policy value estimator achieves a convergence rate that is generally faster than the cases in  \cref{sec:without_sta,sec:fast_slow,sec:without_sta_rate}. Under mild conditions, we obtain the first non-asymptotic first-order-efficient bound in general settings. As a byproduct, we also obtain the super-fast $O(1/n)$ rate for (non-tabular) deterministic MDPs.
   \item  (\cref{sec:extension}) We extend our results to: (i) state-based (instead of state-action-based) minimax estimators \citep{Liu2018,FengYihao2019AKLf} and (ii) off-policy policy optimization \citep{XieTengyang2020QASf}. 
\end{itemize}
We provide a discussion of and comparison to related work in \cref{sec:comparison}. We offer some concluding remarks and directions for future research in \cref{sec:summary}.
All proofs are given in \cref{ape:proof}.

\begin{table}[t!]
\caption{
Rates of convergence of estimation errors for the policy value $J$ under various settings.
The columns ``Realizability'' and ``Completeness'' denote whether the hypothesis classes for $q$- and $w$-functions satisfy the realizability and completeness conditions, repsectively.
$\surd^{\dagger}$ refers to the adjoint version of the completeness condition. 
The column ``Stab'' denotes whether the minimax estimators use a stabilizer, as introduced in \cref{sec:preparation}. 
``Rate'' shows the rate derived when the hypothesis class is a VC class or a Sobolev ball of order $\alpha$ on $d$-dimensional inputs. ``First Order Efficiency'' denotes whether the bound is first-order efficient. 
``Convergence of $q,w$'' denotes whether we can also derive estimation rates for the $q$- and $w$-functions themselves and the nature of the errors bounded (Bellman residual errors or L2-errors).
 }
   \begin{tabular}{c|ccccc|ccc|c}\toprule
   \multirow{2}{*}{Sec} &\multicolumn{2}{c}{Realizability} &  \multicolumn{2}{c}{Completeness} & \multirow{2}{*}{Stab} & \multicolumn{2}{c}{Rate} & First Order & Convergence  \\\cline{2-5} \cline{7-8}
      &  $q$ & $w$ & $q$ & $w$ &  & VC & Sobolev & Efficiency & of $q,w$\\ \hline
     \ref{sec:without_sta} & $\surd$ & $\surd$ &  &  & No & $n^{-\frac{1}{2}}$ & $n^{\frac{-1}{2}\vee\frac{-\alpha}{d}}$  &   \\  \hline
    \multirow{4}{*}{\ref{sec:without_sta_rate}} & $\surd$ &  &  &  $\surd$ &  \multirow{4}{*}{No}  & \multirow{4}{*}{$n^{-\frac{1}{4}}$} & \multirow{4}{*}{$n^{\frac{-1}{4}\vee\frac{-\alpha}{2d}}$}   &  & $q$ (Bellman)  \\  
         & $\surd$ &  &  &  $\surd^\dagger$ &   &  &   &  & $q$ (L2)  \\  
     &  & $\surd$ & $\surd$ &  & & &  & & $w$ (Bellman)  \\ 
     &  & $\surd$ & $\surd^\dagger$ &  & & &  & & $w$ (L2)  \\ \hline 
    \multirow{4}{*}{\ref{sec:fast_slow}} & $\surd$ &  &  &  $\surd$ &  \multirow{4}{*}{Yes}  & \multirow{4}{*}{$n^{-\frac{1}{2}}$} & \multirow{4}{*}{$n^{-\frac{\alpha}{2\alpha+d}}$}   &  &  $q$  (Bellman)    \\ 
        & $\surd$ &  &  &  $\surd^\dagger$ &    &  &   &  &  $q$  (L2)    \\ 
     &  & $\surd$ & $\surd$ &  & & &   &  &  $w$  (Bellman)  \\ 
     &  & $\surd$ & $\surd^\dagger$ &  & & &   &  &  $w$  (L2)    \\ \hline
     \ref{sec:efficiency} &  $\surd$ & $\surd$ & $\surd$\big/$\surd^\dagger$ & $\surd$\big/$\surd^\dagger$ &  Yes  & $n^{-\frac{1}{2}}$ & $n^{\frac{-1}{2}\vee\frac{-2\alpha}{ {2\alpha+d}}}$  & $\surd$ & $q,w$  (Bellman) \\\bottomrule
    \end{tabular}
\label{tab:summary}

\end{table}

\subsection{Notation}
\textbf{Functions and their expectations.}
For a function $f: \Scal \times \Acal \times \Rcal \times \Scal \to \Rbbb$, we use $\E[f(s,a,r,s')]$ to denote the expectation with respect to $P_{S,A,R,S'}(\cdot)$ (the data generating distribution in \cref{def:data_generating_process}) and $\E_n[f(s,a,r,s')]$ the empirical average over the $n$ data points. 
We set for any $f:\mathcal S\times\mathcal A\to\mathbb R$ and a policy $\pi(a|s)$, we define $f(s,\pi)\coloneqq \E_{a\sim \pi(a|s)}[f(s,a)]$.
We also define $\|f\|_\infty = \mathrm{ess}\sup_{(s,a,r,s') \sim P_{S,A,R,S'}} |f(s,a,r,s')| $,  $\|f\|_2=(\E[f^2(s,a,r,s')])^{1/2}$ and $\|f\|_{2,n}=(\E_{n}[f^2(s,a,r,s')])^{1/2}$. 
For a set of functions $\Fcal$, we define $\|\Fcal\|_\infty = \sup_{f \in \Fcal} \|f\|_\infty$ and
$\Fcal + f := \{f' + f \mid f' \in \Fcal\}$.
For a measurable space $\Xcal$, $L^2(\Xcal)$ is the space of square integrable functions over $\Xcal$, where the inner product is defined as $\langle f,g\rangle=\E[fg]$ for $f,g \in L^2(\Xcal)$.
For a function $m: \mathbb{R} \to \mathbb{R}$ and an operator $\Tcal: f \mapsto \Tcal f$, $m(\Fcal) := \{m(f(\cdot)): f \in \Fcal\}$ and $\Tcal \Fcal := \{\Tcal f \mid f \in \Fcal\} $ denote sets of mapped functions.

\textbf{Constants.}
Symbols denoted by $c_0,c_1,c_2\ldots$ are universal positive constants that do not depend on any other parameters and can refer to different universal constants in \emph{each} appearance, unless noted otherwise. In contrast, $C_\eta,C_w,\ldots$ is reserved for non-universal, problem-dependent constants.

\textbf{Comparisons.}
For $x,y \in \mathbb{R}$, we write $x \lesssim y$ to mean $x \leq c_0 y$.
For random variables $X$ and $Y$, we write $X \lesssim_\delta Y$ to mean $\mathbb P(X \leq c_0 Y)\geq 1-\delta$ holds for \textit{any} $\delta \in (0,\frac12)$, where $c_0$ is independent of $\delta$.
For sequences $\{x_n\}_{n \in \mathbb{N}}$ and $\{y_n\}_{n \in \mathbb{N}}$, $x_n = O(y_n)$ denotes that there exist $C > 0$ and $\Bar{n}\in \mathbb{N}$ such that $|x_n| \leq C |y_n|$ for all $n \geq \Bar{n}$,
$x_n = \tilde O(y_n)$ denotes that there exists $k\in\mathbb Z$ such that $|x_n| \leq C |y_n|\log^k(n)$,
$x_n = \Omega(y_n)$ denotes $y_n=O(x_n)$,
$x_n = \Theta(y_n)$ denotes both $x_n=O(y_n)$ and $y_n=O(x_n)$ hold,
and similarly for $\tilde\Omega,\,\tilde\Theta$.

\textbf{Size of sets and matrices.}
For a set $S$ equipped with a norm $\|\cdot\|$ and $\tau > 0$, let $\Ncal(\tau,S,\|\cdot\|)$ denote the $\tau$-covering number of $S$ with respect to $\|\cdot\|$, \ie,
the minimal number of $\tau$-radius $\|\cdot\|$-balls needed to cover $S$. We say a set $S$ is \textit{star-shaped} if there exists $z_0\in S$ such that $\forall z\in S,\forall \alpha\in[0,1],z_0\alpha+z(1-\alpha)\in S$.
For a matrix $A$, we denote the set of its all singular values of $A$ by $\sigma(A)$, and $\sigma_{\max}(A) = \sup\sigma(A)$ denotes its maximum. 


\textbf{Rademacher complexities and critical radii.}
The empirical localized Rademacher complexity of a class $\Gcal$ of functions $V\to \Rl$ for a given sample $v_1,\dots,v_n$ is 
\begin{align*} 
   \Rcal_n(\eta;\Gcal)=\E_{\epsilon}\bracks{\sup_{g\in \Gcal: \|g\|_{2,n}\leq \eta }\frac{1}{n}\sum_i \epsilon_i g(v_i)}, 
\end{align*}
where $\E_{\epsilon}$ is shorthand for $\frac1{2^n}\sum_{\epsilon\in\{-1,1\}^n}$, \ie, expectation over independent Rademacher variables (\ie, not marginalizing any other randomness).
%
The empirical (non-localized) Rademacher complexity of $\Gcal$ is $\Rcal_n(\Gcal)=\Rcal_n(\infty;\Gcal)$. A critical radius of $\Gcal$ is any solution $\eta$ to $\Rcal_n(\eta;\Gcal)\leq \eta^2/C_2$ when $\|\Gcal\|_{\infty}\leq C_2$.

\section{Problem Setup}\label{sec:preparation}


A (discounted) MDP is specified by  $M=\{\Scal,\Acal,P_{S'|S,A},P_{R|S,A},\gamma,d_0\}$ as explained in \cref{sec:intro}.
Our goal is to estimate the normalized discounted return  $J$ under a given initial-state density $d_0$ and the evaluation policy $\pi_e$ defined in \cref{def:policy_value}.
%
%
We remark that the policy value $J$ can equivalently be written as
       $J=\E_{(s,a)\sim d_{\epol,\gamma},r\sim P_{R|S,A}(r|s,a)}[r]$,
where $d_{\epol,\gamma}$ is the discounted occupancy density, $ d_{\epol,\gamma} \coloneqq(1-\gamma)\sum_{t=0}^{\infty}\gamma^{t}d_{\epol,t}$, and $d_{\epol,t}$ is the density of $(s_t,a_t)$ under policy $\epol$ and initial-state density $d_0$.

In OPE, 
we must estimate $J$ given data 
generated in the same MDP by the behavior policy $\bpol$, namely,
$n$ independent and identically distributed quadruplets $\{(s_i,a_i,r_i,s'_i)\}_{i=1}^{n}\sim P_{S,A,R,S'}(s,a,r,s')$ defined in  \cref{def:data_generating_process}. 
With a slight abuse of notation, we refer to the joint density over $(s,a,r,s')$ or its marginal on $(s,a)$ as $P_{S,A,R,S'}(\cdot)$ and $P_{S,A}(\cdot)$.

\begin{remark}
The OPE task is related to treatment effect estimation of dynamic treatment regimes in causal inference \citep{MurphyS.A.2003Odtr,Jiang2019,Kosorok2015}.
We can equivalently phrase our results in terms of counterfactual notation (e.g., potential outcomes) if we assume consistency and sequential ignorability \citep{ErtefaieAshkan2018Cdtr,KennedyEdwardH2019NCEB,LuckettDanielJ.2018EDTR,ShiC2020SIot}. Thus our work is also directly applicable to this setting. 
\end{remark}

We define the MDP's transition and (forward) Bellman operators as, respectively,
\begin{align*}
\Tcal:[S\times\mathcal A\to\mathbb R]\to[S\times\mathcal A\to\mathbb R],~ f &\mapsto \prns{(s,a)\mapsto\E_{s'\sim P_{S'|S,A}(s'|s,a)}[f(s',\epol)]}, \\ 
\Bcal:[S\times\mathcal A\to\mathbb R]\to[S\times\mathcal A\to\mathbb R],~ f &\mapsto \prns{(s,a)\mapsto\gamma \Tcal f(s,a)+\E_{r\sim P_{R|S,A}(r|s,a)}[r]}.
\end{align*}

The (true) $q$-\textit{function} is the long-term value of a given state-action pair under $\epol$:
\begin{align*}
    q_{\epol}(s,a) \coloneqq\E_{\epol}\left[\sum_{t=0}^{\infty}\gamma^{t}r_t|s_0=s,a_0=a\right], 
\end{align*}
where the expectation is taken with respect to $P_{R|S,A}(r_0\mid s_0,a_0)P_{S'|S,A}(s_1|s_0,a_0)\epol(a_1\mid s_1)\cdots$. 
%
The corresponding (true) state-value function is $v_{\epol}(s)=q_{\epol}(s,\epol)$. 





The (true) $w$-\textit{function} is the density ratio of the state-action occupancy measures. 
In addition to the $w$-function, we also define its state-only analog as well as the instantaneous density ratio:
$$ w_{\epol}(s,a)\coloneqq \frac{d_{\epol,\gamma}(s,a)}{P_{S,A}(s,a)},\qquad
w_{\epol,S}(s)\coloneqq \frac{d_{\epol,\gamma}(s)}{P_{S}(s)},\qquad
\eta_{\epol}(s,a)\coloneqq \frac{\epol(a\mid s)}{\bpol(a\mid s)},$$
where $d_{\epol,\gamma}(s)$ is the marginal density of $d_{\epol,\gamma}(s,a)$ over $s\in \Scal$.



Based on \cref{eq:Jforms} and given estimators $\hat{q}$ and $\hat{w}$ for $q_{\epol}$ and $w_{\epol}$, respectively, there are three general approaches to estimating $J$:
(i) the direct method (DM) given by $(1-\gamma)\E_{s\sim d_0}[\hat q(s,\epol)]$, (ii) the marginal importance sampling (IS) estimator \citep{Liu2018} given by $\E_n[\hat w(s,a)r]$, and (iii) the infinite-horizon doubly robust (DR) estimator \citep{KallusNathan2019EBtC} given by $(1-\gamma)\E_{s\sim d_0}[\hat q(s,\epol)]+\E_n[\hat w(s,a)(r-\hat q(s,a)+\gamma \hat q(s',\pi_e))]$.  



\begin{remark}
Our goal is to obtain finite sample results with polynomial dependence on problem-dependent quantities such as effective horizon. In contrast, IS or doubly robust estimators using sequential density ratios $\prod_{k=1}^t \eta_{\epol}(a_k\mid s_k)$ \citep{Precup2000,JiangBinyan2019ELFD,jiang,thomas2016} have MSE growing \emph{exponentially} in the effective horizon, which is known as curse of horizon. 
Finite sample results that avoid this dependence have been largely limited to the DM method based on FQI \citep{ernst2005tree}, where
with $q$-realizability ($q_{\epol}\in \Qbbb$) and $q$-completeness ($\Bcal \Qbbb\subset \Qbbb$) good finite-sample results are obtained \citep{munos2008finite,ChenJinglin2019ICiB}. 
\end{remark}


We assume throughout that $P_{R|S,A}$ is supported only on the bounded interval $[-R_{\max},R_{\max}]$. We further assume throughout the following
\begin{assumption}[Concentrability]\label{asm:comp} 
Letting $P_{S'}(\cdot)$ denote the density of $s'$ under $P_{S,A,R,S'}(\cdot)$, there exist $C_\eta,C_m,C_w>0$ such that
\begin{align*}
\|\eta_{\epol}\|_{\infty}\leq C_{\eta},\quad \|P_{S'}(\cdot)/P_{S}(\cdot)\|_{\infty} \leq C_m,\quad  \|w_{\epol}\|_{\infty}\leq C_w=C_{\eta}C_m. 
\end{align*}
\end{assumption}
These assumptions are standard in offline RL \citep{antos2008learning} and ensure the data has sufficient exploration to identify $J$. 
In the non-dynamic setting ($\gamma=0$) and in causal inference, it is also called the overlap assumption \citep{KennedyEdwardH2019NCEB}. 
Note that if $P_{S}(\cdot)$ is invariant (stationary) under the Markov kernel $\pi_b(a\mid s)P_{S'\mid S,A}(\cdot\mid s,a)$ then $P_{S}(\cdot)=P_{S'}(\cdot)$ and we can take $C_m=1$.


Under \cref{asm:comp}, $\Tcal$ is a bounded operator.
Thus, there uniquely exists the \emph{adjoint} transition operator \citep{DebnathLokenath2005Hswa}, as summarized below (\cref{asm:comp} is assumed implicitly everywhere).
\begin{lemma}\label{lem:boundedness} 
The adjoint transition operator $\Tcal'$ is uniquely defined such that
\begin{align*}
    \langle f,\Tcal g\rangle=  \langle \Tcal'f, g\rangle,\, \forall f,g\in  L^2(\Scal \times \Acal). 
\end{align*}
\end{lemma}
In \cref{ape:adjoint}, we show the explicit form of the adjoint transition operator and how it can be interpreted in specific MDPs such as tabular MDPs and linear MDPs \citep{pmlr-v125-jin20a}.

We also define the \emph{backward} Bellman operator:
$$ \Bcal':~f \mapsto \prns{(s,a)\mapsto\gamma \Tcal'f(s,a) +(1-\gamma)\frac{d_{0}(s,a)\epol(a|s)}{P_{S,A}(s,a)}}.$$
In RL, $\Tcal,\Bcal$ are natural for describing $q$-functions. 
$\Tcal',\Bcal'$ can be understood as the analogues for $w$-functions. Then next lemma establishes the $q$- and $w$-Bellman equations.  
\begin{lemma}\label{lem:basic}
By letting $\Tcal_{\gamma} \coloneqq\gamma \Tcal-I,\Tcal'_{\gamma} \coloneqq\gamma \Tcal'-I$, we have
\begin{align*}
(\Bcal-I)q=\Tcal_{\gamma}(q-q_{\epol})\;\forall q,\qquad(\Bcal'-I)w=\Tcal'_{\gamma}(w-w_{\epol})\;\forall w. 
\end{align*}
In particular, we have $\Bcal q_{\epol}=q_{\epol},\,\Bcal' w_{\epol}=w_{\epol}$.
\end{lemma}

The Bellman residual error of any $q$ is defined as $\| \Bcal q-q\|_2$. This is a common error objective in RL \citep{ChenJinglin2019ICiB} since it is zero for $q=q_{\epol}$.
We introduce an analogous notion for the $w$-function by defining Bellman residual error of any $w$ as $\| \Bcal' w-w\|_2$. This is again $0$ when $w=w_{\epol}$.
In contrast, L2-estimation error refers to $\|q_{\epol}-q\|_2$ or $\|w_{\epol}-w\|_2$.

Generally, Bellman residual errors and L2-errors are not equivalent, unless $\gamma$ is small. 
\begin{lemma}\label{lem:bellman} 
We have $(1-\gamma \sqrt{C_mC_{\eta}})\|q-q_{\epol}\|_2 \leq \|\Bcal q-q\|_2\leq (1+\gamma \sqrt{C_mC_{\eta}})\|q-q_{\epol}\|$ and $(1-\gamma \sqrt{C_mC_{\eta}})\|w-w_{\epol}\|_2 \leq \|\Bcal' w-w\|_2\leq (1+\gamma \sqrt{C_mC_{\eta}})\|w-w_{\epol}\|_2 $.
Thus, when $\gamma \sqrt{C_mC_{\eta}}<1$, Bellman residual errors and L2-errors are equivalent.
\end{lemma}

\section{Minimax Estimators for $q$-functions and $w$-functions  }\label{sec:estimation}

We next define the minimax estimators for $q$-functions and $w$-functions in a generic, unified way. 
The estimators take a form of minimax games, unlike standard empirical risk minimization.
To introduce the estimators, we observe that the following moment conditions \citep{KallusNathan2019EBtC} characterizing $w$- and $q$-functions hold: for any functions $w, q\in L^2(\Scal \times \Acal)$, the true functions $q_{\epol}$ and $w_{\epol}$ satisfy
\begin{align*}
    0 &=  \E[w(s,a)\{r-q_{\epol}(s,a)+\gamma q_{\epol}(s',\epol)\}], \\ 
    0 &= \E[w_{\epol}(s,a)\{-q(s,a)+\gamma v(s')\}]+(1-\gamma)\E_{s_0\sim d_0 }[v(s_0)],\,v(s)=q(s,\epol). 
\end{align*}
These two equations are derived by the Bellman equations of $q$-functions and Bellman flow constraints of state-action occupancy measures. 
%
Note that the first condition is also formulated as a conditional moment equation: $\E[ r-q_{\epol}(s,a)+\gamma q_{\epol}(s',\epol)\mid s,a]=0$. From the above moment conditions, we have 
\begin{align*}
  w_{\epol} &\in    \argmin_{w \in L^2(\Scal \times \Acal)}\max_{q \in L^2(\Scal \times \Acal)}|\E[w(s,a)\{-q(s,a)+\gamma q(s',\epol)\}]+(1-\gamma)\E_{s_0\sim d_0 }[q(s_0,\epol)]|,\\
   q_{\epol}&\in  \argmin_{q\in L^2(\Scal \times \Acal)}\max_{w \in L^2(\Scal \times \Acal)}|\E[w(s,a)\{r-q(s,a)+\gamma q(s',\epol)\}]|.
\end{align*}
Motivated\footnote{Note we take this purely as motivation as the equations need not uniquely define $w_{\epol},q_{\epol}$; see \cref{ex:notconverge}. Later, we will prove specific guarantees for our specific estimators under appropriate conditions.} by this,
by introducing function classes $\Wbbb_1,\Wbbb_2,\Qbbb_1,\Qbbb_2$ and possible stabilization coefficients $\lambda,\lambda'\geq0$ and by replacing $\E[\cdot]$ with the empirical analogue $\E_n[\cdot]$, we define the minimax indirect learning (MIL) estimators:
\begin{align*}
  \hat w_{\MIL} &=   \argmin_{w \in \Wbbb_1}\max_{q \in \Qbbb_1}\E_n[w(s,a)\{-q(s,a)+\gamma q(s',\epol)\}]+(1-\gamma)\E_{s_0\sim d_0 }[q(s_0,\epol)]-\lambda \|\Jcal q\|^2_{2,n},\\
  \hat q_{\MIL}&=   \argmin_{q\in \Qbbb_2}\max_{w \in \Wbbb_2}\E_n[w(s,a)\{r-q(s,a)+\gamma q(s',\epol)\}]-\lambda' \|w\|^2_{2,n}, 
\end{align*}
where $\mathcal{J} q=-q(s,a)+\gamma q(s',\epol)$.
Intuitively, $\Wbbb_1,\Wbbb_2$ are hypothesis classes for $w_\epol$ and $\Qbbb_1,\Qbbb_2$ for $q_\epol$. When there is no confusion, we often write $\Wbbb$ for $\Wbbb_1,\Wbbb_2$ and $\Qbbb$ for $\Qbbb_1,\Qbbb_2$. 
We let $C_{\Qbbb},C_{\Wbbb}$ be some constants such that $\|q\|_{\infty}\leq C_{\Qbbb},\forall q\in \Qbbb,\|w\|_{\infty}\leq C_{\Wbbb},\forall w\in \Wbbb$. For simplicity, we assume these function classes are symmetric ($\Wbbb_1=-\Wbbb_1$, etc.). 

Given these estimators for $w_\epol$ and $q_\epol$, we consider the resulting DM ($q$-based), IS ($w$-based), and DR ($q,w$-based) estimators for $J$ based on $\hat w_{\MIL},\hat q_{\MIL} $:
\begin{align*}
    \hat J^{\MIL}_{q}&= \E_{s_0 \sim d_0}[\hat q_{\MIL}(s_0,\epol)],\quad \hat J^{\MIL}_{w} = \E_n[\hat w_{\MIL}(s,a)r],\\
      \hat J^{\MIL}_{wq}&= \E_n[\hat w_{\MIL}(s,a)\{r-\hat q_{\MIL}(s,a)+\gamma \hat q_{\MIL}(s',\epol)\}]+(1-\gamma)\E_{s_0\sim d_0}[\hat q_{\MIL}(s_0,\epol)].
\end{align*}

These minimax estimators are inspired by and in some sense unify the various proposals in \citep{antos2008learning,ChenJinglin2019ICiB,UeharaMasatoshi2019MWaQ,zhang2019gendice,ChowYinlam2019DBEo,FengYihao2019AKLf,tang2019harnessing,LiaoPeng2020BPLi}.  When $\lambda=0.5$, $\hat q_{\MIL}$ is modified BRM (Bellman Residual Minimization) \citep{antos2008learning,LiaoPeng2020BPLi,ChenJinglin2019ICiB}. When $\lambda=0$, $\hat w_{\MIL}$ is MWL in \citep{UeharaMasatoshi2019MWaQ}. When $\lambda'=0$, $\hat q_{\MIL}$ is MQL in \citep{UeharaMasatoshi2019MWaQ}. Note that, however, this is different from \citep{YangMengjiao2020OEvt} as in their formulation $\lambda \|\mathcal{J}q\|^2_{2,n}$ is replaced by $\lambda \|q(s,a)\|^2_{2,n}$, where $\lambda=1$ reduces to GenDICE \citep{zhang2019gendice}. Though \citep{YangMengjiao2020OEvt} uses more general norm rather than $L^2$-norm for stabilizers, we can also extend our analysis to this case. At any rate, \citep{YangMengjiao2020OEvt} does not provide any finite-sample guarantees. There are also several related estimators have been proposed in the non-dynamic setting ($\gamma=0$)  \citep{HirshbergDavid2019AMLE,ZhangRui2020MMRf,WongRaymondKW2018Kcfb,ArmstrongTimothyB2017FOEa,ChernozhukovVictor2018DMLo}.  




\section{Finite Sample Analysis Without Stabilizers under $w$- and $q$-realizability}\label{sec:sufficient}


In this section we analyze the MIL estimators without stabilizers ($\lambda=\lambda'=0$ throughout this section). 

\label{sec:without_sta}

First, we consider the case where we assume that $\Wbbb$ and $\Qbbb$ each contain the true $w$- and $q$-functions, respectively, which we call $w$- and $q$-realizability. Here, we do \emph{not} assume any \comp assumptions.
The below extends \citep[Theorem 9]{UeharaMasatoshi2019MWaQ} to $\hat  J^{\MIL}_{wq}$ and empirical complexities. More crucially, going beyond \citep{UeharaMasatoshi2019MWaQ}, we will proceed to analyze the error terms to also obtain a rate.

\begin{theorem}[Finite sample error bound of MIL]\label{thm:generalization}
Define 
\begin{align*}
 \Gcal(\Wbbb,\Qbbb)& \coloneqq \{w(s,a)(-q(s,a)+\gamma q(s',\epol)):w\in\Wbbb,\,q\in\Qbbb\}\\
   \mathrm{Err}_{\MIL,w} & \coloneqq\Rcal_n(\Gcal(\Wbbb_1,\Qbbb_1))+C_{\Wbbb_1}(R_{\max}+C_{\Qbbb_1})\sqrt{\log(c_0/\delta)n^{-1}},\\
   \mathrm{Err}'_{\MIL,w}& \coloneqq R_{\max}\Rcal_n(\Wbbb_1)+R_{\max}C_{\Wbbb_1}\sqrt{\log(c_0/\delta)n^{-1}},\\
   \mathrm{Err}_{\MIL,q}& \coloneqq R_{\max}\Rcal_n(\Wbbb_2)+\Rcal_n(\Gcal(\Wbbb_2,\Qbbb_2))  +C_{\Wbbb_2}(R_{\max}+C_{\Qbbb_2})\sqrt{\log(c_0/\delta)n^{-1}},\\
   \mathrm{Err}_{\MIL,wq}& \coloneqq \Rcal_n(\Gcal(\Wbbb_1,\Qbbb_2))+C_{\Wbbb_1}(R_{\max}+C_{\Qbbb_2})\sqrt{\log(c_0/\delta)n^{-1}}.
\end{align*}
Then:
\begin{enumerate}
    \item If $w_{\epol}\in \Wbbb_1,q_{\epol}\in \Qbbb_1$, then $|\hat  J^{\MIL}_w-J| \lesssim_\delta \mathrm{Err}_{\MIL,w}+\mathrm{Err}'_{\MIL,w}$.
    \item \label{item:q_mil} If $w_{\epol}\in \Wbbb_2,q_{\epol}\in \Qbbb_2$, then $   |\hat  J^{\MIL}_q-J| \lesssim_\delta \mathrm{Err}_{\MIL,q}$.
    \item If either $w_{\epol}\in \Wbbb_1,\Qbbb_2 -q_{\epol} \subset \Qbbb_1$ or $\Wbbb_1-w_{\epol} \subset \Wbbb_2,q_{\epol}\in \Qbbb_2$, then $
|\hat  J^{\MIL}_{wq}-J|\lesssim_\delta \max(   \mathrm{Err}_{\MIL,w},  \mathrm{Err}_{\MIL,q} )+\mathrm{Err}_{\MIL,wq}$. 
\end{enumerate}

\end{theorem}



We remark on the implications. First, the dominating term with respect to $n$ in all of the error terms is $\Rcal_n(\Gcal(\Wbbb_i,\Qbbb_j)),i,j=1,2$. In other words, When $\Wbbb=\Wbbb_1=\Wbbb_2$, $\Qbbb=\Qbbb_1=\Qbbb_2$, $\Rcal_n(\Gcal(\Wbbb,\Qbbb))$ is an essential term in all of the bounds. We will analyze this term in detail soon when we consider complexity assumptions on $\Wbbb,\Qbbb$. 
Second, comparing to $\hat  J^{\MIL}_w$ and $\hat  J^{\MIL}_q$, the estimator $\hat  J^{\MIL}_{wq}$ requires an additional condition, that $\Qbbb_1$ or $\Wbbb_2$ is large enough in addition to realizability. 
Third, $\hat w_{\MIL},\hat q_{\MIL}$ might not actually converge to $w_{\epol},q_{\epol}$, and irrespective of that, we can still ensure that the MIL OPE estimators converge. This is observed in the following example. 

\begin{example}\label{ex:notconverge}
Consider $\Wbbb_2=\{w_{\epol}\},\Qbbb_2=\{\alpha_1 q_{\epol}+\alpha_2 : \alpha_1 \in \mathbb{R},\,\alpha_2 \in \mathbb{R}\}$. Note that
\begin{align*}
    0=\E[w_{\epol}(s,a)\{r-q(s,a)+\gamma q(s',\epol)\}]
\end{align*}
for any  $q(s,a)=\alpha_1 q_{\epol}(s,a)+\alpha_2$ such that  $(\gamma-1)\alpha_2 =\E[w_{\epol}(s,a)r(1-\alpha_1)]$. Therefore, $\hat q_{\MIL}$ may converge to any such $q$ (or not converge at all), and not necessarily to $q_{\epol}$. This is intuitively because the discriminator class $\Wbbb_2$ is not rich enough to identify $q_{\epol}$ since $\Wbbb_2$ is smaller than $\Qbbb_2$. Irrespective of that,  we can still ensure $\hat  J^{\MIL}_q= J+\op(1)$ since the realizability conditions $w_{\epol}\in \Wbbb_2,q_{\epol}\in \Qbbb_2$ hold as in \cref{item:q_mil} in \cref{thm:generalization}. 
\end{example}

Finally, we remark that we can also obtain an extension of \cref{thm:generalization} agnostic to realizability assumptions, which is useful to analyze sieve estimators. In this case, the final error is given as the estimation error terms in \cref{thm:generalization} plus an approximation error coming from the violation of realizability. For instance, the approximation error term for $\hat J^{\MIL}_q$ is 
\begin{align}\label{eq:error_rea}
    \min_{w\in \Wbbb_2}\max_{q \in \Qbbb_2}\E[(w-w_{\epol})\Tcal_{\gamma}(q-q_{\epol})]+\min_{q \in \Qbbb_2}\max_{w\in \Wbbb_2}\E[w\Tcal_{\gamma}(q-q_{\epol})]. 
\end{align}
The first term corresponds to the violation of $w_{\epol}\in \Wbbb_2$ and the second term corresponds to the violation of $q_{\epol}\in \Qbbb_2$. We formalize the above extension and also discuss the approximation error terms of other estimators in \cref{ape:misspecification}. 




\subsection{Analysis of the error term} \label{sec:rate_without}
We next investigate the error terms in \cref{thm:generalization}. 
To simplify the presentation, we focus on the order with respect to the sample size.
In particular, in the below, all constants in $O(\cdot)$ notation have \emph{polynomial} dependence on both $\log(1/\delta)$ and all problem-dependent quantities such as the effective horizon ($1/(1-\gamma)$) and the concentrability constants.
Below, $\hat J^{\MIL}$ refers to any one of $\hat J^{\MIL}_w,\hat J^{\MIL}_q, \hat J^{\MIL}_{wq}$ assuming the conditions in the corresponding item in \cref{thm:generalization}.
We let $\Wbbb_1=\Wbbb_2=\Wbbb,\Qbbb_1=\Qbbb_2=\Qbbb$.


\begin{example}[VC-subgraph classes]
Let $V(\Wbbb),V(\Qbbb)$ be the VC-subgraph dimension (pseudodimension) of $\Wbbb,\Qbbb$, respectively \citep[Chapter 19]{VaartA.W.vander1998As}. For example,  $\Wbbb=\{(s,a)\mapsto \theta^{\top}\phi(s,a) :\theta \in \Rl^d\}$ has VC-subgraph dimension $d+1$. 
\begin{corollary} \label{cor:vc_cor}
Suppose $V(\Wbbb)$ and $V(\Qbbb)$ are finite. 
Then, 
with probability $1-\delta$, 
\begin{align*}
     |\hat J^{\MIL}-J|=O(\sqrt{\max(V(\Wbbb),V(\Qbbb) )/n}) 
\end{align*}
\end{corollary}

\end{example}

\begin{example}[Nonparametric Models]  \label{exa:nonpara0}
To study a case where $\Wbbb$ and $\Qbbb$ are nonparametric, we can consider their covering numbers.
This covers examples such as function classes given by balls in \Holder, Sobolev, or Besov space \citep{korostelev2011mathematical}, which do not have VC-subgraph dimension. 
\Eg, when $\Wbbb$ is a ball with radius $c$ in the {\Holder} space with smoothness $p$ and input dimension $d$, we have $\log \Ncal(\tau,\Wbbb,\|\cdot\|_{\infty})=O((c/\tau)^{d/p})$ \citep[Chapter 5]{WainwrightMartinJ2019HS:A}. 
\begin{corollary} \label{cor:nonpara_cor}
Suppose $\log \Ncal(\tau,\Wbbb,\|\cdot\|_{\infty})=O(1/\tau^{\beta}),\log \Ncal(\tau,\Qbbb,\|\cdot\|_{\infty})=O(1/\tau^{\beta})$ for 
some $\beta>0$. Then, 
with probability $1-\delta$, 
$|\hat J^{\MIL}-J|=O(n^{-1/(2\vee\beta)})$ if $\beta\neq2$ and $|\hat J^{\MIL}-J|=\tilde O(n^{-1/2})$ if $\beta=2$. 
\end{corollary}
\end{example}

\begin{example}[Linear Sieves]\label{ex:sieves}
In practice, it is difficult to implement MIL with nonparametric classes as in \cref{exa:nonpara0} because it involves an infinite-dimensional optimization.
Instead, we can achieve similar results by restricting to a function class that is linear in some basis functions, such as polynomials, splines, or wavelets \citep[Section 2.3]{ChenXiaohong2007C7LS}, and then slowly grow the number of such functions, aka a linear sieve.

We focus on $\mathcal{S} \times \mathcal{A} = [0,1]^d$
and the target function class being $\Lambda^p_K([0,1]^d)$, the ball of radius $c$ in the {\Holder} space with order $p$ and $d$-dimensional input.
Let $\phi_1,\phi_2,\dots$ be given bases functions such that, 
\begin{align}\textstyle
  \sup_{f\in \Lambda^p_K([0,1]^d)}\inf_{g\in \operatorname{span}(\phi_1,\dots,\phi_k)}\|f-g\|_{\infty}= O(k^{-p/d}) \label{eq:sieve}.
\end{align}
This can, for example, be satisfied by the above mentioned basis functions \citep[Section 2.3]{ChenXiaohong2007C7LS}.  
Then, for an appropriate choice $k_n$, we set $\Qbbb=(-(1-\gamma)^{-1}R_{\max})\vee((1-\gamma)^{-1}R_{\max}\wedge\operatorname{span}(\phi_1,\dots,\phi_{k_n}))$ and $\Wbbb=(-C_w)\vee(C_w\wedge\operatorname{span}(\phi_1,\dots,\phi_{k_n}))$.

Since the function classes $\Qbbb,\Wbbb$ change with $n$, it is not suitable to directly assume $w$, $q$-\rea, as we did Corollary \ref{cor:vc_cor} and \ref{cor:nonpara_cor}. Instead, we assume {\rea} with respect to $\Lambda^p_K([0,1]^d)$, which $\Qbbb,\Wbbb$ approximate, and instead of applying \cref{thm:generalization} directly, we will need to also incorporate an approximation error as in \cref{eq:error_rea}. 

\begin{corollary}\label{cor:linear_sieve}
Suppose \eqref{eq:sieve} holds. Assume $w_{\epol},q_{\epol}\in \Lambda^p_K([0,1]^d)$. Then, for an appropriate choice of $k_n$, with probability $1-\delta$, $|\hat J^{\MIL}-J|= \tilde O(n^{-{p} / ({2p+d})})$. 
\end{corollary}
This rate can be worse than the rate when we directly optimize over $\Lambda^p_K([0,1]^d)$, in which case the rate is $1/n^{\frac 1 2 \wedge \frac p d}$ (Corollary \ref{cor:nonpara_cor}).
\end{example}

\begin{example}[Neural Networks] \label{exa:neural}
Neural networks can also form a (nonlinear) sieve for nonparametric classes \citep{yarotsky2017error,JMLR:v21:20-002}.
We consider neutral networks with a ReLU activation function, which can approximate functions in $H^{p}_K([0,1]^d)$, the $K$-ball in the Sobolev space with an order $p$ on $d$-dimensional input \citep{korostelev2011mathematical}. 
Again we focus on $\mathcal{S} \times \mathcal{A} = [0,1]^d$ and assume $w_{\epol},q_{\epol}\in H^{p}_K([0,1]^d)$ instead of the {\rea} of $\Wbbb,\Qbbb$. We then have the following:


\begin{corollary}\label{cor:neural_net}
Assume $w_{\epol},q_{\epol}\in H^{p}_K([0,1]^d)$.
Let $\Wbbb=\Fcal_{NN}=\Qbbb=\Fcal_{NN}$, where $\Fcal_{NN}$ is the set of functions given by neural networks with ReLU activation, $L$ layers, and $\Xi$ weight parameters in $[-B,B]$. Then, by taking $L=\Theta(\log n),\,\Xi=\Theta(n^{d/(2p+d)})$, with probability $1-\delta$, $|\hat J^{\MIL}-J|= \tilde O(n^{-{p} / ({2p+d})})$.   
\end{corollary}

The rate in Corollary \ref{cor:neural_net} is derived by balancing the approximation (bias) and estimation errors (Rademacher complexities).
%
The above analysis does not necessarily capture practical situations. Usually the number of parameters $\Xi$ is taken much larger than the sample size $n$ so that the bias is very small.
To take this into consideration,  we discuss an overparameterized case, where instead of limiting the number of weight parameters,
we focus on the the norm of the weight parameter matrix.  
We then obtain the following result by combining an approximation rate of neural networks \citep{yarotsky2017error} and the parameter norm bound \citep{golowich2018size}. 

\begin{corollary} \label{cor:neural_net_overparam}
Fix $M_F(\ell) \geq  0,\,\ell = 1,\dots,L$ and
let $\Wbbb=\Fcal_{NN}^\circ,\Qbbb=\Fcal_{NN}^\circ$, where $\mathcal{F}_{NN}^\circ$ is the set of neural network with $L$ layers and ReLU activation function subject to $\|A_\ell\|_{F} \leq M_F(\ell),\,\ell = 1,\dots,L$, where $A_\ell$ is the weight matrix of the $\ell$-th layer. 
Suppose that $\Scal\times \Acal$ is compact and that for any $\varepsilon > 0$ there exists $w_0 \in \Wbbb, q_0 \in \Qbbb$ such that $\max \{\|w_{\epol} - w_0\|_\infty,\|q_{\epol} - q_0\|_\infty\} \leq \varepsilon$.
Then, with probability $1-\delta$, $|\hat J^{\MIL}-J|= \tilde O(\varepsilon +  \sqrt{\log L} \prod_{\ell=1}^L M_F(\ell)/\sqrt{n})$.
\end{corollary}

Here, $\varepsilon$ represents the approximation error of neural networks. The approximation error is guaranteed to be small enough due to the large number of parameters. The final error only depends on the parameter norm bound $M_F(\ell)$ but not the number of parameters. Comparing to Corollary \ref{cor:neural_net}, which uses $o(n)$ parameters, Corollary \ref{cor:neural_net_overparam} captures the case in which the number of parameters can be much larger than the sample size.

\end{example}


\section{Finite Sample Analysis Without Stabilizers under $w$- and $q$-completeness}\label{sec:without_sta_rate}

In this section, we analyze MIL estimators without stabilizers (again, $\lambda=\lambda'=0$ throughout this section) where we relax $w$- or $q$-realizability. Instead, we introduce completeness conditions and show they are sufficient to obtain convergence rates.
All of the conditions are new aside from the (non-adjoint) $q$-completeness condition, which is closely related to the standard completeness condition in FQI as discussed below. 



\begin{theorem}[Convergence rates in MIL]\label{thm:recovery1} 
For a problem-dependent constant $C_\xi$,
\begin{enumerate}
\item\label{thm:recovery1 wadj} If $w_{\epol}\in \Wbbb_1,C_{\xi}(\Wbbb_{1}-w_{\epol})\subset \Tcal_{\gamma}\Qbbb_{1}$ then  $  \|\hat w_{\MIL}- w_{\epol}\|_2\lesssim_\delta \sqrt{C^{-1}_{\xi}\mathrm{Err}_{\MIL,w}}$, 
\item\label{thm:recovery1 wcom} If $w_{\epol}\in \Wbbb_1,C_{\xi}\Tcal'_{\gamma}(\Wbbb_{1}-w_{\epol})\subset \Qbbb_{1} $ then  $  \|\hat w_{\MIL}-\Bcal'\hat w_{\MIL}\|_2\lesssim_\delta \sqrt{C^{-1}_{\xi}\mathrm{Err}_{\MIL,w}}$,  
\item\label{thm:recovery1 qcom} If  
$q_{\epol}\in \Qbbb_2,C_{\xi} \Tcal_{\gamma}(\Qbbb_{2}-q_{\epol})\subset \Wbbb_{2}$ then  $  \|\hat q_{\MIL}-\Bcal\hat q_{\MIL} \|_2\lesssim_\delta \sqrt{C^{-1}_{\xi}\mathrm{Err}_{\MIL,q}}$, 
\item\label{thm:recovery1 qadj} If  
$q_{\epol}\in \Qbbb_2,C_{\xi}(\Qbbb_{2}-q_{\epol})\subset \Tcal'_{\gamma}\Wbbb_{2} $, then  $  \|\hat q_{\MIL}-q_{\epol} \|_2\lesssim_\delta  \sqrt{C^{-1}_{\xi}\mathrm{Err}_{\MIL,q}}$,
\end{enumerate}
where $\mathrm{Err}_{\MIL,*}$ is as defined in \cref{thm:generalization}.
\end{theorem}

Note that, whereas \cref{thm:generalization} provided bounds directly on OPE (and $\hat q^{\MIL},\hat w^{\MIL}$ may well be inconsistent), \cref{thm:recovery1} provides convergence rates for $q$ and $w$ themselves.
The $\mathrm{Err}_{\MIL,*}$ terms can be controlled exactly as done in \cref{sec:rate_without}. Note, however, that these terms appear  in a square root in the bounds. In this sense, the rates are \emph{slow}, for example of order $O(n^{-1/4})$ for VC-subgraph classes.
We discuss the meaning of the assumptions in the next two subsections. 
%
We discuss the implications for OPE in \cref{subsec:application}. 
This will also correspond to slow rates for OPE as we will explain. 
We will obtain \emph{fast} rates in \cref{sec:fast_slow} by introducing stabilizers.


\subsection{$q$-completeness and adjoint $q$-completeness} 

First, we discuss the conditions in \cref{thm:recovery1 qcom,thm:recovery1 qadj} in \cref{thm:recovery1}.
By Lemma \ref{lem:basic}, the condition in \cref{thm:recovery1 qcom} is equivalent to $C_\xi (\Bcal-I)\Qbbb_2\subset\Wbbb_2$.
When $\Wbbb_2=\Qbbb_2$ is convex symmetric, this condition is immediately implied by $\Bcal\Qbbb_2\subset \Qbbb_2$, which is the usual completeness condition of FQI \citep{munos2008finite,ChenJinglin2019ICiB}. Because of this closeness, we refer to this condition as $q$-completeness, but our condition may be very slightly weaker in general.
In linear MDPs \citep{pmlr-v125-jin20a} where the reward and transition densities are linear, it is known that this condition holds when $\Qbbb$ is linear functions \citep[Chapter 15]{agarwal2019reinforcement} (see also \RNum{1}, \RNum{2} in Lemma \ref{lem:q_completenss} below).

The condition in \cref{thm:recovery1 qadj} in \cref{thm:recovery1} is a flipped version of $q$-completeness, which we call adjoint $q$-completeness. This condition also holds under some additional assumptions, as we show below (Item \RNum{3} in Lemma \ref{lem:q_completenss}). In particular, in the tabular case, we can use unrestricted classes $\Wbbb,\Qbbb$ and $q$-\comp and adjoint $q$-\comp hold (Item \RNum{4} in Lemma \ref{lem:q_completenss}).  This is summarized in the following lemma. 

\begin{lemma}\label{lem:q_completenss}
For $\phi:\mathcal S\times\mathcal A\to\mathbb R^d$, suppose $P_{S'|S,A}(s'|s,a)=\vartheta(s')^{\top}\phi(s,a)$, $\E_{r\sim P_{R|S,A}}[r\mid s,a]=\theta^{\top} \phi(s,a)$ for some $\vartheta:\Scal\to \RR,\theta\in \RR^d$.
Let $\Qbbb\coloneqq \{(s,a)\mapsto\beta^{\top} \phi(s,a):\beta\in \mathbb{R}^d\}$. Then:
(\RNum{1}) $ q_{\epol}\in\Qbbb$; (\RNum{2}) $\Tcal_{\gamma}(\Qbbb-q_{\epol})\subset \Qbbb$; (\RNum{3}) for $M_{\epol} \in \mathbb{R}^{d \times d}$ satisfying $\Tcal (\beta^{\top}\phi)=(M_{\epol}\beta)^{\top}\phi$ (which always exists), if $1/\gamma \notin \sigma(M_{\epol})$ and $\E[\phi(s,a)\phi^{\top}(s,a)]$ is non-singular then $(\Qbbb-q_{\epol})\subset \Tcal'_{\gamma}\Qbbb $; (\RNum{4}) if the state and action spaces are finite, $ P_{S,A}(s,a)>0,\,\forall (s,a)\in \Scal\times \Acal$, and $\phi$ is the standard basis in $\mathbb R^{|\mathcal S||\mathcal A|}$, then $ \Tcal_{\gamma}(\Qbbb-q_{\epol})\subset \Qbbb,\,(\Qbbb-q_{\epol})\subset \Tcal'_{\gamma}\Qbbb$.  
\end{lemma}

The condition $1/\gamma \notin \sigma(M_{\epol})$ in \RNum{3} in Lemma \ref{lem:q_completenss} is satisfied, when we can guarantee $M_{\epol}$ is a stochastic matrix, i.e., its entries are positive and rows add up to $1$ \citep{BarretoA.M.S2014PIBo}. 

\begin{lemma}\label{lemma:stochmatrix}
Suppose $\mathbf{1}^{\top}\phi(s,\epol)=1\,\forall s\in \Scal,\int \vartheta(s)\rd\mu(s)=\mathbf{1}$, where $\mathbf{1}$ is a $d$-dimensional vector whose element is $1$. Then, $M_{\epol}$ is a stochatic matrix. Thus, $\forall b\in \sigma(M_{\epol})$, $|b|\leq 1$. 
\end{lemma}

Lemmas \ref{lem:q_completenss} and \ref{lemma:stochmatrix} together ensure adjoint $q$-\comp holds ($(\Qbbb-q_{\epol})\subset \Tcal'_{\gamma}\Qbbb$).
%
%
%
%
%
Finally, we remark that when $P_{S'|S,A}(s'|\cdot)\,P_{R|S,A}(r|\cdot)$ belong to the {\Holder} or Sobolev space, a similar argument shows $q$-completeness holds when $\Wbbb_2,\,\Qbbb_2$ are balls in that space.

\subsection{$w$-\comp and adjoint $w$-\comp} 

Next, we discuss the conditions in \cref{thm:recovery1 wcom,thm:recovery1 wadj} in \cref{thm:recovery1}.
By Lemma \ref{lem:basic}, the condition in \cref{thm:recovery1 wcom} is equivalent to $C_\xi (\Bcal-I)\Wbbb_1\subset\Qbbb_1$.
We refer to this as $w$-\comp.
We refer as adjoint $w$-\comp to its flipped version in \cref{thm:recovery1 wadj}. Note $w$-\comp and adjoint $w$-\comp do not imply each other, and each condition leads to a different convergence guarantee. 
We can again guarantee that both $w$-\comp and adjoint $w$-\comp hold in many settings and in particular in the tabular case.   

\begin{lemma}\label{lem:w_completenss} 
For $\psi:\mathcal S \to\mathbb R^d$, suppose $P_{S,A|S'}(s,a|s')=\vartheta(s,a)^{\top}\psi(s')$, $d_{0}(s)/P_{S}(s)=\theta^{\top} \psi(s)$ for some $\vartheta:\Scal\times\Acal \to \RR^d,\theta\in \RR^d$, where $P_{S,A|S'}(s,a|s')$ is the posterior density of $(s,a)$ given $s'$. Assume $P_{S'}=P_S$, \ie, the data come from a stationary distribution associated with $\pi_b$.
Let $\Wbbb\coloneqq \{ (s,a)\mapsto\beta^{\top}\psi(s)\eta_{\epol}(s,a):\beta \in\mathbb{R}^d\}.$
Then:
(\RNum{1}) $w_{\epol}\in\Wbbb$;
(\RNum{2}) $\Tcal'_{\gamma}(\Wbbb-w_{\epol})\subset \Wbbb$; 
(\RNum{3}) for $M'_{\epol}$ satisfying $\Tcal'(\beta^{\top}\eta_{\epol}\psi)=(M'_{\epol}\beta)^{\top}\eta_{\epol}\psi$ (which always exists), if $1/\gamma \notin \sigma(M'_{\epol})$ and $\E[\psi(s)\psi^{\top}(s)]$ is non-singular  then $(\Wbbb-w_{\epol})\subset \Tcal_{\gamma}\Wbbb$,
%
(\RNum{4})
If the state and action spaces are finite, $P_{S,A}(s,a)>0\,\forall (s,a)\in \Scal\times \Acal$, 
and $\psi(s,a)$ is the standard basis in $\mathbb R^{|\mathcal S||\mathcal A|}$, then $\Tcal'_{\gamma}(\Wbbb-w_{\epol})\subset \Wbbb,\,(\Wbbb-w_{\epol})\subset \Tcal_{\gamma}\Wbbb$.
\end{lemma}

The stationarity assumption in Lemma \ref{lem:w_completenss} is reasonable if the data is generated by long trajectories.   
Notice that compared to a linear MDP, here the linearity assumption is agnostic to the reward density and is imposed on the \emph{posterior} transition density $P_{S,A|S'}(s,a\mid s')$ rather than the transition density $P_{S'|S,A}(s'|s,a)$.
Again, when the feature vector is stochastic, we will have $\sigma_{\max}(M'_{\epol})\leq 1$ and satisfy (\RNum{3}) in Lemma \ref{lem:w_completenss}. Concretely, let $\mathbf{1}$ be a $d$-dimensional vector. Assume $$\mathbf{1}^{\top}\eta_{\epol}(s,a)\psi(s)=1\,\forall (s,a)\in \Scal\times \Acal,\int \vartheta(s,a)\rd\mu(s,a)=\mathbf{1}.$$
Then, $M'^{\top}_{\epol}$ is a stochastic matrix. This implies for any $b\in \sigma(M'_{\epol})$, $|b|\leq 1$ \citep{meyer2000matrix}. 
 




\subsection{Application to Policy Value Estimation}\label{subsec:application}

We now show how \cref{thm:recovery1} can be translated into an error bound for OPE. 
\begin{corollary}\label{cor:implication}  
For a problem-dependent constant $C_\xi$,
\begin{enumerate}
\item If $w_{\epol}\in \Wbbb_1$ and either $C_{\xi}\Tcal'_{\gamma}(\Wbbb_1-w_{\epol})\subset \Qbbb_1$ or $C_{\xi}(\Wbbb_{1}-w_{\epol})\subset \Tcal_{\gamma}\Qbbb_{1}$, then $ |\hat J^{\MIL}_w-J|\lesssim_\delta \mathrm{Err}'_{\MIL,w}+C_{\Qbbb}\sqrt{C^{-1}_{\xi}\mathrm{Err}_{\MIL,w}}$.
\item If $q_{\epol}\in \Qbbb_2$ and either $C_{\xi}\Tcal_{\gamma}(\Qbbb_2-q_{\epol})\subset \Wbbb_2$ or $C_{\xi}(\Qbbb_2-q_{\epol})\subset \Tcal'_{\gamma}\Wbbb_2$,  then $ |\hat J^{\MIL}_q-J |\lesssim_\delta  C_{\Wbbb}\sqrt{C^{-1}_{\xi}\mathrm{Err}_{\MIL,q}}$. 
\item\label{cor:implication dr} If any one of the four conditions in Theorem \ref{thm:recovery1} hold, 
then
$$
    |\hat J^{\MIL}_{wq}-J |\lesssim_\delta  \mathrm{Err}_{\MIL,wq}+C^{-1/2}_{\xi}\max(\break C_{\Qbbb}\sqrt{\mathrm{Err}_{\MIL,w}},\,C_{\Wbbb}\sqrt{\mathrm{Err}_{\MIL,q}}). 
$$
\end{enumerate}
\end{corollary}

While the conditions do not require both $w$-realizability and $q$-realizability, the rates in Corollary \ref{cor:implication} are generally slower than the ones of \cref{thm:generalization}. For example, for VC-subgraph classes, the order in Corollary \ref{cor:implication} is $O(n^{-1/4})$ and the order in \cref{thm:generalization} is  $O(n^{-1/2})$. \Cref{cor:implication dr} in  Corollary \ref{cor:implication} is a double robustness result: it says that the estimator is consistent as long as $w$-\rea and some $w$-\comp or $q$-\rea and some $q$-\comp holds. 

\section{Finite Sample Analysis With Stabilizers under $w$- and $q$-\comp}\label{sec:fast_slow}



In this section, we show MIL with stabilizers ($\lambda>0,\lambda'>0$) attains \emph{fast} rates under completeness. This gives the \emph{first} result for a fast rate of estimating $w$-functions in a general function approximation setting. For $q$-functions, this allows us to obtain fast rates using MIL of the same order as FQI (see \cref{ape:fqi_analysis} regarding FQI).
We then study implications for OPE and compare to results without stabilizers.

\subsection{Fast Rates for $w$- and $q$-function Estimation} 

\begin{theorem}\label{thm:fast_w}
Let $\eta'_{w,n}$ be an upper bound on the critical radii of $\Gcal_{w1}$ and $\Gcal_{w2}$, where  
\begin{align*}
    \Gcal_{w1}&\coloneqq \{(s,a,s')\mapsto  (q(s,a)-\gamma q(s',\epol)): q\in \Qbbb_1 \}, \\
    \Gcal_{w2}&\coloneqq  \{(s,a,s')\mapsto ( w(s,a)-w_{\epol}(s,a))(-q(s,a)+\gamma q(s',\epol)) : w\in \Wbbb_1,q\in \Qbbb_1 \}. 
\end{align*}
Suppose $w_{\epol}\in \Wbbb_1$, $C_{\xi}\Tcal'_{\gamma}(\Wbbb_1-w_{\epol}) \subset \Qbbb_1$ (i.e., $C_{\xi}(\Bcal'-I)\Wbbb_1 \subset \Qbbb_1$), and that $\Qbbb_1$ is star shaped.
Then, 
\begin{align*}
              &\|\Bcal' \hat w_{\MIL}-\hat w_{\MIL}\|_2\lesssim_\delta C^{*}_{w,n} \eta_{w,n},\\
              &\textstyle C^{*}_{w,n}= 1+C_mC_{\eta}+\prns{\lambda^2+C^2_{\Wbbb_1}}/\lambda+C_{\Wbbb_1}+\gamma \sqrt{C_mC_{\eta}}C_{\Wbbb_1}+{C^{-1}_{\xi}},\\&\textstyle \eta_{w,n}= \eta'_{w,n}c_0\sqrt{{\log(c_1/\delta)}/{n}}.
\end{align*}
\end{theorem}

\begin{theorem}\label{thm:fast_q}
Let $\eta'_{q,n}$ be an upper bound on the critical radii of $\Wbbb_2$ and $\Gcal_q$, where 
\begin{align*}
    \Gcal_q:=\{(s,a,s')\mapsto (-q(s,a)+q_{\epol}(s,a)+\gamma q(s',\epol)-\gamma q_{\epol}(s',\epol)) w(s,a):q \in \Qbbb_2,w\in \Wbbb_2\}. 
\end{align*}
Suppose $q_{\epol}\in \Qbbb_2$, $C_{\xi}\Tcal_{\gamma}(\Qbbb_2-q_{\epol})\subset \Wbbb_2$ (i.e., $C_{\xi}(\Bcal-I)\Qbbb_2 \subset \Wbbb_2$), and that $\Wbbb_2$ is star shaped. Then, 
\begin{align*}
   &\|\Bcal \hat q_{\MIL}-\hat q_{\MIL} \|_2\lesssim_\delta C^{*}_{q,n}\eta_{q,n},\\
   & \textstyle C^{*}_{q,n}=1+R_{\max}+C_{\Qbbb_2}+\prns{\lambda'^2+(R_{\max}+C_{\Qbbb_2})^2}/{\lambda'}+{C^{-1}_{\xi}},\\&\textstyle\eta_{q,n}=\eta'_{q,n}c_0\sqrt{{\log(c_1/\delta)}/{n}}.
\end{align*}
\end{theorem}


These theorems state that MIL with stabilizers converge at fast rates when we assume \rea and completeness. Critical radii are commonly used to tightly derive the convergence rate of regression with general function approximation \citep{WainwrightMartinJ2019HS:A}. For example, when we use VC-subgrah classes (parametric models), we can obtain a rate of $\tilde O(n^{-1/2})$ from \cref{thm:fast_q,thm:fast_w} as we will see soon. Recall that the rate in \cref{sec:without_sta_rate} without stabilizers is $O(n^{-1/4})$ and the primary required assumptions are the same, i.e., \rea and completeness. The improved rate is owing to a localization technique based on localized Rademacher complexity \citep{bartlett2005}. We emphasize the required analysis in our setting is significantly more challenging than in the standard nonparametric regression setting. 

In \cref{thm:fast_q,thm:fast_w}, we bound Bellman residual errors under \rea and \comp conditions . 
Like \cref{thm:recovery1}, similar bounds can be obtained for L2-errors 
under corresponding \rea and \emph{adjoint} \comp conditions, which we give in \cref{ape:l2_error}.

We can also obtain bounds agnostic to \rea and completeness assumptions instead with some bias terms accounting for the approximation. 
For example, for $w$- or $q$-function estimation, we have the additional bias terms, respectively
\begin{align}
\label{eq:mis}
        &&\min_{w\in \Wbbb_1}\|\Tcal'_{\gamma}(w_{\epol}-w)\|_2\qquad&\text{and}&&\max_{w\in \Wbbb_1}\min_{q\in \Qbbb_1}\|q-\Tcal'_{\gamma}(w-w_{\epol})\|_2,\quad\text{or}\\
\label{eq:mis2}
    &&\min_{q\in \Qbbb_2}\|\Tcal_{\gamma}(q_{\epol}-q)\|_2\qquad&\text{and}&&\max_{q\in \Qbbb_2}\min_{w\in \Wbbb_2}\|w-\Tcal_{\gamma}(q-q_{\epol})\|_2.
\end{align}
The first term in each of \cref{eq:mis,eq:mis2} corresponds to violations of \rea and the second to violations of \comp. 
Detail is given in \cref{ape:misspecification}. 
We use this more general result for analyzing sieve estimators.

\subsection{Analysis of critical radius}\label{sec:critical}
We next analyze $\eta_{q,n}$ for several function classes, focusing on the rate with respect to $n$. In particular, all constant in $O(\cdot)$ notation below have \emph{polynomial} dependence on both $\log(1/\delta)$ and all problem-dependent quantities.  
In each case we study, the rate of $\eta_{w,n}$ can analogously be calculated.


\begin{example}[VC-subgraph classes]
Assume $V(\Wbbb_2),V(\Qbbb_2)$ are finite.
\begin{corollary}\label{cor:vc_critical}
With probability $1-\delta$, 
$$ \eta_{q,n}=O(\sqrt{\max(V(\Wbbb_2),V(\Qbbb_2) )\log n/n}).$$
\end{corollary}
\end{example}

\begin{example}[Nonparametric Models]\label{exa:nonpaara2}
As in \cref{exa:nonpara0}, we consider nonparametric models that are characterized by covering numbers.
\begin{corollary}\label{cor:nonpara_critical}
Assume $\log\Ncal(\tau,\Wbbb_2,\|\cdot\|_{\infty})=O(1/\tau^{\beta}),\log\Ncal(\tau,\Qbbb_2,\|\cdot\|_{\infty})=O(1/\tau^{\beta})$ for some $\beta>0$. Then, 
with probability $1-\delta$, $\eta_{q,n}=\tilde O(n^{-1/((2+\beta)\vee(2\beta))})$.
\end{corollary}
To gain intuition, compare to the much simpler problem of nonparametric regression, in which case the minimax optimal rate for L2-error for a function class with covering entropy $O(1/\tau^\beta)$ is $\tilde O(n^{-1/(2+\beta)})$ \citep[Chapter 15]{WainwrightMartinJ2019HS:A}. We match this rate for $\beta\leq 2$.
For $\beta>2$ we can typically improve the rate with a more analysis depending on the particular function class.
\end{example}

\begin{example}[Linear Sieves]
We consider linear sieves $\Wbbb,\Qbbb$ as in \cref{ex:sieves}.
\begin{corollary}\label{thm:sieve_rate}
Suppose \eqref{eq:sieve} holds.
    Assume  $q_{\epol}\in \Lambda^p_K([0,1]^d)$, $P_{S'|S,A}(s'|\cdot)\in \Lambda^p_K([0,1]^d)$, and $2\Qbbb_2+\Wbbb'_2\subset \Wbbb_2$.
%
Then, for an appropriate choice of $k_n$,
with probability $1-\delta$, $\|\Bcal \hat q_{\MIL}-\hat q_{\MIL}\|_2=\tilde O(n^{-p/(2p+d)})$. 
\end{corollary}

Since the function classes $\Qbbb_2,\Wbbb_2$ grow with $n$, it is not suitable to directly assume \rea and \comp on $\Qbbb_2,\Wbbb_2$. We instead assume sufficient conditions that ensure \rea and \comp on the {\Holder} balls and leverage a sieve approximation and an extension of \cref{thm:fast_q} given in \cref{ape:misspecification} that accounts for the approximation errors.
In particular, $P_{S'|S,A}(s'|\cdot)\in \Lambda^p_K([0,1]^d)$ is sufficient to ensure \comp of {\Holder} balls. 
Notice that when $d>2p$, the rate here is faster than what would be implied by Corollary \ref{cor:nonpara_critical} for {\Holder} balls (which have a covering entropy rate of $\beta=d/p$).
\end{example}

\begin{example}[Neural Networks]\label{exa:neural_fast}
As in \cref{exa:neural}, we consider using neural networks as sieves. Like the above example, we impose the \rea and \comp on the ball $H^{p}_K([0,1]^d)$ in the {\Holder} space. Then, we have the following result. 

\begin{corollary}\label{cor:neural_policy_rate}
    Assume $q_{\pi^e} \in H_K^p([0,1]^d)$ and  $P_{S'|S,A}(s'|\cdot)\in H_K^p([0,1]^d)$.
    Let $\Qbbb_2=\Fcal_{NN}$, where $\Fcal_{NN}$ is the set of functions given by neural networks with ReLU activation, $L$ layers, and $\Xi$ weight parameters in $[-B,B]$.
    Let $\Wbbb'_2$ be the same set and define $\Wbbb_2$ as the $(L+1)$-layer networks given by $\Wbbb_2 := \Fcal_{NN}+\Fcal_{NN}+\Fcal_{NN}$. 
    Then, by taking $L=\Theta(\log n),\Xi=\Theta(n^{d/(2p+d)})$, with probability $1-\delta$, we have $$\|\Bcal \hat q_{\MIL}-\hat q_{\MIL}\|_2 =  \tilde{O}(n^{-p/(2p + d)})$$
\end{corollary}

As with \cref{exa:neural}, we can again extend the result to the overparametrized regime. Here, we can do so by noting that the squared critical radius is always upper-bounded by the (non-localized) Rademacher complexity, which is calculated in Corollary \ref{cor:neural_net_overparam}.  


\end{example}

\subsection{Application to Policy Value Estimation}

We next translate the above analysis into the error bounds for policy value estimation. 
\begin{corollary}\label{cor:implication2}
(1) Under the conditions of \cref{thm:fast_w},  $$  |\hat J^{\MIL}_w-J|\lesssim_\delta \|q_{\epol}\|_2\eta_{w,n}C^{*}_{w,n}\mathrm{Err}'_{\MIL,w}.$$ (2)  Under the conditions of \cref{thm:fast_q},  $$ |\hat J^{\MIL}_q-J |\lesssim_\delta  \|w_{\epol}\|_2\eta_{q,n}C^{*}_{q,n}.$$
\end{corollary}

Let us compare Corollary \ref{cor:implication2} to \cref{thm:generalization,cor:implication} by focusing on the case of $\hat J^{\MIL}_q$.
For VC classes, both \cref{thm:generalization,cor:implication2} give $\tilde O(n^{-1/2})$. and Corollary \ref{cor:implication} gives $O(n^{-1/4})$.
For nonparametric models with covering number entropy exponent $\beta<2$, \cref{thm:generalization} gives $O(n^{-1/2})$, Corollary \ref{cor:implication} gives $O(n^{-1/4})$, and Corollary \ref{cor:implication2} gives $O(n^{-1/(2+\beta)})$. 
For $\beta\geq 2$, \cref{thm:generalization} gives $O(n^{-1/\beta})$ and both Corollary \ref{cor:implication} and \ref{cor:implication2} give $O(n^{-1/(2\beta)})$ (when $\beta=2$, we have $\tilde O$).
Therefore, the rate in \cref{thm:generalization} is fastest in each case and Corollary \ref{cor:implication} the slowest, while Corollary \ref{cor:implication2} coincides with the best and worst in different cases.  
Under $q$-\comp and $q$-\rea, using no stabilizers yields faster rates than using stabilizers. On the other hand, under $q$-\comp and $q$-\rea, the order is flipped and using a stabilizer yields faster rates than using no stabilizers.


\section{First-Order Efficiency Under $w,q$-realizability and $w,q$-\comp}\label{sec:efficiency}


We now study the constant term in the $O(1/\sqrt{n})$ rates. We show $\hat{J}_{wq}^{\MIL}$ (with sample splitting) can achieve the optimal constant under both realizability and completeness for $w$ and $q$.

When we only assume $w$- and $q$-realizability as in \cref{thm:generalization}, the dependence of the constant on problem dependent quantities may be poor. For example, when the classes $\Wbbb,\Qbbb$ are  both finite, the first-order term of the error $|\hat J^{\MIL}_{wq}-J|$ in \cref{thm:generalization} wrt $n$ can be upper-bounded by Bernstein's inequality: with probability $1-\delta$,
\begin{align}\label{eq:finite}
      |\hat J^{\MIL}_{wq}-J|&\leq\sqrt{\frac{32\sup_{w\in\Wbbb,q\in \Qbbb}\var[w(s,a)\{r-q(s,a)  +\gamma q(s',\epol)\}]\times \log(|\Wbbb||\Qbbb|/\delta) }{n}} \notag \\
      &\quad+O\prns{\frac{\log(|\Wbbb||\Qbbb|/\delta) }{n}}.
\end{align}
This first-order term (the $1/\sqrt{n}-$term) might incur an error characterized by the worst variance term. In this section, we prove that the first-order term is further improved under more conditions in that $\sup_{w\in\Wbbb,q\in \Qbbb}\var[w\{r-q+q(s',\epol)\}]$ is surprisingly replaced with $\var[w_{\epol}\{r-q_{\epol}+\gamma q_{\epol}(s',\epol)\}]$. In addition, $\log(|\Wbbb||\Qbbb|/\delta)$ is replaced with $\log(c/\delta)$, where $c$ is some universal constant. 
Our proof strategy is to localize around the true $w_{\epol},q_{\epol}$ by assuming both realizability and completeness.
This is crucial as it means the first-order term matches the asymptotic efficiency bounds of \citep{KallusNathan2019EBtC}, thus we obtain the \emph{first} finite-sample bound with first-order efficiency in non-tabular environments. 

We analyze a sample-splitting version of $\hat J^{\MIL}_{wq}$ as in \citep{KallusNathan2019EBtC,KallusUehara2019} but using MIL estimators for $w$ and $q$. The sample-splitting procedure simplifies the analysis of the stochastic equicontinuity term in the error \citep{newey94,ChernozhukovVictor2018Dmlf}. 
%
%
Letting $\phi(s,a,r,s';w,q)\coloneqq w(s,a)\{r-q(s,a)+\gamma q(s',\epol)\}+\E_{d_0}[q(s_0,\epol)]$, the procedure is described as follows:
\begin{enumerate}
    \item Split the whole sample into two halves $\Dcal_0$, $\Dcal_1$. 
    \item For $j=0,1$, construct estimators $\hat w^{\MIL}_j,\hat q^{\MIL}_j$ by applying MIL only to the data in $\Dcal_j$.
    \item Return $\tilde J_{wq}^{\MIL}=\frac12\sum_{j=0,1}\frac{1}{\abs{\Dcal_j}}\sum_{i\in \Dcal_j}\phi(s_i,a_i,r_i,s'_i;\hat w^{\MIL}_{1-j},\hat q^{\MIL}_{1-j})$.
\end{enumerate}


\begin{theorem}[First-order efficiency]\label{thm:efficiency}
Assume the conditions in \cref{thm:fast_q,thm:fast_w} hold. Suppose there exists $C_{\iota,n}$ s.t. $\|q-q_{\epol}\|_2\leq C_{\iota,n}\|\Bcal q-q\|_2,\,\forall q\in \Qbbb_2$ and $\|w-w_{\epol}\|_2\leq C_{\iota,n}\|\Bcal'w-w\|_2,\,\forall w\in \Wbbb_1$. Then, with probability $1-(\delta+10\delta')$,
 we have
\begin{align}\ts 
    \abs{\tilde J^{\MIL}_{wq}-J} &  \leq \sqrt{\frac{2\E[w^2_{\epol}(s,a)\{r-q_{\epol}(s,a)+v_{\epol}(s')\}^2]\log(1/\delta)}{n}} \notag \\
    & \quad + C_1 C_{\iota,n} \eta_{w,n}\eta_{q,n} + \frac{C_2 C_{\iota,n}(\eta_{w,n} + \eta_{q,n})}{\sqrt{n}} + \frac{C_3}{n}, \label{eq:firstfirst}
\end{align}
where $C_1 := c_0 C^{*}_{w,n}C^{*}_{q,n}$,  $C_2 :=  c_0(C^{*}_{w,n}+C^{*}_{q,n}\sqrt{1+C_{\eta}C_m})\sqrt{\log(1/\delta')}$, and $C_3 := c_0{C_{\Wbbb_1}C_{\Qbbb_2}(\log(1/\delta)+\log(1/\delta'))}$. 
\end{theorem}
 

This theorem yields a first-order efficient bound provided that $C_{\iota,n}\eta_{w,n}\eta_{q,n}=o_p(1/\sqrt{n})$ and assuming $w$- and $q$-\rea and $w$- and $q$-completeness, as well as $\|q-q_{\epol}\|_2\leq C_{\iota,n}\|\Bcal q-q\|_2$ and $\|w-w_{\epol}\|_2\leq C_{\iota,n}\|\Bcal'w-w\|_2$. We refer to the latter two conditions as 
 %
the \emph{recovery assumption}, which is related to a common condition used in the literature on nonparametric instrumental variable estimation \citep{chen2015sieve,blundell2003semi}.  Note $C_{\iota,n}$ can depend on $n$. 

For example, consider the setting in \cref{exa:neural_fast} where we use neural networks for $w$- and $q$-functions to approximate the Sobolev balls $H^{\beta_1}_{K_1}([0,1]^{d_1})$ and $H^{\beta_2}_{K_2}([0,1]^{d_2})$ respectively. Then, if we additionally assume $C_{\iota,n}=\Omega(1)$, then \cref{eq:firstfirst} becomes
$$
\sqrt{\frac{2\E[w^2_{\epol}(s,a)\{r-q_{\epol}(s,a)+v_{\epol}(s')\}^2]\log(1/\delta)}{n}}+o_p(1/n^{\prns{\frac{\beta_1}{2\beta_1+d_1}+\frac{\beta_2}{2\beta_2+d_2}} \wedge \frac12})
$$
This shows that even if one of the function classes is non-Donsker, i.e., $d_j/\beta_j\geq 2$ for $j=1$ or $2$, we can still obtain a parametric rate $O(n^{-1/2})$ as long as $\frac{\beta_1}{2\beta_1+d_1}+\frac{\beta_2}{2\beta_2+d_2}\geq\frac12$. On the other hand, when we only assume $w$- and $q$-realizability in \cref{thm:generalization}, if one function class is non-Donsker, we cannot ensure the error bound is $O(n^{-1/2})$, per \cref{exa:nonpara0}. 

Whenever $C_{\iota,n}\eta_{w,n}\eta_{q,n}=o(n^{-1/2})$ (e.g., $\frac{\beta_1}{2\beta_1+d_1}+\frac{\beta_2}{2\beta_2+d_2}>\frac12$ and $C_{\iota,n}=\Omega(1)$ in the neural network example above), the first-order term with respect to $n$ is just the first term in \cref{eq:firstfirst}, which is much tighter than both \cref{thm:generalization,eq:finite}. The first-order term exactly corresponds to a sub-Gaussian tail with sub-Gaussian variance parameter $\mathrm{EB}/n$ where 
\begin{align}   \label{eq:bound}
   \mathrm{EB}=\E[w^2_{\epol}(s,a)\{r-q_{\epol}(s,a)+v_{\epol}(s')\}^2],
\end{align}
which is precisely the asymptotic efficiency bound derived by \citep[Theorem 5]{KallusNathan2019EBtC}.\footnote{The asymptotic efficiency bound has two meanings. First, it is the minimum limiting variance of $\sqrt{n}(\hat J-J)$ among regular estimators $\hat J$ \citep[Theorem 25.20]{VaartA.W.vander1998As}. Second, it is a locally minimax lower bound on $\sqrt{n}\E[(\hat J-J)^2]$ among all estimators \citep[Theorem 25.21]{VaartA.W.vander1998As}. In parametric models, such as when the MDP is tabular, it coincides with the Cram\'er-Rao lower bound.} Although \citep{KallusNathan2019EBtC} showed this lower bound can be achieved asymptotically by a general meta-algorithm, their result is asymptotic and assumes estimators are given satisfying certain rates, where achieving these rates is nontrivial and requires \rea and/or \comp assumptions not discussed therein. In their asymptotic result, the higher-order terms may include terms with exponential dependence on problem dependent quantities such as the effective horizon, which is very problematic in RL as discussed in \cref{sec:preparation}. Similarly, the result of \citep{LiaoPeng2020BPLi} for the average reward setting ($\gamma=1$) is also asymptotic and may have exponential dependence in higher-order terms. Ours is the first first-order finite-sample result in a general setting with polynomial dependence on problem dependent quantities, and we give explicit estimators. 

When the MDP is deterministic, i.e., the transition and reward distributions are deterministic given state-action, our result implies that the rate can be \emph{faster than $O(1/\sqrt{n})$}. This is because, in this case, $\mathrm{EB}=0$ in \cref{eq:bound}, so that in \cref{eq:firstfirst} the first term is $0$. In turn, the error rate is controlled by the higher-order terms. For example, when $\eta_{q,n}$ and $\eta_{w,n}$ are parametric rates $O(1/\sqrt{n})$ and $C_{\iota,n}=O(1)$, the final rate is $O(1/n)$. As far as we know, this type of result is the first in the non-tabular offline setting with general function approximation. 

In \cref{thm:efficiency}, we need the aforementioned recovery assumption to translate Bellman residual errors into $L^2$-errors. As in Lemma \ref{lem:bellman}, when $\gamma \sqrt{C_mC_{\eta}}< 1$, we have $C_{\iota,n}=1/(1-\gamma \sqrt{C_mC_{\eta}})$ for any function classes $\Wbbb_1,\Wbbb_2,\Qbbb_1,\Qbbb_2$. By utilizing the structure of function classes, however, the recovery assumption may hold even if $\gamma \sqrt{C_mC_{\eta}}\geq 1$ under mild assumptions. For example, when we use linear models, the recovery assumption is ensured by more tangible conditions.  

\begin{theorem}[Sufficient conditions for the recovery assumption]\label{thm:operator_calculation}
Consider linear models $\Wbbb_1=\Wbbb_2=\Qbbb_1=\Qbbb_2=\{\beta^{\top}\phi(s,a);\beta\in \mathbb{R}^d \}$, and define $X \coloneqq \E[\phi(s,a)\phi(s,a)^{\top}]$.

(1) Suppose $q_{\epol}\in \Qbbb_2$, and there exists a matrix $M_{\epol}$ s.t. $\Tcal (\beta^{\top}\phi)=(M^{\top}_{\epol}\beta)^{\top}\phi$. If $\bar  M_{\epol}, X$ are non-singular, letting $\bar M_{\epol}=\gamma M^{\top}_{\epol}-I$, 
the recovery assumption holds with $$ C_{\iota,n}=\{\sigma_{\max}((\bar  M_{\epol} X\bar M^{\top}_{\epol})^{-1/2}X(\bar  M_{\epol} X\bar M^{\top}_{\epol})^{-1/2})\}^{1/2}.$$ 

(2) Suppose $w_{\epol}\in \Wbbb_1$ and there exists a matrix $M'_{\epol}$ such that $\Tcal'(\beta^{\top}\phi)=(M'_{\epol}\beta)^{\top} \phi$. If  $\bar M'_{\epol}$, $X$ are non-singular, letting $\bar M'_{\epol}\coloneqq \gamma M'^{\top}_{\epol}-I$, the recovery assumption holds with $$ C_{\iota,n}=\{\sigma_{\max}((\bar M'_{\epol}X(\bar M'_{\epol})^{\top})^{-1/2}X(\bar M'_{\epol}X(\bar M'_{\epol})^{\top})^{-1/2})\}^{1/2}.$$
\end{theorem}

We explain the required assumptions in \cref{thm:operator_calculation}. The existence of $M_{\epol},M'_{\epol}$ essentially implies completeness. Lemma \ref{lem:q_completenss} ensures the existence of $M_{\epol}$ in linear MDPs. Lemma \ref{lem:w_completenss} ensures the existence of $M'_{\epol}$ in an MDPs where the posterior transition density is linear. As in \cref{sec:without_sta_rate}, the non-singularity of $\bar M_{\pi},\bar M'_{\pi}$ is ensured if the feature vectors are stochastic. Thus, the required assumptions are quite reasonable. 

In particular, the recovery assumption holds in the tabular case. 
 
\begin{example}[First-order efficiency in a tabular case] 
In a tabular setting, all of the assumptions in \cref{thm:efficiency} are satisfied from the following lemma. Thus, the first-order lower bound is achieved.
\begin{lemma}\label{cor:tabular_recovery} If $|\mathcal S|<\infty$, $|\mathcal A|<\infty$, and $P_{S,A}(s,a)>0,\,\forall (s,a)\in \Scal\times \Acal$, then,  $\bar M_{\epol},\bar M'_{\epol},X$ in \cref{thm:operator_calculation} are non-singular by taking $\phi(s,a)$ as the standard basis vector in $\mathbb{R}^{|\Scal||\Acal|}$.  
\end{lemma}
\end{example}

Finally, we can also obtain the asymptotic statement. This clearly shows the first-order bound is achieved. This can be directly applied to hypothesis testing. 

\begin{corollary}\label{cor:final_efficiency}
Assume the conditions in \cref{thm:efficiency}. Suppose $\eta_{w,n}\eta_{q,n} C_{\iota,n}=o(n^{-1/2}),C_{\iota,n}\eta_{w,n}=o(1),C_{\iota,n}\eta_{q,n}=o(1)$ with probability $1-\delta$. Then, $ \sqrt{n}(\tilde J^{\MIL}_{wq}-J)$ converges weakly to a Gaussian distribution with mean $0$ and variance  $\mathrm{EB}$ as defined in \cref{eq:bound}. 
\end{corollary}

\section{Extensions}\label{sec:extension}


So far, we studied the theoretical properties of state-action-based minimax estimators and their implications to policy evaluation. In this section, we consider some extensions that our analyses admit.
First, we extend our results to state-only-based minimax methods. Then, we discuss how to apply these minimax estimators to policy \emph{optimization} and study the theoretical properties of learned policies. 

\subsection{State-Based MIL}
We now investigate an alternative approach to OPE based on estimating $v_{\epol}(s)$ and $w_{\epol,S}(s)$, rather than $q_{\epol}(s,a)$ and $w_{\epol}(s,a)$ as discussed so far.

Analogously to \cref{eq:Jforms}, we have 
\begin{align*}
    J = \E_{s_0\sim d_0}[v_{\epol}(s_0)]=\E[w_{\epol,S}(s)\eta_{\epol}(s,a)r]. 
\end{align*}
Thus, when we know $\bpol(a\mid s)$ (so that $\eta_{\epol}$ is known), then we can also estimate $J$ based on estimates of $v_{\epol}(s)$ and $w_{\epol,S}(s)$. Motivated by the above equation, given function classes $\Wbbb_{S1},\Wbbb_{S2},\Vbbb_1,\Vbbb_2$ in $[\Scal\to \RR]$, we can consider the state-based MIL estimators:
\begin{align*}
   \hat J^{\MIL}_{w,S} &=\E_n[\hat w_{\MIL,S}(s)\eta_{\epol}(s,a)r],\,    \hat J^{\MIL}_{v,S} =\E_{d_0}[\hat v_{\MIL,S}(s)],\\
     \hat J^{\MIL}_{wv,S} &=\E_n[\hat w_{\MIL,S}(s)\eta_{\epol}(s,a)\{r-\hat v_{\MIL,S}(s)+\gamma \hat v_{\MIL,S}(s')\}]+(1-\gamma)\E_{d_0}[\hat v_{\MIL,S}(s_0)],\\
   \hat w_{\MIL,S}&=\min_{w \in \Wbbb_{S1}}\max_{v \in \Vbbb_1}\E_n[w(s)\eta_{\epol}(s,a)\{-v(s)+\gamma v(s')\}]+(1-\gamma)\E_{d_0}[v(s_0)]-\lambda \|\Jcal v\|^2_{2,n}, \\
\hat v_{\MIL,S}&=\min_{v' \in \Vbbb_2}\max_{w' \in \Wbbb_{S2}}\E_n[w(s)\eta_{\epol}(s,a)\{r-v(s)+\gamma v(s')\}]- \lambda' \|w\|^2_{2,n}, 
\end{align*}
where $\Jcal v=-v(s)+\gamma v(s')$. This recovers as special cases the estimators of \citep{Liu2018,FengYihao2019AKLf,tang2019harnessing} when we do not use stabilizers ($\lambda=\lambda'=0$). Even for these special cases, our theoretical guarantees are novel. For simplicity, we assume each function class $\Wbbb_{S1},\Wbbb_{S2},\Vbbb_1,\Vbbb_2$ is symmetric. And, for the rest of the subsection, assume $\eta_{\epol}(s,a)$ is known (\ie, $\bpol$ is known).

We define the state-based Bellman and transition operators, respectively, as
\begin{align*}
   \Bcal_{S}:~ f \mapsto \prns{s\mapsto\E_{a\sim \epol(a|s),r\sim p(r|s,a),s'\sim p(s'|s,a)}[r+\gamma f(s')]},\\
  \Tcal_S:~ f \mapsto \prns{s\mapsto\E_{a\sim \epol(a|s),r\sim p(r|s,a),s'\sim p(s'|s,a)}[f(s')]}. 
\end{align*}
We again define $\Tcal'_S$ as the adjoint operator of $\Tcal_S$ s.t. $\langle f,\Tcal_S g\rangle=\langle \Tcal'_S f,g\rangle$ on $L_2(S)$. 
And, we define the state-based backward Bellman operator as $$\Bcal'_S:~f \mapsto \prns{s\mapsto \gamma \Tcal'f(s)+(1-\gamma)\frac{d_0(s)}{P_S(s)}}.$$
Finally, we let $\Tcal_{\gamma,S}=\gamma \Tcal_S-I$, $\Tcal'_{\gamma,S}=\gamma \Tcal'_S-I.$


First we adapt \cref{thm:generalization} to state-based MIL estimators under $w_{S}$- and $v$-reazalibility.

\begin{theorem}[Finite sample error bound of state-based MIL estimators under \rea]\label{thm:generalization_s}
Set $\lambda=\lambda'=0$. Define 
\begin{align*}
 \Gcal(\Wbbb,\Vbbb)& \coloneqq \{w(s)\eta_{\epol}(s,a)(-v(s)+\gamma v(s')):w\in\Wbbb_S,\,v \in\Vbbb\},\\
   \mathrm{Err}_{\MIL,w_S} & \coloneqq\Rcal_n(\Gcal(\Wbbb_{S1},\Vbbb_1))+C_{\Wbbb_{S1}}(R_{\max}+C_{\Vbbb_1})\sqrt{\log(c_0/\delta)n^{-1}},\\
   \mathrm{Err}'_{\MIL,w_S}& \coloneqq R_{\max}\Rcal_n(\Wbbb_{S1})+R_{\max}C_{\Wbbb_{S1}}\sqrt{\log(c_0/\delta)n^{-1}},\\
   \mathrm{Err}_{\MIL,v}& \coloneqq R_{\max}\Rcal_n(\Wbbb_{S2})+\Rcal_n(\Gcal(\Wbbb_{S2},\Vbbb_2))  +C_{\Wbbb_{S2}}(R_{\max}+C_{\Vbbb_2})\sqrt{\log(c_0/\delta)n^{-1}},\\
   \mathrm{Err}_{\MIL,wv}& \coloneqq \Rcal_n(\Gcal(\Wbbb_{S1},\Vbbb_2))+C_{\Wbbb_{S1}}(R_{\max}+C_{\Vbbb_2})\sqrt{\log(c_0/\delta)n^{-1}}.
\end{align*}
\begin{enumerate}
    \item If $w_{\epol,S}\in \Wbbb_{S1},v_{\epol}\in \Vbbb_1$, then $|\hat  J^{\MIL}_{w,S}-J| \lesssim_\delta \mathrm{Err}_{\MIL,w_S}+\mathrm{Err}'_{\MIL,w_S}$.
       \item If $w_{\epol,S}\in \Wbbb_{S2},v_{\epol}\in \Vbbb_2$, then $   |\hat  J^{\MIL}_v-J| \lesssim_\delta \mathrm{Err}_{\MIL,v}$.
    \item If either $w_{\epol,S}\in \Wbbb_{S1},v_{\epol}-\Vbbb_2 \subset \Vbbb_1$ or $w_{\epol,S}-\Wbbb_{S1} \subset \Wbbb_{S2},v_{\epol}\in \Vbbb_2$, then $
|\hat  J^{\MIL}_{wv}-J|\lesssim_\delta \max(   \mathrm{Err}_{\MIL,w_S},  \mathrm{Err}_{\MIL,v} )+\mathrm{Err}_{\MIL,wv}$. 
\end{enumerate}
\end{theorem}

Comparing to \cref{thm:generalization}, the error bounds of state-based MIL methods are generally tighter since the complexities of $\Wbbb_S$ and $\Vbbb$ should be smaller than the ones of $\Wbbb$ and $\Qbbb$, as the former are marginalizations of the latter and have smaller input dimension. The caveat is that the knowledge of $\eta_{\epol}(s,a)$ plays a critical role in \cref{thm:generalization_s}. When $\eta_{\epol}(s,a)$ is unknown (even if it is estimated and plugged in), \cref{thm:generalization_s} no longer holds.

Next, we consider the adaptation of \cref{thm:recovery1} to state-based MIL estimators under realizability and completeness of $w_S$- or $v$-functions. 

\begin{theorem}[Convergence rates in state-based MIL estimators under \comp and \rea]\label{thm:recovery1_s} 
Set $\lambda=\lambda'=0$. For a problem-dependent constant $C_\xi$,
\begin{enumerate}
\item\label{thm:recovery1 wadj_s} If $w_{\epol,S}\in \Wbbb_{S1},C_{\xi}(\Wbbb_{S1}-w_{\epol,S})\subset \Tcal_{\gamma,\Scal}\Vbbb_{1}$ then  $  \|\hat w_{\MIL,S}- w_{\epol,S}\|_2\lesssim_\delta \sqrt{C^{-1}_{\xi}\mathrm{Err}_{\MIL,w_{\Scal}}}$, 
\item\label{thm:recovery1 wcom_s} If $w_{\epol,S}\in \Wbbb_{S1},C_{\xi}\Tcal'_{\gamma,\Scal}(\Wbbb_{S1}-w_{\epol,S})\subset \Vbbb_{1} $ then  $  \|\hat w_{\MIL,S}-\Bcal'_{\Scal}\hat w_{\MIL,S}\|_2\lesssim_\delta \sqrt{C^{-1}_{\xi}\mathrm{Err}_{\MIL,w_{\Scal}}}$,  
\item\label{thm:recovery1 qcom_s} If  
$v_{\epol}\in \Vbbb_2,C_{\xi} \Tcal_{\gamma,\Scal}(\Vbbb_{2}-v_{\epol})\subset \Wbbb_{S2} $ then  $  \|\hat v_{\MIL}-\Bcal_{\Scal}\hat v_{\MIL} \|_2\lesssim_\delta \sqrt{C^{-1}_{\xi}\mathrm{Err}_{\MIL,v}}$, 
\item\label{thm:recovery1 qadj_s} If  
$v_{\epol}\in \Vbbb_2,C_{\xi}(\Vbbb_{2}-v_{\epol})\subset \Tcal'_{\gamma,\Scal}\Wbbb_{S2} $, then  $  \|\hat v_{\MIL}-v_{\epol} \|_2\lesssim_\delta  \sqrt{C^{-1}_{\xi}\mathrm{Err}_{\MIL,v}}$. 
\end{enumerate}
\end{theorem}
For example, \cref{thm:recovery1 wcom_s} in \cref{thm:recovery1_s} says when we assume $w_{S}$-realizability and $w_{S}$-completeness, we have a convergence guarantee for $\hat w_{\MIL,S}$. By a state-based analog of Lemma \ref{lem:basic}, $w_{S}$-completeness condition is also equivalent to $C_{\xi}(\Bcal'_{S}-I)\Wbbb_{S1} \subset \Vbbb_{1}$.
Similarly, the $v$-completeness condition (\cref{thm:recovery1 qcom_s} in \cref{thm:recovery1_s}) is equivalent to $C_{\xi}(\Bcal_{S}-I)\Vbbb_{2} \subset \Wbbb_{S2}$.

The aforementioned results are slow rates. As in \cref{sec:fast_slow}, when we use stabilizers $\lambda,\lambda'>0$, we can obtain \emph{fast rates} for $\|\Bcal'_{S}\hat w_{\MIL,S}-\hat w_{\MIL,S}\|_2$ and  $\|\Bcal_{S}\hat v_{\MIL}-\hat v_{\MIL}\|_2$ by analogously modifying \cref{thm:fast_w,thm:fast_q} to stated-based estimators as done in the above theorems.
We omit the details for the sake of brevity.
We note that by utilizing these fast-rate results along with sample splitting, an extension of \cref{thm:efficiency} would yield that, under $w_{S}$-relizability, $w_{S}$-\comp, $v$-realizability, $v$-\comp, recovery assumption, and sufficiently fast rates to eliminate higher order terms, $\hat J^{\MIL}_{wv,S}$ has a finite sample guarantee of the form 
\begin{align}\label{eq:v_eff}
    \sqrt{{2\E[w^2_{\epol}(s,a)\{r-v_{\epol}(s)+\gamma v_{\epol}(s')\}^2]\log(1/\delta)}/{n}}+o(1/\sqrt{n}).
\end{align}
This, however, is \emph{not} first-order efficient since the leading constant is generally \emph{larger} than \cref{eq:bound} (see also Remark 7 in \citep{KallusNathan2019EBtC}). In this specific sense, state-based methods can be inefficient.
The leading constant does agree, however, with the asymptotic efficiency bound in \citep{KallusNathan2020EEoN} for evaluating natural stochastic policies, where state-based method \emph{are} efficient.


\subsection{Policy Optimization}\label{sec:optimization}

In (off-policy) policy optimization we seek to use off-policy data to find a new, well-performing policy. Our results can be directly extended to this task in two different ways.


The first way is a policy-based way \citep{AtheySusan2017EPL,kallus2018balanced}. In this case, given a policy class $\Pi$ and an OPE estimator $\hat J(\pi)$ for the policy value of any one policy $\pi$, we can pick a policy by $\argmax_{\pi \in \Pi}\hat J(\pi)$. We can, for example, use MIL for $\hat J(\pi)$.
When $\Pi$ is finite, we can simply take union bound over our results in order to immediately obtain a regret bound for this policy optimization method. When $\Pi$ is infinite we need to take care to chain our results over successive approximations thereof, given for example a cover entropy condition on $\Pi$. We leave the details of this to future work.


An alternative to this is the value-based approach, where we seek to estimate $q^*$, the (true) $q$-function of the (true) optimal policy, and then deploy the policy that is greedy with respect to our estimate.
Define the Bellman \emph{optimality} operator as
$$\Bcal^{*}:~f \to \prns{(s,a)\mapsto\E_{P_{S'|S,A}(s'|s,a)}[r+\gamma \max_a f(s',a)]},$$
and note that $q^{*}$ satisfies the fixed point equation $q^*=\Bcal^*q^*$.
This, in turn, implies the following set of moment conditions: 
\begin{align*}
    0 &=  \E[w(s,a)\{r-q^{*}(s,a)+\gamma \max_{a \in \Acal} q^{*}(s',a)\}],\quad \forall w(s,a). 
\end{align*}
Similar to the MIL estimator $\hat q_{\MIL}$, motivated by the above moment condition, we can consider the following estimator for $q^*$ and the corresponding learned policy:
\begin{align*}
      &\hat q^{*}_{\MIL} =   \argmin_{q\in \Qbbb_2}\max_{w \in \Wbbb_2}\E_n[w(s,a)\{r-q(s,a)+\gamma \max_{a \in \Acal} q(s',a)\}]-\lambda' \|w\|^2_{2,n},\\
      &\hat \pi^{*}_{\MIL}(\hat a^*_{\MIL}(s)|s)  =1,~~\text{for any choice}~~\hat a^*_{\MIL}(s)\in\argmax_{a\in \Acal} \hat q^{*}_{\MIL}(s,a). 
\end{align*}
When $\lambda'=0$, this is equivalent to MABO \citep{XieTengyang2020QASf} and when $\lambda'=0.5$ to MBRM \citep{antos2008learning}.  

We can extend our results of MIL to the above value-based policy optimization method. First, we consider the case where $\lambda'=0$.
Let $\Pi_{\Qbbb}=\{s\mapsto \argmax_a q(s,a);q\in \Qbbb\}$ be the policy class induced by $\Qbbb$.

\begin{theorem}\label{thm:mabo} Set $\lambda'=0$. 
Suppose \cref{asm:comp} holds for any $\pi\in \Pi_{\Qbbb}$. If $w_{\pi}\in \Wbbb$ for $\forall \pi\in \Pi_{\Qbbb}$ and $q^{*}\in \Qbbb$,  
then
\begin{align*}
&\max_{\pi\in \Pi_{\Qbbb}} J(\pi)-J(\hat \pi^{*}_{\MIL}) \lesssim (1-\gamma)^{-1}\mathrm{Err}_{\mathrm{ope}},\\
&\mathrm{Err}_{\mathrm{ope}}\coloneqq \Rcal_n(\Gcal)+\sqrt{\frac{C_{\Wbbb}(R_{\max}+C_{\Qbbb})\log(1/\delta)}{n}},\\
&\Gcal=\{w(s,a)\{r-q(s,a)+\gamma \max_{a} q(s,a)\};w\in \Wbbb,q\in \Qbbb\}. 
\end{align*}
\end{theorem}

This is analogous to the second statement in \cref{thm:generalization}. Note that assuming \cref{asm:comp} holds for any $\pi\in \Pi_{\Qbbb}$ subsumes the standard concentrability condition for policy optimization, $\sup_{\pi \in \Pi}\|w_{\pi}\|_{\infty}\leq C_{w}$ \citep{munos2008finite}. The rate of $\Rcal_n(\Gcal)$ can calculated exactrly as in \cref{sec:rate_without}. Note \citep{XieTengyang2020QASf} bounds the regret of MABO ($\lambda'=0$) only for a finite hypothesis class. 

Next, we consider the case when $\lambda'>0$, extending \cref{thm:fast_q} to optimization. 

\begin{theorem}\label{thm:fast_q_opti}
Let $\eta'_{q,n}$ be an upper bound on the critical radii of $\Wbbb_2$ and $\Gcal_q$, where 
\begin{align*}
       \Gcal_q:=\{(s,a,s')\mapsto (-q(s,a)+q^{*}(s,a)+\gamma [\max_{a \in \Acal}q(s',a)-\max_{a \in \Acal} q^{*}(s',a)]) w(s,a):q \in \Qbbb_2,w\in \Wbbb_2\}. 
\end{align*}
Suppose $q_{\epol}\in \Qbbb_2$, $C_{\xi}(\Bcal^{*}-I)\Qbbb_2 \subset \Wbbb_2$, and that $\Wbbb_2$ is star shaped. Then, 
\begin{align*}
   &\|\Bcal^{*} \hat q^{*}_{\MIL}-\hat q^{*} _{\MIL} \|_2\lesssim_\delta C^{*}_{q,n}\eta_{q,n},\\
   &C^{*}_{q,n}=1+R_{\max}+C_{\Qbbb_2}+\prns{\lambda'^2+(R_{\max}+C_{\Qbbb_2})^2}/{\lambda'}+\frac1{C_{\xi}},\\&\eta_{q,n}=\eta'_{q,n}+c_0\sqrt{{\log(c_1/\delta)}/{n}}.
\end{align*}
\end{theorem}

\begin{corollary}\label{cor:policy_optimize}
Under the conditions in \cref{thm:fast_q_opti}, we have 
\begin{align*}
    \max_{\pi\in \Pi_{\Qbbb}} J(\pi)-J(\hat \pi^{*}_{\MIL}) \leq \sup_{\pi\in \Pi}\|w_{\pi}\|_2 C^{*}_{q,n}\eta_{q,n}. 
\end{align*}

\end{corollary}

The fast convergence rate for $\|\Bcal^{*} \hat q^{*} _{\MIL}-\hat q^{*}_{\MIL} \|_2$ is, to the best our knowledge, novel even for MBRM ($\lambda'=0.5$). The rate of $\eta_{q,n}$ for each model can calculated exactly as in \cref{sec:critical}.

\section{Discussion of Related Work}\label{sec:comparison}


\paragraph*{\textbf{FQI}} FQI is a method to estimate $q$-functions by iterating regression so that the one-step Bellman residual error is minimized \citep{ernst2005tree,munos2008finite,ChenJinglin2019ICiB,FanJianqing2019ATAo}. FQI gives the same rate as $\hat q_\MIL$ with stabilizers wrt $n$ under $q$-\comp and -realizability. Since most of the work on FQI focuses on policy optimization (Bellman optimality residual error), in \cref{ape:fqi_analysis} we modify the existing FQI analysis for policy evaluation for a direct comparison.

\paragraph*{\textbf{Modified BRM and DICE}} Modified BRM (Bellman Residual Minimization;
\citep{antos2008learning}) is equivalent to $\hat q^{\MIL}$ with $\lambda=0.5$. \citep{JMLR:v17:13-016} obtained a fast rate for modified RBM when $\Wbbb_2,\Qbbb_2$ are nonparametric classes by directly assuming smoothness of the transition operators instead of completeness. Comparing to \citep{JMLR:v17:13-016}, our analysis can handle more general function classes such as neural networks. In addition, the analysis is substantially different since the assumptions are different. Similar comparisons applies to \citep{LiaoPeng2021OEoL,LiaoPeng2020BPLi}, which discuss the fast rate of modified RBM in the average-reward setting. A series of papers consider several minimax RL estimators with stabilizers under the generic name DICE \citep{zhang2019gendice,ChowYinlam2019DBEo}. They typically obtain a rate $O(1/\sqrt{n})$ for the policy value assuming the function classes have finite VC dimensions. In our work, we show how the rate changes depending on the various assumptions on $\Wbbb,\Qbbb$. To our knowledge, they do not discuss our main results: $w$-completeness, fast rates for $q$- and $w$-estimation, and first-order efficiency.

Finally, \citep{DuanYaqi2021RBaR}, whose preprint was posted shortly after ours, also obtained a fast-rate result for modified BRM with general function approximation for policy optimization in a time-varying MDP, which corresponds to our result in \cref{sec:optimization}. Comparing to \citep{DuanYaqi2021RBaR}, our setting is a time-homogeneous MDP (a.k.a infinite horizon setting). In addition, they do not discuss our main results: $w$-completeness, fast rates for $w$-estimation, and first-order efficiency.

\paragraph*{\textbf{Tabular and linear cases}}
\citep{YinMing2020NOPU,YinMing2020AEOE} analyzed the finite sample error bound of model-based estimators in a tabular time-varying MDP and showed that the the first-order lower bound is achieved. Our result is much more general since it allows general function approximation, covers both the non-tabular and tabular cases, and applies to the more common time-homogeneous case. \citep{DuanYaqi2020MOEw,HaoBotao2020SFSM,edsarx.2010.1189520200101} considered linear models and showed upper and lower bounds, but their results are specialized to the linear function approximation setting.

\paragraph*{\textbf{Non-dynamic (bandit) setting}}

In the non-dynamic setting ($\gamma=0$, also termed the logged bandit problem or cross-sectional causal inference), similar results to \cref{thm:efficiency} are obtained by \citep{FosterDylanJ.2019OSL,SmuclerEzequiel2019Auaf,ChernozhukovVictor2018Dmlf,narita2018,BenkeserD2017Drni}. The RL setting is much more complicated in that we need to take dynamics and therefore completeness into account. The proof idea of \cref{thm:fast_q,thm:fast_w} is inspired by \citep{DikkalaNishanth2020MEoC}, which studies the minimax estimation of nonparametric instrumental variable problems. The adaption to our setting is non-trivial since the problem for estimating $w_{\epol}$ is not characterized by a conditional moment equation. 

\paragraph*{\textbf{Efficiency bounds in RL}}

Efficiency bounds are commonly used in causal inference to discuss the optimality of estimators \citep{robins94,TsiatisAnastasiosA2006STaM,KosorokMichaelR2008ItEP}. The efficiency bound in our setting (time-homogeneous MDPs) is derived by \citep{KallusNathan2019EBtC}, which in turn was extended to the average reward setting by \citep{LiaoPeng2020BPLi} and the non-stationary time-homogeneous MDP setting by \citep{BibautAurelien2021Scii}. In all of these papers, the final results for OPE estimators are primarily asymptotic and a polynomial dependence on problem dependent quantities such as horizon is not shown. We focus on the more challenging non-asymptotic analysis. Our result can also be extended to the closely related time-varying MDP setting. In this setting, the efficiency bound in the tabular case is derived by \citep{jiang} and in the general case by \citep{KallusUehara2019}. 

\paragraph*{\textbf{Dynamic treatment regimes}}

OPE is also often studied in the context of dynamic treatment regimes \citep{Kosorok2015,TsiatisAnastasiosA.AnastasiosAthanasios2020Dtr:}. The data generating process in this context is typically non-Markovian and time-inhomogeneous, which we call an NMDP. Though NMDPs are very general, methods for OPE in NMDPs have errors with exponential dependence on horizon \citep{thomas2016,jiang,ZhangBaqun2013Reoo,TsiatisAnastasiosA2006STaM}. In fact, the efficiency bounds under NMDP is exponential in horizon so this is unavoidable \citep{KallusUehara2019}. In contrast, in MDPs, by leveraging the Markovian property, one can avoid exponential dependence, such as we have obtained here. 

\section{Summary and Future Work }\label{sec:summary}


We analyzed the rates of MIL estimators under four cases: (1) $w$- and $q$-realizability, (2) $w$-realizability and -completeness, (3) $q$-realizability and -completeness, (4) $w$- and $q$-realizability and $w$- and $q$-completeness. 
While our generic results for MIL with general function approximation are new in each setting,
settings (2) and (4) are particularly novel, showing OPE is feasible in new cases not studied previously. And, in (4), together with a recovery assumption, we establish the first finite-sample bound with first-order efficiency in general settings. We further extended our results to state-based estimation and to policy optimization.


A few important questions remain. The first is a lower bound under each assumption. The first-order efficiency bound is an asymptotic notion and only concerns the leading term in $n$. Crucially, our finite-sample bounds also establish full polynomial dependence on horizon and other parameters, not just in the $1/\sqrt{n}$ term. Corresponding non-asymptotic lower bounds for each setting, showing whether the dependence in all terms is optimal, is therefore interesting future work. Recent work \citep{edsarx.2010.1189520200101,AmortilaPhilip2020AVot} has shown exponential-in-horizon lower bounds with only $q$-realizability. Our work shows that adding \emph{either} $w$-realizability or $q$-completeness yields polynomial upper bounds.

Another important direction is to further explore policy optimization. We briefly discuss in \cref{sec:optimization} how to adapt our OPE results to policy optimization in a relatively straightforward way. Recent work empirically shows that some penalization is crucial to deal with the inevitable issue of weak overalp \citep{Yu2020,Kidambi2020}. Theoretically, such penalization may help avoid dependence on the worst-case overlap over the policy class, $\sup_{\pi\in \Pi}\|w_{\pi}\|_{\infty}$ \citep{RashidinejadParia2021BORL}. How to incorporate this idea into the MIL framework is an interesting question.


\bibliographystyle{chicago}
\bibliography{rc}
\newpage  
\appendix




\section{Notation}

We summarize the notations in \cref{tab:notation}. 
\begin{table}[!h]
    \centering
        \caption{Notation}
            \label{tab:notation}
    \begin{tabular}{l|l}
     $\Scal,\Acal,\Xcal$ & State space, action space, $\Scal\times \Acal$  \\
     $d_0(\cdot)$ & Initial state density  \\
     $\gamma$  & Discounting factor \\
     $P_{R|S,A}(r|s,a),P_{S'|S,A}(s'|s,a)$ & Reward density, transition density  \\
     $J$ & Policy value\\
     $\epol,\bpol$    & Evaluation policy, Behavior policy \\ 
     $d_{\epol,t}$ & Distribution of $(s_t,a_t)$ under policy $\epol$\\
     $q_{\epol}(s,a),\,v_{\epol}(s)$ & Q-function, value function (for evaluation policy) \\
     $d_{\epol,\gamma}(s,a)$ & Discounted occupancy measure \\
     $P_{S,A,R,S'}(s,a,r,s')$ & Density for the batch data  \\ 
     $\E[f(s,a,r,s')]$ & $\E_{d_{\bpol}}[f(s,a,r,s')]$\\
     $\|f\|_2$ & $\E[\{f(s,a,r,s')\}^2]^{1/2}$\\
     $\|f\|_{\infty}$ & $\max_{x}|f(x)|$\\
     $w_{\epol}(s,a)$ & True weight function $d_{\epol,\gamma}(s,a)/d_{\bpol}(s,a)$\\
     $\eta_{\epol}(s,a)$ & $\epol(a|s)/\bpol(a|s)$ \\
     $C_w,C_{\eta}$ & Constants s.t. $\|w_{\epol}(\cdot)\|_{\infty}\leq C_w,\|\eta(\cdot)\|_{\infty}\leq C_{\eta}$ \\
     $\Wbbb,\Qbbb,\Wbbb_1,\Qbbb_1,\Wbbb_2,\Qbbb_2$  & Function classes \\ 
     $V(\Qbbb)$ & VC dimension \\ 
    $C_{\Wbbb},C_{\Qbbb}$  & Constant envelopes for the function classes \\ 
    $C_{\xi}$ & Values appearing in the closedness assumptions \\ 
    $C_{m}$ & $\|P_{S'}(\cdot)/P_{S}(\cdot)\|_{\infty}\leq C_m$\\ 
    $C_{\iota,n}$ & $ \|\hat q-q_{\epol}\|_2\leq C_{\iota,n}\|\Bcal \hat q-\hat q\|_2$\\ 
    $\lesssim $ & Inequality up to problem dependent values \\ 
    $\lesssim_{\delta} $ & Inequality up to problem dependent values  with probability $1-\delta$\\ 
    $\mu$ & Baseline measure \\
    $\Tcal,\Tcal'$ &   Transition operator, adjoint operator of the transition operator \\
  $\Tcal_{\gamma},\Tcal'_{\gamma}$ &   $\Tcal_{\gamma}=\gamma \Tcal-I$,  $\Tcal'_{\gamma}=\gamma \Tcal'-I$,\\
    $\Bcal,\Bcal'$ &  Forward Bellman operator, backward Bellman operator \\ 
    $\Jcal$ & $\Jcal q=-q(s,a)+\gamma v(s')$\\ 
    $\Rcal_n(\Gcal)$ & Rademacher complexity of the function class $\Gcal$ \\ 
    $\Rcal_n(\eta,\Gcal)$ & Localized Rademacher complexity of the function class $\Gcal$ \\
    $\Ncal(\tau,\Gcal,\|\cdot\|_{\infty})$ & Covering number of $\Gcal$ wrt $\|\cdot\|_{\infty}$ \\
    $\lambda,\lambda'$ & Parameters wrt stabilizers\\ 
    $c_0,\cdots$ & Universal constants \\
    $H^{\alpha}(I^d)$ & Sobolev space \\
    $\Lambda^{\alpha}(I^d)$ & \Holder\,space  \\
    $\sigma(A)$  & Spectrum of a matrix $A$\\
    $\sigma_{\max}(A)$  & Maximum singular value of $A$ 
     \end{tabular}
\end{table}

\newpage 


\section{Adjoint Transition Operators }\label{ape:adjoint}


\subsection{Specific form of adjoint operators}

\begin{lemma}\label{lem:form}
We have 
\begin{align*}
    \Tcal': f(s,a) \mapsto \eta(s,a)\int \frac{P_{S'|S,A}(s|s',a')f(s',a')P_{S,A}(s',a')}{P_{S}(s)}\rd\mu(s',a').
\end{align*}
When $P_{S}(s)=P_{S'}(s)$, 
\begin{align*}
    \Tcal': f(s,a) \mapsto \eta(s,a)\int P_{S,A|S'}(s',a'|s)f(s',a')\rd\mu(s',a'). 
\end{align*}
\end{lemma}
\begin{proof}
The first statement is proved by 
\begin{align*}\ts
    \langle f,\Tcal g\rangle &=\ts \int f(s,a)\{\int g(s',a')\epol(a'|s')P_{S'|S,A}(s'|s,a)\rd\mu(s',a')\}P_{S,A}(s,a)\rd\mu(s,a) \\
    &\ts =\int f(s,a)g(s',a')\epol(a'|s')P_{S'|S,A}(s'|s,a)P_{S,A}(s,a)\rd\mu(s,a,s',a')\\
     &\ts =\int f(s',a')g(s,a)\epol(a|s)P_{S'|S,A}(s|s',a')P_{S,A}(s',a')\rd\mu(s,a,s',a')\\
     &\ts =\int g(s,a)\epol(a|s)P_{S'|S,A}(s|s',a')f(s',a')P_{S,A}(s',a')\rd\mu(s,a,s',a')\\
     &\ts =\int g(s,a)\frac{\epol(a|s)P_{S'|S,A}(s|s',a')f(s',a')P_{S,A}(s',a')}{P_{S,A}(s,a)}P_{S,A}(s,a)  \rd\mu(s,a,s',a')\\
     &\ts =\int g(s,a)\frac{\epol(a|s)P_{S'|S,A}(s|s',a')f(s',a')p_{S,A}(s',a')}{P_{S}(s)\bpol(a|s)}P_{S,A}(s,a)  \rd\mu(s,a,s',a')\\
     &\ts =\int g(s,a)\eta(s,a)\{\int \frac{P_{S'|S,A}(s|s',a')f(s',a')P_{S,A}(s',a')}{P_{S}(s)}\rd\mu(s',a')\}P_{S,A}(s,a)  \rd\mu(s,a)\\
     &=     \langle \Tcal 'f, g\rangle. 
\end{align*}
Here, we use the assumption $P_{S'}(s)\epol(a|s)/P_{S,A}(s,a)\leq C_mC_{\eta}$ for $\forall (s,a)\in \Xcal$ in \cref{asm:comp}. 

Next, we use the assumption $P_{S}(s)=P_{S'}(s)$. Then, the second statement is proved by 
\begin{align*}
    &\frac{P_{S'|S,A}(s|s',a')P_{S,A}(s',a')}{P_{S}(s)}=\frac{P_{S'|S,A}(s|s',a')P_{S,A}(s',a')}{P_{S'}(s)}\\
    &=P_{S,A|S'}(s',a'|s). 
\end{align*}
\end{proof}

\subsection{Examples}
When we restrict $L^2(\Xcal)$ to more restricted finite dimensional classes, the operators $\Tcal,\Tcal'$ are represented by matrices. We define $a\cdot b=a^{\top}b$. 

\begin{example}[Tabular model] \label{exa:tabular}
We use a matrix formulation. Let $\Xcal:=S\times A$, and also $\Pcal^\pi$ be an $|\Scal|| \Acal| \times |\Scal|| \Acal|$ matrix, where its $\Pcal^\pi_{s'a',sa} \coloneqq \Pr(s_{t + 1} = s', a_{t + 1} = a' | s_{t} = s, a_{t} = a, a_{t + 1} \sim \pi(\cdot|s_{t + 1}))$. Given a function $f(s,a)$ on $\Xcal$, we represents this $f(s,a)$ as a $|S||A|\times 1$ vector and denote it by $[f]:=(f(x_1),\cdots,f(x_{|S||A|}))^{\top}$.  In a tabular case, the feature vector $\phi(s,a)$ is given as $(I((s,a)=x_1),\cdots, I((s,a)=x_{|S||A|}))^{\top}$. Thus, $f(s,a)=[f]^{\top}\phi$. Note we often take 
\begin{align*}
    \Qbbb_2 =\Wbbb_2=\Qbbb_1=\Wbbb_1=\{(s,a)\mapsto \theta^{\top}\phi(s,a);\theta\in \mathbb{R}^{|S||A|}\}. 
\end{align*}

Then, we have 
\begin{align*}
   & [ \Tcal f]=P^{\epol}[f],\,    [\{(\gamma \Tcal-I)f\}]=(\gamma P^{\epol}-I)[f],\\
    & \Tcal \phi =(P^{\epol})^{\top}\phi,\,    (\gamma \Tcal-I)\phi=(\gamma P^{\epol}-I)^{\top}\phi. 
\end{align*}
In addition, $\E[\phi\phi^{\top}]$ is a diagonal matrix whose $(i,i)$-th entry corresponding $x_i\in \Xcal$, is $\E[I((s,a)=x_i)]$, i.e., $P_{S,A}(x_i)$. This is non-singular when $P_{S,A}(s,a)>0\,\forall (s,a)\in \Xcal$ . Then, $\langle f,g \rangle=\E[fg]$ is 
\begin{align*}
\langle f,g\rangle= [f]^{\top}\E[\phi\phi^{\top}][g]. 
\end{align*}
In addition, from 
\begin{align*}
    \langle \Tcal f,g \rangle =\{P^{\epol}[f]\}^{\top} \E[\phi\phi^{\top}][g]=[f]^{\top}\{P^{\epol}\}^{\top} \E[\phi\phi^{\top}][g]= [f]^{\top}\E[\phi\phi^{\top}]\{\E[\phi\phi^{\top}]^{-1} \{P^{\epol}\}^{\top}\E[\phi\phi^{\top}]\}[g], 
\end{align*}
and $\langle \Tcal f,g \rangle=\langle f,\Tcal'g \rangle$, we have 
\begin{align*}
 &[\Tcal'f]=( \E[\phi\phi^{\top}]^{-1}(P^{\epol})^{\top} \E[\phi\phi^{\top}])[f],\,
    [(\gamma \Tcal'-I)f]=(\E[\phi\phi^{\top}]^{-1}(\gamma P^{\epol}-I)^{\top} \E[\phi\phi^{\top}])[f],\\
    & \Tcal'\phi=( \E[\phi\phi^{\top}]^{-1}(P^{\epol})^{\top} \E[\phi\phi^{\top}])^{\top}\phi=\E[\phi\phi^{\top}] P^{\epol}\E[\phi\phi^{\top}]^{-1} \phi.  
\end{align*}
\end{example}

\begin{example}[Linear models -ver 1-]\label{exa:ver1}
Consider a $d$-dimensional feature vector $\phi$. For $f_1(s,a)=\theta^{\top}_1\phi,f_2(s,a)=\theta^{\top}_2\phi$, the inner product is defined as 
\begin{align*}
    \langle f_1,f_2 \rangle =\theta^{\top}_1\E[\phi \phi^{\top}]\theta_2. 
\end{align*}
Then, suppose that we have a matrix s.t.
\begin{align*}\ts 
    \Tcal \phi=M^{\top}_{\epol} \phi,\,\Tcal (\theta\cdot \phi)= M_{\epol}\theta \cdot \phi. 
\end{align*}
This matrix $M_{\epol}$ is referred to as a matrix mean embedding of the conditional transition operator \citep{DuanYaqi2020MOEw}. In linear MDPs,  $M_{\epol}$ is explicitly calculated. See Lemma \ref{lem:q_completenss}. We consider an adjoint operator of this operator. For $f_1(s,a)=\theta^{\top}_1\phi,f_2(s,a)=\theta^{\top}_2\phi$, then,
\begin{align*}
    \langle \Tcal f_1,f_2\rangle =(M_{\epol}\theta_1)^{\top}\E[\phi \phi^{\top}]\theta_2=\theta^{\top}_1M_{\epol}^{\top}\E[\phi \phi^{\top}]\theta_2=\theta^{\top}_1\E[\phi \phi^{\top}]\{\E[\phi \phi^{\top}]\}^{-1}M_{\epol}^{\top}\E[\phi \phi^{\top}]\theta_2
\end{align*}
when $\E[\phi \phi^{\top}]$ is non-singular. From  $\langle \Tcal f_1,f_2\rangle=\langle f_1,\Tcal'f_2\rangle$, 
\begin{align}\label{eq:linear_ver1}
    \Tcal':f_2\to \{\E[\phi \phi^{\top}]\}^{-1}M_{\epol}^{\top}\E[\phi \phi^{\top}]\theta_2\cdot \phi. 
\end{align}

Therefore, 
\begin{align*}
     \Tcal' \phi&=\{\{\E[\phi \phi^{\top}]\}^{-1}M_{\epol}^{\top}\E[\phi \phi^{\top}]\}^{\top}\phi\\
     &=\E[\phi \phi^{\top}]M_{\epol}\E[\phi \phi^{\top}]^{-1}\phi\
\end{align*}
\end{example}

\begin{example}[Linear models -ver 2-]\label{exa:ver2}
Consider a $d$-dimensional feature vector $\phi$. For $f_1(s,a)=\theta^{\top}_1\phi,f_2(s,a)=\theta^{\top}_2\phi$, the inner product is defined as 
\begin{align*}
    \langle f_1,f_2 \rangle =\theta^{\top}_1\E[\phi \phi^{\top}]\theta_2. 
\end{align*}
Then, suppose that we have a matrix s.t.
\begin{align*}\ts 
    \Tcal' \phi=M'^{\top}_{\epol} \phi,\, \Tcal' (\theta\cdot \phi)= M'_{\epol}\theta\cdot \phi
\end{align*}
In the case of Lemma \ref{lem:w_completenss},  $M'_{\epol}$ is explicitly calculated. See Lemma \ref{lem:w_completenss}. We consider an adjoint operator of this operator. For $f_1(s,a)=\theta^{\top}_1\phi,f_2(s,a)=\theta^{\top}_2\phi$, then,
\begin{align*}
    \langle \Tcal' f_1,f_2\rangle =(M'_{\epol}\theta_1)^{\top}\E[\phi \phi^{\top}]\theta_2=\theta^{\top}_1{M'_{\epol}}^{\top}\E[\phi \phi^{\top}]\theta_2=\theta^{\top}_1\E[\phi \phi^{\top}]\{\E[\phi \phi^{\top}]\}^{-1}{M'_{\epol}}^{\top}\E[\phi \phi^{\top}]\theta_2
\end{align*}
when $\E[\phi \phi^{\top}]$ is non-singular. From  $\langle \Tcal' f_1,f_2\rangle=\langle f_1,\Tcal f_2\rangle$, 
\begin{align}\label{eq:linear_ver2}
    \Tcal:f_2\to \{\E[\phi \phi^{\top}]\}^{-1}{M'_{\epol}}^{\top}\E[\phi \phi^{\top}]\theta_2\cdot \phi. 
\end{align}
Therefore, 
\begin{align*}
     \Tcal \phi&=\{\{\E[\phi \phi^{\top}]\}^{-1}{M'_{\epol}}^{\top}\E[\phi \phi^{\top}]\}^{\top}\phi\\
     &= \E[\phi \phi^{\top}]{M'_{\epol}}\E[\phi \phi^{\top}]^{-1}\phi. 
\end{align*}
\end{example}



\section{Theorems Agnostic to Realizability and Completeness}\label{ape:misspecification}


We present the theorems agnostic to \rea and \comp. These theorems are useful to analyze sieve estimators and neural networks when $\Qbbb$ and $\Wbbb$ grow with $n$. 

First, we discuss the extension of \cref{thm:generalization}, which takes the bias terms into account. We only present the result of $\hat J^{\MDL}_{w1}$ and $\hat J^{\MIL}_q,\hat J^{\MIL}_w,\hat J^{\MIL}_{wq}$. 


\begin{theorem}\label{thm:allow_mis1}
When $\lambda=\lambda'=0$, 
 \begin{align*}
     |\hat J^{\MIL}_q-J|&\leq c_0\mathrm{Err}_{\MIL,q}+\min_{w\in \Wbbb_2}\max_{q \in \Qbbb_2}\E[(w-w_{\epol})\Tcal_{\gamma}(q-q_{\epol})]+\min_{q \in \Qbbb_2}\max_{w\in \Wbbb_2}\E[w\Tcal_{\gamma}(q-q_{\epol})],\\
    |\hat J^{\MIL}_w-J|&\leq c_0\mathrm{Err}_{\MIL,w}+\min_{w\in \Wbbb_1}\max_{q \in \Qbbb_1}\E[(w-w_{\epol})\Tcal_{\gamma}q]+\min_{q \in \Qbbb_1}\max_{w\in \Wbbb_1}\E[(w-w_{\epol})\Tcal_{\gamma}(q-q_{\epol})],\\
    |\hat J^{\MIL}_{wq}-J|&\leq c_0\mathrm{Err}_{\MIL,wq}+\min_{w\in \Wbbb_2}\max_{q \in \Qbbb_2,w'\in \Wbbb_1}\E[(w-w'-w_{\epol})\Tcal_{\gamma}(q-q_{\epol})]+\min_{q \in \Qbbb_2}\max_{w\in \Wbbb_2}\E[w\Tcal_{\gamma}(q-q_{\epol})],\\
    |\hat J^{\MIL}_{wq}-J|&\leq c_0\mathrm{Err}_{\MIL,wq}+\min_{w\in \Wbbb_1}\max_{q \in \Qbbb_1}\E[(w-w_{\epol})\Tcal_{\gamma}q]+\min_{q \in \Qbbb_1}\max_{w\in \Wbbb_1,q\in \Qbbb_2}\E[(w-w_{\epol})\Tcal_{\gamma}(q-q'-q_{\epol})].
 \end{align*}
\end{theorem}


In \cref{thm:allow_mis1}, the term $\min_{q \in \Qbbb_2}\max_{w\in \Wbbb_2}\E[(w-w_{\epol})\Tcal_{\gamma}(q-q_{\epol})]$ represents the discrepancy from the assumption $q_{\epol}\in \Qbbb_2$. The term $\min_{w\in \Wbbb_2}\max_{q \in \Qbbb_2}\E[(w-w_{\epol})\Tcal_{\gamma}(q-q_{\epol})$ represents the discrepancy from the assumption $w_{\epol}\in \Wbbb_2$. A similar observation is made in \cref{thm:allow_mis1} for the other estimators.

Next, we present the extension of \cref{thm:fast_q}, which takes the bias terms into account. We can similarly extend \cref{thm:fast_w}. 

\begin{theorem}\label{thm:agnostic_fast_q}
 Let $\eta'_{q,n}$ be an upper bound on the critical radius of $\Wbbb$ and $\Gcal_q$, where 
\begin{align*}
    \Gcal_q:=\{(s,a,s')\mapsto  \{-q(s,a)+q_{\epol}(s,a)+\gamma q(s',\epol)-\gamma q_{\epol}(s',\epol)\} w(s,a); q \in \Qbbb_2,w\in \Wbbb_2 \}. 
\end{align*}
Assume   $\Wbbb_2$ is symmetric and star-shaped. By defining $\epsilon_n\coloneqq \max_{q \in \Qbbb_2}\min_{w\in \Wbbb_2}\|w-\Tcal_{\gamma}(q-q_{\epol})\|_2$, for any $q_{0}\in \Qbbb_2$, with probability $1-\delta$, 
\begin{align*}
    \|\Tcal_{\gamma}(\hat q- q_{\epol})\|_2\lesssim  (1+C_1+\lambda+C^2_1/\lambda)  \eta_{q,n}+\epsilon_n+\frac{\|\Tcal_{\gamma}(q_{\epol}- q_0)\|^2_2 }{\lambda \eta_{q,n}}+\|\Tcal_{\gamma}(q_{\epol}- q_0)\|_2,
\end{align*}
where $\eta_{q,n}=\eta'_{q,n}+c_0\sqrt{\frac{\log(c_1/\delta)}{n}}$ for universal constants $c_0,c_1$. 
\end{theorem}

When $q_0=q_{\epol}$ and $\epsilon_n=0$, this result is reduced to \cref{thm:fast_q}. In \cref{thm:agnostic_fast_q}, $\|\Tcal_{\gamma}(q_{\epol}- q_{0})\|_2$ represents the discrepancy from the condition $q_{\epol}\in \Qbbb_2$. The term $\epsilon_n$ represents the discrepancy from the condition $\Tcal_{\gamma}(\Qbbb_2-q_{\epol})\subset \Wbbb_2$.

\section{Fast Rates for $L^2$-errors}\label{ape:l2_error}


In the main text, we show how to derive the fast rates in terms of Bellman residual errors. We show how to obtain the fast rates in terms of $L^2$-errors. 

\begin{theorem}\label{thm:fast_q_l2}
Let $\eta'_{q,n}$ be an upper bound on the critical radius of $\Wbbb_2$ and $\Gcal_q$, where 
\begin{align*}
    \Gcal_q:=\{(s,a,s')\mapsto \{-q(s,a)+q_{\epol}(s,a)+\gamma q(s',\epol)-\gamma q_{\epol}(s',\epol)\} w(s,a); q \in \Qbbb_2,w\in \Wbbb_2\}. 
\end{align*}
Assume $q_{\epol}\in \Qbbb_2,C_{\xi}(\Qbbb_2-q_{\epol})\subset \Tcal'_{\gamma}\Wbbb_2,\sup_{w\in \Wbbb_2}\|w\|/\|\Tcal'_{\gamma}w\|\leq C_{\iota}$,\,$\Wbbb_2$ is symmetric and star-shaped.  Then, with probability $1-\delta$, 
\begin{align*}\ts
   &\| \hat q_{\MIL}-q_{\epol} \|_2\lesssim  (1+C_1+C_1^2/\lambda'+C_{\Wbbb_2}(1+\gamma \sqrt{C_mC_{\eta}})+1/C_{\xi}+\lambda'\{1+C^2_{\iota}\} )\eta_{q,n}\,\\
    C_1 &= C_{\Qbbb_2}+ (1-\gamma)^{-1}R_{\max}, 
\end{align*}
where $\eta_{q,n}=\eta'_{q,n}+c_0\sqrt{\frac{\log(c_1/\delta)}{n}}$ for universal constants $c_0,c_1$. 
\end{theorem}

\begin{theorem}\label{thm:fast_w_l2}
Let $\eta'_{w,n}$ be an upper bound on the critical radius of $ \Gcal_{w1},  \Gcal_{w2}$, where  
\begin{align*}
 \Gcal_{w1}:&=\{(s,a,s')\mapsto  \{q(s,a)-\gamma q(s',\epol)\}; q\in \Qbbb_1 \}, \\
 \Gcal_{w2}:&=\{(s,a,s')\mapsto \{ w(s,a)-w_{\epol}(s,a)\}\{-q(s,a)+\gamma q(s',\epol)\} ; w\in \Wbbb_1,q\in \Qbbb_1 \}. 
\end{align*}
Suppose  $w_{\epol} \in \Wbbb_1$, $C_{\xi}(\Wbbb_1-w_{\epol}) \subset \Tcal_{\gamma}\Qbbb_1,\,\sup_{q\in \Qbbb_1}\|q\|/\|\Tcal_{\gamma}q\|\leq C_{\iota}$,\,$\Qbbb_1$ is symmetric and star-shaped. Then, with probability $1-\delta$, 
\begin{align*}
       \|\hat w_{\MIL}-w_{\epol}\|_2\lesssim \braces{1+C^2_{\Wbbb_1}/\lambda+C_{\Wbbb_1}+C_{\Qbbb_1}+1/ C_{\xi}+\lambda\{1+C^2_{\iota}(1+C_mC_{\eta})\}}\eta_{w,n}
\end{align*}
where $\eta_{w,n}=\eta'_{w,n}+c_0\sqrt{\frac{\log(c_1/\delta)}{n}}$ for universal constants $c_0,c_1$.
\end{theorem}



\section{Complete Analysis of FQI for Policy Evaluation}\label{ape:fqi_analysis}


The FQI has been analyzed in detail \citep{munos2008finite,ChenJinglin2019ICiB} for policy optimization. 
For policy evaluation, \citep{LeHoang2019BPLu} analyzes FQI in detail when the hypothesis classes are VC-classes. For the sake of completeness, we show the analysis when using general function classes. To make the analysis easier, we consider the sample splitting version of FQI: 
\begin{enumerate}
    \item Set the number of iterations $T$, and split the whole $n$ sample into $T$ bins. We denote the $t$ th bin by $\zeta_t (t\in [1,\cdots,T])$. 
    \item At the initial stage $t=0$, take some function $f_0$. 
     \item Repeat the following from $t=1,\cdots,T$ 
         \begin{align*}
            f_t=\argmin_{q\in \Qbbb}\sum_{j\in \zeta_t}\{f_{t-1}(s'^j,\epol)+r^j-q(s^j,a^j)\}^2.
        \end{align*}
    \item Define the estimator $\E_{d_0}[f_T(s_0,\epol)]$. 
\end{enumerate}

\begin{theorem}\label{th,:fqi_analysis}
We set the critical radius of $\mathrm{star}(\Qbbb-q_{\epol})$ as $\eta_t$ at $t$-th stage, i.e., a solution to 
\begin{align*}
    \Rcal_{\zeta_t}(\eta;\Qbbb)\leq \eta^2/C_{\Qbbb}. 
\end{align*}
where $\Rcal_{\zeta_t}(\eta;\Qbbb)$ is an empirical localized Rademacher complexity of $\Qbbb$ based on the data $\zeta_t$. Suppose $q_{\epol}\in \Qbbb,\Bcal\Qbbb\subset\Qbbb$. With probability $1-\delta$, 
\begin{align*}
|\E_{s\sim d_0}[f_T(s, \epol)] - J| \lesssim \frac{(1 - \gamma^{T/2})}{\sqrt{1 - \gamma}}\|w_{\epol}\|_2 C_{\Qbbb}\eta_q + \gamma^{{T}/{2}}C_{\Qbbb}, 
\end{align*}
where $\eta_q=\max_{t\in [1,\cdots,T]}(\eta_{t}+\sqrt{c_0\log(c_1T/\delta)/n_t})$ for some universal constants $c_0,c_1$. 
\end{theorem}
\begin{proof}
We directly combine the two lemmas, Lemma \ref{lem:fast_rate_reg} and Lemma \ref{lem:fqipop}. The statement is immediately concluded. 
\end{proof}

The critical radius $\eta_q$ is similarly calculated as in \cref{sec:fast_slow}. For example, when $\Qbbb$ belongs to VC classes, and the VC dimension is $V(\Qbbb)$,  then, 
\begin{align*}
    \eta_q=O(\max_{t\in [1,\cdots,T]}(\sqrt{V(\Qbbb)/n_t}+\sqrt{c_0\log(c_1T/\delta)/n_t} )). 
\end{align*}
The above theorem concludes that the error of FQI is roughly $\eta_q$ when we only focus on the sample size. This result is equivalent to the analysis of $\hat J^{\MIL}_q$ with stabilizers as in Corollary \ref{cor:implication2}. Note that the dependence of $\|w_{\epol}\|_2$ is the same.

\begin{lemma}\label{lem:fast_rate_reg}
We set the critical radius of $\mathrm{star}(\Qbbb-q_{\epol})$ as $\eta_t$ at $t$-th stage, i.e., a solution to 
\begin{align*}
    \Rcal_{\zeta_t}(\eta;\Qbbb)\leq \eta^2/C_{\Qbbb}. 
\end{align*}
Suppose $\Bcal \Qbbb\subset \Qbbb$. Then, with probability $1-\delta$, $\|f_t - \Bcal f_{t - 1}\|_{2} \lesssim C_{\Qbbb}\eta$
where $\eta=\eta_t+\sqrt{c_0\log(c_1/\delta)/n_t}$ for some universal constants $c_0,c_1$. 
\end{lemma}

\begin{lemma}
\label{lem:fqipop}
 If we have $\|f_t - \Bcal f_{t - 1}\|_{2} \leq \varepsilon$\, for any $T\geq t \geq 1$, then
\begin{align*}
|\E_{s\sim d_0}[(1-\gamma)f_T(s, \epol)] - J]| \leq \frac{\{1 - \gamma^{T/2}\}(1+\gamma^{1/2})}{\sqrt{1-\gamma}} \|w_{\epol}\|_2 \varepsilon + 2\gamma^{{T}/{2}}(1-\gamma)C_{\Qbbb}.
\end{align*}
\end{lemma}

\section{Proofs}\label{ape:proof}


\subsection{Proof of  \cref{sec:preparation}}

\begin{proof}[Proof of Lemma \ref{lem:boundedness}]

We show the operator $\Tcal$ is bounded: 
\begin{align*}
    \|\Tcal\|=\sup_{u\in L^2(\Xcal)} \|\Tcal u\|_2/\|u\|_2\leq \sqrt{C_mC_{\eta}}. 
\end{align*}

\begin{align*}
    \|\Tcal u\|^2_2&=\|\E_{s'\sim P_{S'|S,A}(s'|s,a),a'\sim \epol(a'|s')}[u(s',a')]\|^2_2\\
    &\leq \E_{s'\sim P_{S'|S,A}(s'|s,a),a'\sim \epol(a'|s'),(s,a)\sim P_{S,A}}[u(s',a')^2] \tag{Jensen}\\
       &= \E_{ (s,a)\sim P_{S,A}}\bracks{\frac{P_{S'}(s)\epol(a|s)}{P_S(s)\bpol(a|s)} u(s,a)^2}\leq C_mC_{\eta}\|u\|^2_2. 
\end{align*}
Them, the adjoint operator is uniquely defined following standard theory of Hilbert spaces \citep[Theorem 4.3.13]{DebnathLokenath2005Hswa}. In addition, note that the adjoint operator is also bounded: $  \|\Tcal'\|=\|\Tcal\|\leq \sqrt{C_mC_{\eta}}$. 

\end{proof}

\begin{proof}[Proof of Lemma \ref{lem:basic}]
~ \\
\textbf{Statement $(\Bcal-I)q=\Tcal_{\gamma}(q-q_{\epol})$} \\ 
\begin{align*}
    (\Bcal-I)q&=\E_{P_{S'|S,A}(s'|s,a),P_{R|S,A}(r|s,a)}[r+\gamma q(s',\epol)-q(s,a)]\\
    &=\E_{P_{S'|S,A}(s'|s,a)}[\gamma q(s',\epol)-\gamma q_{\epol}(s',\epol)-q(s,a)+q_{\epol}(s,a)]\\
    &=(\gamma \Tcal-I)(q-q_{\epol})=\Tcal_{\gamma}(q-q_{\epol}). 
\end{align*}

\textbf{Statement $\Bcal q_{\epol}=q_{\epol}$} \\ 
This is proved by 
\begin{align*}
    \Bcal q_{\epol}=\Tcal_{\gamma}(q_{\epol}-q_{\epol})+q_{\epol}=q_{\epol}. 
\end{align*}

\textbf{Statement $(\Bcal'-I)w=\Tcal'_{\gamma}(w-w_{\epol})$} \\ 

We check
\begin{align*}\ts 
    -\frac{(1-\gamma)d_{0}(s)\epol(a|s)}{P_{S,A}(s,a)}=(\gamma \Tcal'-I)w_{\epol}. 
\end{align*}
This is proved as 
\begin{align*} 
\ts (\gamma \Tcal'-I)\frac{d_{\epol,\gamma}(s,a)}{P_{S,A}(s,a)}&= \ts \gamma \eta_{\epol}(s,a)\int \frac{P_{S'|S,A}(s|s',a')d_{\epol,\gamma}(s',a')}{P_{S}(s)}\rd\mu(s',a')  -\frac{d_{\epol,\gamma}(s,a)}{P_{S,A}(s,a)} \\
&= \ts \gamma \epol(a|s)\int \frac{P_{S'|S,A}(s|s',a')d_{\epol,\gamma}(s',a')}{P_{S,A}(s,a)}\rd\mu(s',a')  -\frac{d_{\epol,\gamma}(s,a)}{P_{S,A}(s,a)} \\
    &=\ts -\frac{(1-\gamma)d_{0}(s)\epol(a|s)}{P_{S,A}(s,a)}. 
\end{align*}
We use Lemma 11 \citep{UeharaMasatoshi2019MWaQ}:
\begin{align*}
    \gamma \epol(a|s)\int P_{S'|S,A}(s|s',a')d_{\epol,\gamma}(s',a')\rd\mu(s',a')-d_{\epol,\gamma}(s,a)=-(1-\gamma)d_{0}(s)\epol(a|s). 
\end{align*}
Thus, We have 
\begin{align*}
    (\Bcal'-I)w&=\gamma \Tcal' w+\frac{(1-\gamma)d_{0}(s)\epol(a|s)}{P_{S,A}(s,a)}-w\\
    &=(\gamma \Tcal'-I)(w-w_{\epol})=\Tcal'_{\gamma}(w-w_{\epol}). 
\end{align*}

\textbf{Statement $\Bcal' w_{\epol}=w_{\epol}$}\\ 
This is proved by 
$\Bcal'w_{\epol}=\Tcal'_{\gamma}(w_{\epol}-w_{\epol})+w_{\epol}=w_{\epol}$. 
\end{proof}

\begin{proof}[Proof of Lemma \ref{lem:bellman}]
~ \\ 
\textbf{Statement: $\|q-q_{\epol}\|_2$ and $ \|q-\Bcal q\|_2$ are equivalent } \\ 
We prove
\begin{align*}
    (1-\gamma \sqrt{C_mC_{\eta}})\|q-q_{\epol}\|_2 \leq   \|\Bcal q-q\|_2\leq (1+\gamma \sqrt{C_mC_{\eta}})\|q-q_{\epol}\|_2.
\end{align*}

First, 
\begin{align*}
    \|\Bcal q-q\|_2&=\|(\gamma \Tcal-I)(q-q_{\epol})\|_2\\
    &\leq \gamma\|\Tcal(q-q_{\epol})\|_2+ \|q-q_{\epol}\|_2\\
        &\leq \gamma\sqrt{C_mC_{\eta}}\|q-q_{\epol}\|_2+\|q-q_{\epol}\|_2. \tag{From the proof of Lemma \ref{lem:boundedness}}
\end{align*}
In addition, 
\begin{align*}
     \|q-q_{\epol}\|_2&\leq \|q-\Bcal q\|_2+\|\Bcal q-q_{\epol}\|_2= \|q-\Bcal q\|_2+\|\Bcal q-\Bcal q_{\epol}\|_2\\
     &= \|q-\Bcal q\|_2+\gamma \|\Tcal(q-q_{\epol})]\|_2\\
          &= \|q-\Bcal q\|_2+\gamma\sqrt{C_mC_{\eta}}\|q-q_{\epol}\|_2. 
\end{align*}
This shows $\|q-q_{\epol}\|_2\{1-\gamma\sqrt{C_mC_{\eta}}\}\leq \|q-\Bcal q\|_2.$
~ \\ 
\textbf{Statement:  $\|w-w_{\epol}\|_2$ and $\|w-\Bcal' w\|_2$ are equivalent  }
~
We prove
\begin{align*}
      (1-\gamma \sqrt{C_mC_{\eta}})\|w-w_{\epol}\|_2 \leq \|\Bcal' w-w\|_2\leq  (1+\gamma \sqrt{C_mC_{\eta}})\|w-w_{\epol}\|_2.
\end{align*}

First, 
\begin{align*}
    \|\Bcal' w-w\|_2&\leq \|(\gamma \Tcal'-I)(w-w_{\epol})\|_2\\ 
    &\leq \gamma \|\Tcal'w-\Tcal' w_{\epol}\|_2+\|w-w_{\epol}\|_2\\
    &\leq (1+\gamma \sqrt{C_mC_{\eta}})\|w-w_{\epol}\|_2.   \tag{From the proof of Lemma \ref{lem:boundedness}}
\end{align*}
Besides,
\begin{align*}
    \|w-w_{\epol}\|_2&\leq     \|w-\Bcal' w\|_2+\|\Bcal' w-\Bcal' w_{\epol}\|_2\\
    &\leq \|w-\Bcal' w\|_2+\gamma \|\Tcal'w-\Tcal' w_{\epol}\|_2\\
        &\leq \|w-\Bcal' w\|_2+\gamma \sqrt{C_mC_{\eta}}\|w-w_{\epol}\|_2. 
\end{align*}
Thus, $(1-\gamma \sqrt{C_mC_{\eta}})     \|w-w_{\epol}\|_2\leq  \|w-\Bcal' w\|_2$. 

\end{proof}

\subsection{Proof of  \cref{sec:sufficient}}

\begin{proof}[Proof of  \cref{thm:generalization}]
~

\textbf{Proof for $\hat J^{\MIL}_w$}

Estimators for the weight functions in MIL are defined as 
\begin{align*}
    &\hat w=\argmin_{w\in\Wbbb_1}\max_{q\in \Qbbb_1}|\E_n[f_{\MIL_1}(s,a,r,s';w,q)]|,\\
    & f_{\MIL_1}(s,a,r,s';w,q)=w(s,a)\{q(s,a)-\gamma q(s',\epol)\}-(1-\gamma)\E_{d_0}[q(s_0,\epol)]. 
\end{align*}
We write this solution to the above minimax optimization as $(\hat w,\hat q^{\dagger})$. Thus, 
\begin{align*}
   \hat q^{\dagger} \coloneqq\argmax_{q\in \Qbbb_1}|\E_n[f_{\MIL_1}(s,a,r,s';\hat w,q)]|. 
\end{align*}
We define $J^{\MIL}_w=\E[\hat w(s,a)r]$. Then, we have
\begin{align*}
    |J^{\MIL}_w-J |\leq  \max_{q\in \Qbbb_1}|\E[\{\hat w-w_{\epol})\}\Tcal_{\gamma} q]|+\min_{q \in \Qbbb_1}|\E[\{\hat w-w_{\epol})\}\{\Tcal_{\gamma}(q-q_{\epol})\}]|. 
\end{align*}
The following is the proof. We have for any $q\in \Qbbb_1$:
\begin{align*}
    &|J^{\MIL}_w-J |=\E[\hat w(s,a)r]-J| \\
    &=|\E[\hat w(s,a)\{q_{\epol}(s,a)-\gamma v_{\epol}(s')\}]-(1-\gamma)\E_{d_0}[q_{\epol}(s_0,\epol)]| \\
   &\leq |\E[\hat w(s,a)\{q(s,a)-\gamma v(s')\}]-(1-\gamma)\E_{d_0}[v(s_0)]| \\
   &+|\E[\hat w(s,a)\{-q(s,a)+q_{\epol}(s,a)+\gamma v(s')-\gamma v_{\epol}(s')\}]-(1-\gamma)\E_{d_0}[q_{\epol}(s_0,\epol)-v(s_0)]| \\
&=|\E[\hat w(s,a)\{q(s,a)-\gamma v(s')\}]-(1-\gamma)\E_{d_0}[v(s_0)]|| \\
   &+|\E[\{\hat w(s,a)-w_{\epol}(s,a)\}\{-q(s,a)+q_{\epol}(s,a)+\gamma v(s')-\gamma v_{\epol}(s')\}]|.
\end{align*}
Then, by defining 
\begin{align*}
    q^{\diamond} = \argmin_{q\in \Qbbb_1}|\E[\{\hat w(s,a)-w_{\epol}(s,a)\}\{-q(s,a)+q_{\epol}(s,a)+\gamma v(s')-\gamma v_{\epol}(s')\}]|.
\end{align*}
we have
\begin{align}
|J^{\MIL}_w-J | &\leq |\E[\hat w(s,a)\{q^{\diamond}(s,a)-\gamma v^{\diamond}(s')\}]-(1-\gamma)\E_{d_0}[v^{\diamond}(s_0)]| \\
   &+|\E[\{\hat w(s,a)-w_{\epol}(s,a)\}\{-q^{\diamond}(s,a)+q_{\epol}(s,a)+\gamma v^{\diamond}(s')-\gamma v_{\epol}(s')\}]| \\ 
&\leq  \max_{q\in \Qbbb_1}|\E[\hat w(s,a)\{q(s,a)-\gamma v(s')\}]-(1-\gamma)\E_{d_0}[v(s_0)]| \nonumber \\
   &+\min_{q\in \Qbbb_1}|\E[\{\hat w(s,a)-w_{\epol}(s,a)\}\{-q(s,a)+q_{\epol}(s,a)+\gamma v(s')-\gamma v_{\epol}(s')\}]|\nonumber \\
  &\leq \max_{q\in \Qbbb_1}|\E[\{\hat w-w_{\epol}\}\Tcal_{\gamma} q]|+\min_{q \in \Qbbb_1}|\E[\{\hat w-w_{\epol}\}\{\Tcal_{\gamma}(q-q_{\epol})\}]|.  \label{eq:wqreal}
\end{align}
Here, we use $\E[w_{\epol}(s,a)\{q(s,a)-\gamma v(s')\}]-(1-\gamma)\E_{d_0}[v(s_0)]=0$ from \citep{UeharaMasatoshi2019MWaQ}. Since the second term is zero under $q\in \Qbbb_1$, we further analyze the first term:
\begin{align*}
    \max_{q\in \Qbbb_1}|\E[\{\hat w-w_{\epol}\}\Tcal_{\gamma} q]|. 
\end{align*}

Recall we have 
\begin{align*}
\E[f_{\MIL_1}(w,q)]=\E[-(w-w_{\epol})\Tcal_{\gamma}q]. 
\end{align*}
Then, using $w_{\epol}\in \Wbbb_1,q_{\epol}\in \Qbbb_1$, from \cref{eq:wqreal}, 
\begin{align*}
|J^{\MIL}_w-J| \leq  \max_{q\in \Qbbb_1}|\E[(\hat{w}-w_{\epol})\Tcal_{\gamma}q]|=  |\E[(\hat{w}-w_{\epol})\Tcal_{\gamma}\tilde q]|, 
\end{align*}
where 
\begin{align*}
\tilde  q= \argmax_{q\in \Qbbb_1}|\E[(\hat{w}-w_{\epol})\Tcal_{\gamma}q]|.
\end{align*}

As a first step, we bound  $|\E[(\hat w-w_{\epol})\Tcal_{\gamma}\hat q^{\dagger}] |$: 
\begin{align*}
    &|\E[(\hat w-w_{\epol})\hat \Tcal_{\gamma}\hat q^{\dagger}]|=|\E[f_{\MIL_1}(\hat w,\hat q^{\dagger})]|=|\E[f_{\MIL_1}(\hat w,\hat q^{\dagger})]|-|\E[f_{\MIL_1}(w_{\epol},\hat q^{\dagger})]|\\
    &=|\E[f_{\MIL_1}(\hat w,\hat q^{\dagger})]|-|\E_n[f_{\MIL_1}(\hat w,\hat q^{\dagger})]|+|\E_n[f_{\MIL_1}(\hat w,\hat q^{\dagger})]|-|\E_n[f_{\MIL_1}(w_{\epol},\hat q^{\dagger})]|\\
    & \quad +|\E_n[f_{\MIL_1}(w_{\epol},\hat q^{\dagger})]|-|\E[f_{\MIL_1}(w_{\epol},\hat q^{\dagger})]|\\
    &\leq |\E[f_{\MIL_1}(\hat w,\hat q^{\dagger})]|-|\E_n[f_{\MIL_1}(\hat w,\hat q^{\dagger})]|+|\E_n[f_{\MIL_1}(w_{\epol},\hat q^{\dagger})]|-|\E[f_{\MIL_1}(w_{\epol},\hat q^{\dagger})]|\\
    &\leq 2\sup_{w\in \Wbbb_1,q\in \Qbbb_1}|(\E-\E_n)[ f_{\MIL_1}(w,q)]|. 
\end{align*}
Here, we use $|\E_n[f_{\MIL_1}(\hat w,\hat q^{\dagger})]|\leq  |\E_n[f_{\MIL_1}(w_{\epol},\hat q^{\dagger})]|$ and $|\E[f_{\MIL_1}(w_{\epol},\hat q^{\dagger})]|=0$. 

Second, we calculate the difference of $|\E[(\hat w-w_{\epol})\Tcal_{\gamma}\hat q^{\dagger}] |$ and $|\E[(\hat w-w_{\epol})\Tcal_{\gamma}\tilde q] |$. 
\begin{align*}
    &|\E[(\hat w-w_{\epol})\Tcal_{\gamma}\tilde q] |-|\E[(\hat w-w_{\epol})\Tcal_{\gamma}\hat q^{\dagger}] |= |\E[f_{\MIL_1}(\hat w,\tilde q)]|-|\E[f_{\MIL_1}(\hat w,\hat q^{\dagger})]|\\
    &= |\E[f_{\MIL_1}(\hat w,\tilde q)]|- |\E_n[f_{\MIL_1}(\hat w,\tilde q)]|+ |\E_n[f_{\MIL_1}(\hat w,\tilde q)]|- |\E_n[f_{\MIL_1}(\hat w,\hat q^{\dagger})]| \\
    & \quad +|\E_n[f_{\MIL_1}(\hat w,\hat q^{\dagger})]|-|\E[f_{\MIL_1}(\hat w,\hat q^{\dagger})]|\\
    &\leq |\E[f_{\MIL_1}(\hat w,\tilde q)]|- |\E_n[f_{\MIL_1}(\hat w,\tilde q)]|+|\E_n[f_{\MIL_1}(\hat w,\hat q^{\dagger})]|-|\E[f_{\MIL_1}(\hat w,\hat q^{\dagger})]|\\
  &\leq 2\sup_{w\in \Wbbb_1,q\in \Qbbb_1}|(\E-\E_n)[ f_{\MIL_1}(w,q)]|. 
\end{align*}
Here, we use $|\E_n[f_{\MIL_1}(\hat w,\tilde q)]|\leq |\E_n[f_{\MIL_1}(\hat w,\hat q^{\dagger})]|$. Thus, 
\begin{align*}
    &|\E[(\hat w-w_{\epol})\Tcal_{\gamma}\tilde q]| \\
    &= |\E[(\hat w-w_{\epol})\Tcal_{\gamma}\tilde q]|-|\E[(\hat w-w_{\epol})\Tcal_{\gamma}\hat q^{\dagger}]|+|\E[(\hat w-w_{\epol})\Tcal_{\gamma}\hat q^{\dagger}]|\\
    &\leq 4|\sup_{w\in \Wbbb_1,q\in \Qbbb_1}(\E-\E_n)[ f_{\MIL_1}(w,q)]|. 
\end{align*}
In addition, with probability $1-\delta$, 
\begin{align*}
    |\hat J^{\MIL}_w- J^{\MIL}_w|\leq \sup_{w\in \Wbbb_1}|(\E_n-\E)[w(s,a)r]|.
\end{align*}
Combining all results together and using the uniform law of large numbers,  with probability $1-\delta$, 
\begin{align*}
    |\hat J^{\MIL}_w-J|&\leq \sup_{w\in \Wbbb_1}|(\E_n-\E)[w(s,a)r]|+4|\sup_{w\in \Wbbb_1,q\in \Qbbb_1}(\E-\E_n)[ f_{\MIL_1}(w,q)]| \\
    &\lesssim R_{\max}\Rcal_n(\Wbbb_1)+R_{\max}\sqrt{\log(4/\delta)/n}+\Rcal_n(\Gcal(\Wbbb_1,\Qbbb_1)) +C_{\Wbbb_1}(C_{\Qbbb_1}+R_{\max})\sqrt{\log(4/\delta)/n}. 
\end{align*}

\textbf{Proof for $\hat J^{\MIL}_q$}

Estimators for q-functions in MIL are defined as 
\begin{align*}
   \hat q=\argmin_{q\in\Qbbb_2}\max_{w\in \Wbbb_2}|\E_n[f_{\MIL_2}(s,a,r,s';w,q)]|,\,f_{\MIL_2}(s,a,r,s';w,q)=w(s,a)\{r-q(s,a)+\gamma v(s')\}
\end{align*}
We write this solution as $(\hat q,\hat w^{\dagger})$, i.e., 
\begin{align*}
  \hat w^{\dagger}=\argmax_{w\in \Wbbb_2}|\E_n[f_{\MIL_2}(s,a,r,s'; w,\hat q)]|. 
\end{align*}
 We have 
\begin{align}\label{eq:mqlmql}
|\hat J^{\MIL}_q-J|&\leq  \max_{w\in \Wbbb_2}|\E[w(s,a)\Tcal_{\gamma}(\hat q-q_{\epol})]|+ \min_{w\in \Wbbb_2}|\E[(w-w_{\epol})\Tcal_{\gamma}(\hat q-q_{\epol})]|. 
\end{align}
This is proved as follows. We have for any $w\in \Wbbb_2$:
\begin{align*}
    |\hat J^{\MIL}_q-J|&=|(1-\gamma)\E_{d_0}[\hat q(s_0,\epol)]-\E[w_{\epol}(s,a)r]|\\
&=|\E[w_{\epol}(s,a)\{-\hat q(s,a)+\gamma \hat v(s')\}]-\E[w_{\epol}(s,a)r]|\\
    &=|\E[w_{\epol}(s,a)\{r-\hat q(s,a)+\gamma \hat v(s')\}]|\\
    &\leq |\E[w(s,a)\{r-\hat q(s,a)+\gamma \hat v(s')\}]|+|\E[(w-w_{\epol})\{r-\hat q(s,a)+\gamma \hat v(s')\}]|\\
    &= |\E[w \Tcal_{\gamma}(\hat q-q_{\epol})]|+|\E[(w-w_{\epol})\Tcal_{\gamma}(\hat q-q_{\epol})]|. 
\end{align*}
This concludes the proof of \cref{eq:mqlmql}. 

Then, recall 
\begin{align*}
    \E[ f_{\MIL_2}(s,a,r,s';w,q)]=\E[w\Tcal_{\gamma}(\hat q-q_{\epol})]. 
\end{align*}
Then, from the assumption and \eqref{eq:mqlmql}, we have 
\begin{align*}
      |\hat J^{\MIL}_q-J|\leq \max_{w\in \Wbbb_2}|\E[w(s,a)\Tcal_{\gamma}(\hat q-q_{\epol})]|=|\E[\tilde w\Tcal_{\gamma}(\hat q-q_{\epol})]|,
\end{align*}
where 
\begin{align*}
    \tilde w= \argmax_{w\in \Wbbb_2}|\E[w\Tcal_{\gamma}(\hat q-q_{\epol})]|.
\end{align*}

As a first step, we bound $|\E[\hat w^{\dagger}\Tcal_{\gamma}(\hat q-q_{\epol})]|$: 
\begin{align*}
   & |\E[\hat w^{\dagger}\Tcal_{\gamma}(\hat q-q_{\epol})]|=|\E[f_{\MIL_2}(\hat w^{\dagger},\hat q)]|=|\E[f_{\MIL_2}(\hat w^{\dagger},\hat q)]|-|\E[f_{\MIL_2}(\hat w^{\dagger},q_{\epol})]|\\
    &=|\E[f_{\MIL_2}(\hat w^{\dagger},\hat q)]|-|\E_n[f_{\MIL_2}(\hat w^{\dagger},\hat q)]|+|\E_n[f_{\MIL_2}(\hat w^{\dagger},\hat q)]|-|\E_n[f_{\MIL_2}(\hat w^{\dagger},q_{\epol})]| \\
    & \quad +|\E_n[f_{\MIL_2}(\hat w^{\dagger},q_{\epol})]|-|\E[f_{\MIL_2}(\hat w^{\dagger},q_{\epol})]|\\
    &\leq |\E[f_{\MIL_2}(\hat w^{\dagger},\hat q)]|-|\E_n[f_{\MIL_2}(\hat w^{\dagger},\hat q)]|+|\E_n[f_{\MIL_2}(\hat w^{\dagger},q_{\epol})]|-|\E[f_{\MIL_2}(\hat w^{\dagger},q_{\epol})]|\\
    &\leq 2|\sup_{w\in\Wbbb_2,q\in \Qbbb_2}(\E-\E_n)[f_{\MIL_2}(q,w)]| 
\end{align*}
Here, we use $|\E_n[f_{\MIL_2}(\hat w^{\dagger},\hat q)]|\leq |\E_n[f_{\MIL_2}(\hat w^{\dagger},q_{\epol})]|$ and  $|\E[f_{\MIL_2}(\hat w^{\dagger},q_{\epol})]|=0$. 

Second, we bound the difference of $ |\E[\hat w^{\dagger}\Tcal_{\gamma}(\hat q-q_{\epol})]|$ and $ |\E[\tilde w\Tcal_{\gamma}(\hat q-q_{\epol})]|$: 
\begin{align*}
    &|\E[\tilde w\Tcal_{\gamma}(\hat q-q_{\epol})]|-|\E[\hat w^{\dagger}\Tcal_{\gamma}(\hat q-q_{\epol})]|\\
    & =|\E[f_{\MIL_2}(\tilde w,\hat q)]|-|\E[f_{\MIL_2}(\hat w^{\dagger},\hat q)]|\\
    & =|\E[f_{\MIL_2}(\tilde w,\hat q)]|-|\E_n[f_{\MIL_2}(\tilde w,\hat q)]|+|\E_n[f_{\MIL_2}(\tilde w,\hat q)]|-|\E_n[f_{\MIL_2}(\hat w^{\dagger},\hat q)]|\\
    & \quad +|\E_n[f_{\MIL_2}(\hat w^{\dagger},\hat q)]| -|\E[f_{\MIL_2}(\hat w^{\dagger},\hat q)]|\\
  & \leq |\E[f_{\MIL_2}(\tilde w,\hat q)]|-|\E_n[f_{\MIL_2}(\tilde w,\hat q)]|+|\E_n[f_{\MIL_2}(\hat w^{\dagger},\hat q)]| -|\E[f_{\MIL_2}(\hat w^{\dagger},\hat q)]|\\
    &  \leq 2|\sup_{w\in\Wbbb_2,q\in \Qbbb_2}(\E-\E_n)[f_{\MIL_2}(q,w) ]| .
\end{align*}
Here, we use $|\E_n[f_{\MIL_2}(\tilde w,\hat q)]|\leq |\E_n[f_{\MIL_2}(\hat w^{\dagger},\hat q)]|$. In the end, 
\begin{align*}
        &\E[\tilde w\Tcal_{\gamma}(\hat q-q_{\epol})]|\\
        &\leq |\E[\tilde w\Tcal_{\gamma}(\hat q-q_{\epol})]|- |\E[\hat w^{\dagger}\Tcal_{\gamma}(\hat q-q_{\epol})]|+|\E[\hat w^{\dagger}\Tcal_{\gamma}(\hat q-q_{\epol})]|\\ 
        &\leq  4|\sup_{w\in\Wbbb_2,q\in \Qbbb_2}(\E-\E_n)[f_{\MIL_2}(q,w)]|. 
\end{align*}
This term is upper-bounded by $\Rcal_n(\Gcal(\Wbbb_2,\Qbbb_2))+C_{\Wbbb_2}C_{\Qbbb_2}\sqrt{\log(2/\delta)/n}$ up to universal constants. 

\textbf{Proof for $\hat J^{\MIL}_{wq}$}

Letting
\begin{align*}
    f(s,a,r,s';w,q)=w(s,a)\{r-q(s,a)+\gamma q(s',\epol)\}+(1-\gamma)\E_{d_0}[q(s_0,\epol)], 
\end{align*}
we define 
\begin{align*}
      J^{\MIL}_{wq}=\E[f(s,a,r,s';\hat w,\hat q)]. 
\end{align*}
Here, we have
\begin{align*}
    |\hat J^{\MIL}_{wq}-    J^{\MIL}_{wq}|\leq \sup_{w\in \Wbbb_1,q\in \Qbbb_2}|(\E-\E_n)[f(w,q)]|. 
\end{align*}
Besides, 
\begin{align*}
    \E[f(s,a,r,s';\hat w,\hat q)]-J=|\E[\{\hat w-w_{\epol}\}\{\Tcal_{\gamma}(\hat q-q_{\epol}\}] |. 
\end{align*}

\textbf{Case 1: when $q_{\epol}\in \Qbbb_2,w_{\epol}-\Wbbb_1\subset \Wbbb_2$ }
Thus, from a similar calculation as the above, 
\begin{align*}
    |  \E[f(s,a,r,s';\hat w,\hat q)]-J|&\leq \max_{w\in \Wbbb_2}|\E[w(s,a)\Tcal_{\gamma}(\hat q-q_{\epol}) ]|+\min_{w\in \Wbbb_2}|\E[\{w-w_{\epol}+\hat w\}\Tcal_{\gamma}(\hat q-q_{\epol}) ]|. 
\end{align*}
The second term is $0$ from the assumption. The first term is calculated as before. 

\textbf{Case 2: when $w_{\epol}\in \Wbbb_1,q_{\epol}-\Qbbb_2\subset \Qbbb_1$ }
Thus, from a similar calculation as the above, 
\begin{align*}
        |  \E[f(s,a,r,s';\hat w,\hat q)]-J|&\leq \max_{q\in \Qbbb_1}|\E[(\hat w-w_{\epol})\Tcal_{\gamma}q]|+\min_{q\in \Qbbb_1}|\E[(\hat w-w_{\epol})\Tcal_{\gamma}\{q-q_{\epol}+\hat q\}]|.
\end{align*}
The second term is $0$ from the assumption. The first term is calculated as before. 

Combining all things, with $1-\delta$,
\begin{align*}
      |\hat J^{\MIL}_{wq}-J|&\leq \sup_{w\in \Wbbb_1,q\in \Qbbb_2}|(\E-\E_n)[f(w,q)]|\\
      &+\max(4\sup_{w\in\Wbbb_1,q\in \Qbbb_1}|(\E-\E_n)[f_{\MIL_1}(q,w)]|, 4\sup_{w\in\Wbbb_2,q\in \Qbbb_2}|(\E-\E_n)[f_{\MIL_2}(q,w)]|)\\
      &\lesssim  \mathrm{Err}_{\MIL,wq}+\max(\mathrm{Err}_{\MIL,w},\mathrm{Err}_{\MIL,q}). 
\end{align*}

\end{proof}

\begin{proof}[Proof of Corollary \ref{cor:vc_cor}]
~\\
\textbf{First part}\\ 
First, we prove that the VC dimension of $\Vbbb\coloneqq \{q(s,\epol);q\in \Qbbb\}$ is upper-bounded by $V(\Qbbb)$. Consider the set of subgraphs of $\Vbbb$ as $\mathcal{V}$. Then, we take $n$ points $s_1,\cdots,s_n$ shattered by $\Vbbb$. This states for any binary sequence $\{b_1,\cdots,b_n\}$, there exists $v' \in \Vbbb$ and $t' \in \mathbb{R}$ s.t. 
\begin{align*}
    \rI(q_{v'}(s_1,\epol) - t' >0)=b_1,\cdots,     \rI(q_{v'}(s_n,\epol) - t'>0)=b_n,
\end{align*}
where $q_{v'} \in \Qbbb$ is a corresponding function such as $v'(s)=q_{v'}(s,\epol)$.
We, then, consider the set of subgraphs of $\Qbbb$ as $\mathcal{Q}$, and show that $\mathcal{Q}$ can shatter $(s_1,a_1),...,(s_n,a_n)$ associated with $s_1,...,s_n$.
To bound the VC dimension of $\Vbbb$, it is sufficient to show
\begin{align}
    \exists a, ~\mathrm{sign}(q_{v'}(s,\epol) - t') = \mathrm{sign}(q_{v'}(s,a) - t'), \label{eq:sign_q}
\end{align}
holds for any $s$.
To show \eqref{eq:sign_q}, without loss of generality, we assume $q_{v'}(s,\epol) \geq t'$, which is equivalent to
\begin{align}\label{eq:contra}
    \ts \int q_{v'}(s,a)\epol(a|s)\rd(a)>t'. 
\end{align}
We use the contradiction argument. Suppose there is no $a\in \Acal$ s.t. $ q_{v'}(s,a)>t'$, then 
\begin{align*}
   \ts  \int q_{v'}(s,a)\epol(a|s)\rd(a)\leq t',
\end{align*}
holds. 
This contradicts \cref{eq:contra}, hence there exists $a\in \Acal$ s.t. $      q_{v'}(s,a)>t' $.
Hence, we can take $\{a_1,\cdots,a_n\}$ as satisfying \eqref{eq:sign_q} for each $s_1,...,s_n$ and obtain
\begin{align*}
    \rI(q_{v'}(s_1,a_1) - t'>0)=b_1,\cdots,  \rI(q_{v'}(s_n,a_n) - t'>0)=b_n. 
\end{align*}
This concludes that $V(\Vbbb)\leq V(\Qbbb)$.

\textbf{Second part}
\begin{align*}
    \Rcal_n(\Gcal(\Wbbb,\Qbbb))&= \Rcal_n(\Wbbb \Qbbb )+\gamma \Rcal_n(\Wbbb \Vbbb )\\ 
    &\leq 0.25\{\Rcal_n((\Wbbb+\Qbbb)^2-(\Wbbb-\Qbbb)^2 )+\gamma \Rcal_n((\Wbbb+\Vbbb)^2-(\Wbbb-\Vbbb)^2 )\} \\ 
    &\leq \{C_{\Wbbb}C_{\Qbbb}\Rcal_n(\Wbbb+\Qbbb))+\gamma  C_{\Wbbb}C_{\Qbbb}\Rcal_n(\Wbbb+\Vbbb)\}\tag{Contraction property \citep{mendelson2002improving}} \\ 
\end{align*}
Here, we have 
\begin{align*}
    \Rcal_n(\Wbbb)&\ts =\int_{0}^{C_{\Wbbb}}\sqrt{\log \Ncal(t,\Wbbb,\|\cdot\|_n)/n}\rd(t)\\
    &\ts =O\braces{ \int_{0}^{C_{\Wbbb}}\sqrt{V(\Wbbb)\log(1/t) /n}\rd(t)   }\\
    &\ts=O(\sqrt{V(\Wbbb)/n}). 
\end{align*}
The covering number is calculated using \citep[Lemma 19.15]{VaartA.W.vander1998As}, that is, for a VC class of functions $\Fcal$ with measurable envelope function $F$ and $r\geq 1$, one has for any probability measure $Q$ with $\|F\|_{Q,r}>0$, 
\begin{align*}
    \Ncal(\epsilon \|F\|_{Q,r},\Fcal,L_r(Q))\lesssim V(\Fcal)(4e)^{V(\Fcal)}\prns{\frac{2}{\epsilon}}^{rV(\Fcal)}. 
\end{align*}
Thus, 
\begin{align*}
       R_{\max}\Rcal_n(\Wbbb)+\Rcal_n(\Gcal(\Wbbb,\Qbbb))= O(\sqrt{(V(\Wbbb)+V(\Qbbb))/n}). 
\end{align*}
This concludes that with probability $1-\delta$, 
\begin{align*}
    |\hat J^{\MIL}-J|= O(\sqrt{(V(\Wbbb)+V(\Qbbb))/n}).
\end{align*}
\end{proof}

\begin{proof}[Proof of Corollary \ref{cor:nonpara_cor}]
WLOG, we prove the case for $\hat J^{\MIL}_q$. From Lemma \ref{lem:covering_G}, we have 
\begin{align*}
		&\sqrt{\log \bN(\tau, \Gcal(\Wbbb,\Qbbb), \|\cdot\|_\infty)} \\
		&\leq \sqrt{\log(0.5\tau/((1+\gamma)C_\Qbbb), \Wbbb, \|\cdot\|_\infty) + \log \bN(0.5\tau/((1+\gamma)C_\Wbbb), \Qbbb, \|\cdot\|_\infty)}\\
		 &\leq\sqrt{\log(0.5\tau/((1+\gamma)C_\Qbbb), \Wbbb, \|\cdot\|_\infty)} + \sqrt{\log \bN(0.5\tau/((1+\gamma)C_\Wbbb), \Qbbb, \|\cdot\|_\infty)}. 
\end{align*}

First, we consider the case when $\beta<2$, from Lemma \ref{lem:dudley}, the final error is 
\begin{align*}
O\prns{n^{-1/2}\braces{1-\braces{\frac{1}{\tau}}^{0.5\beta-1}}+\tau}. 
\end{align*}
By taking $\tau=\Theta(1/n)$, the final error is $O(n^{-1/2})$. 

Next, we consider the case when $\beta=2$.  From Lemma \ref{lem:dudley}, the final error is 
\begin{align*}
     O\prns{n^{-1/2}\{1-\log(\tau)\}+\tau}.  
\end{align*}
By taking $\tau=\Theta(\log(n)/\sqrt{n})$, the final error is $O(n^{-1/2}\log n)$. 

Finally, consider the case  when $\beta\geq 2$.  From Lemma \ref{lem:dudley}, the final error is 
\begin{align*}
O\prns{n^{-1/2}\braces{\braces{\frac{1}{\tau}}^{0.5\beta-1}-1}+\tau}. 
\end{align*}
Here, by taking $\tau=\Theta(n^{-1/\beta})$, the final error is  $\tilde O(n^{-1/\beta})$. 
\end{proof}

\begin{proof}[Proof of Corollary \ref{cor:linear_sieve}]
WLOS, we prove the case for $\hat J^{\MIL}_q$. From \cref{thm:allow_mis1}, 
\begin{align*}
|\hat J^{\MIL}_q-J|&=O(R_{\max}\Rcal_n(\Wbbb)+\Rcal_n(\Gcal(\Wbbb,\Qbbb)))\\
&+\min_{w\in \Wbbb}\max_{q \in \Qbbb}\E[(w-w_{\epol})\Tcal_{\gamma}(q-q_{\epol})]+\min_{q \in \Qbbb}\max_{w\in \Wbbb}\E[w\Tcal_{\gamma}(q-q_{\epol})]).
\end{align*}
Then, from the assumption of sieve space $\Wbbb$, we have 
\begin{align*}
 &\min_{w\in \Wbbb}\max_{q \in \Qbbb}\E[(w-w_{\epol})\Tcal_{\gamma}(q-q_{\epol})]\\
 &\lesssim \min_{w \in \Wbbb}\|w-w_{\epol}\|_{\infty}(C_{\Qbbb}+R_{\max}(1-\gamma)^{-1}) \tag{from $\|\Tcal_{\gamma}(q-q_{\epol})\|_{1}\lesssim C_{\Qbbb}+R_{\max}(1-\gamma)^{-1})$}\\
 &= O(k^{-p/d}_n). 
\end{align*}
Similarly, from the assumption of sieve space $\Qbbb$, we have 
\begin{align*}
    \min_{q \in \Qbbb}\max_{w\in \Wbbb}\E[w\Tcal_{\gamma}(q-q_{\epol})]\leq (1+\gamma)C_{\Wbbb}\min_{q \in \Qbbb}\|q-q_{\epol}\|_{\infty}= O(k^{-p/d}_n). 
\end{align*}
In addition, we have $R_{\max}\Rcal_n(\Wbbb)+\Rcal_n(\Gcal(\Wbbb,\Qbbb))= O(\sqrt{k_n/n})$ from Corollary \ref{cor:vc_cor}. In the end, the final error becomes $$O(\ts \sqrt{k_n/n}+k^{-p/d}_n).$$  Balancing these two terms, the final error is $ O(n^{-\frac{p}{2p+d}})$. 
\end{proof}

\begin{proof}[Proof of Corollary \ref{cor:neural_net}]
WLOS, we prove the case for $\hat J^{\MIL}_q$. 
As a similar discussion as the proof of Corollary \ref{cor:linear_sieve}, $ |\hat J^{\MIL}-J|$ is upper-bounded by 
\begin{align*}
&O(R_{\max}\Rcal_n(\Wbbb))+\Rcal_n(\Gcal(\Wbbb,\Qbbb))+\min_{w\in \Wbbb}\max_{q \in \Qbbb}\E[(w-w_{\epol})\Tcal_{\gamma}(q-q_{\epol})]+\min_{q \in \Qbbb}\max_{w\in \Wbbb}\E[w\Tcal_{\gamma}(q-q_{\epol})]  \tag{  \cref{thm:allow_mis1}}\\ 
&=O(R_{\max}\Rcal_n(\Wbbb))+\Rcal_n(\Gcal(\Wbbb,\Qbbb))+\tilde{O}( \Xi^{-p/ d}) \tag{Lemma \ref{lem:yaro}} \\
&=O(\Rcal_n(\mathcal{F}_{NN}))+\tilde{O}( \Xi^{-p/ d}) \tag{ Contraction property } \\
&= O\prns{\sqrt{\Xi L/n}\int_{0}^{ C_{NN}} \sqrt{\log(LB(\Xi+1)/t)}\rd\mu(t)}+\tilde{O}( \Xi^{-p/ d}),  \tag{Lemma \ref{lem:covering_neural} and Dudley integral}
\end{align*}
where $\Fcal_{NN}\leq C_{NN} = \max \{ C_\Wbbb, C_\Qbbb\}$. Then, by taking $L=\Theta(\log n),\Xi=\Theta(n^{d/(2p+d)})$, the statement is concluded.  
\end{proof}

\begin{proof}[Proof of Corollary \ref{cor:neural_net_overparam}]
WLOS, we prove the case for $\hat J^{\MIL}_q$. We define $C_{NN} := \max \{ C_\Wbbb, C_\Qbbb\} $ which is finite by the boundedness of $A_\ell$ for $\mathcal{F}_{NN}$.
As a similar discussion as the proof of Corollary \ref{cor:linear_sieve},  $    |\hat J^{\MIL}-J|$ is upper-bounded by 
\begin{align*}
      &O(R_{\max}\Rcal_n(\Wbbb))+\Rcal_n(\Gcal(\Wbbb,\Qbbb))+\min_{w\in \Wbbb}\max_{q \in \Qbbb}\E[(w-w_{\epol})\Tcal_{\gamma}(q-q_{\epol})]+\min_{q \in \Qbbb}\max_{w\in \Wbbb}\E[w\Tcal_{\gamma}(q-q_{\epol})]  \tag{  \cref{thm:allow_mis1}}\\ 
        &=O(\Rcal_n(\mathcal{F}_{NN})+\epsilon)  \tag{ Contraction property and the assumption for the approximation error}.  
\end{align*}
We apply Lemma \ref{lem:golowich} and obtain the statement.
\end{proof}

\begin{proof}[Proof of \cref{thm:recovery1}]

~ \\ 
\textbf{Convergence rate of w-functions regarding  $\|\hat w-w_{\epol}\|_2$} 

Estimators for the weight functions in MIL are defined as 
\begin{align*}
    &\hat w=\argmin_{w\in\Wbbb_1}\max_{q\in \Qbbb_1}|\E_n[f_{\MIL_1}(s,a,r,s';w,q)]|,\\
    & f_{\MIL_1}(s,a,r,s';w,q)=w(s,a)\{q(s,a)-\gamma q(s',\epol)\}-(1-\gamma)\E_{d_0}[q(s_0,\epol)]. 
\end{align*}
We write this solution to the above minimax optimization as $\hat w$. In addition, we have 
\begin{align*}
    \E[ f_{\MIL_1}(s,a,r,s';w,q)]=-\E[\Tcal'_{\gamma}(w-w_{\epol})q]. 
\end{align*}
What we want to bound is $\E[(\hat w-w_{\epol})^2]$:
\begin{align*}
    &\E[(\hat w-w_{\epol})^2]\leq  \max_{q\in \Qbbb_1}|\E[C^{-1}_{\xi}(\hat w-w_{\epol})\Tcal_{\gamma}q]|=  |\E[C^{-1}_{\xi}\{(\hat w-w_{\epol})\}\Tcal_{\gamma}\tilde q ]|,
\end{align*}
where 
\begin{align*}
\tilde  q= \argmax_{q\in \Qbbb_1}|\E[(\hat w-w_{\epol})\Tcal_{\gamma}q]|.
\end{align*}
Here, we use $C_{\xi}(\Wbbb_1-w_{\epol})\subset \Tcal_{\gamma}\Qbbb_1$. Therefore, if we can bound $|\E[\Tcal'_{\gamma}(\hat w-w_{\epol})\tilde q ]|$, the convergence rate is obtained. This is bounded as in the proof of \cref{thm:generalization}.  Thus, with probability $1-\delta$, 
\begin{align*}
    \E[\{(\hat w-w_{\epol})\}^2]\leq C^{-1}_{\xi}\mathrm{Err}_{\MIL,w}. 
\end{align*}

\textbf{Convergence rate of w-functions regarding $\|\Bcal'\hat w-\hat w\|_2$} 

What we want to bound is $\E[(\Bcal'\hat w-\hat w)^2]$. Since $\Bcal' w-w=\Tcal'_{\gamma}(w-w_{\epol})$, we bound $\E[\{\Tcal'_{\gamma}(\hat w-w_{\epol})\}^2]$:
\begin{align*}
    &\E[\{\Tcal'_{\gamma}(\hat w-w_{\epol})\}^2]\leq  \max_{q\in \Qbbb_1}|\E[\{\Tcal'_{\gamma}(\hat w-w_{\epol})\}C^{-1}_{\xi}q]|=  |\E[C^{-1}_{\xi}\{\Tcal'_{\gamma}(\hat w-w_{\epol})\}\tilde q ]|,
\end{align*}
where 
\begin{align*}
\tilde  q= \argmax_{q\in \Qbbb_1}|\E[\Tcal'_{\gamma}(\hat w-w_{\epol})q]|.
\end{align*}
Here, we use $C_{\xi}\Tcal'_{\gamma}(\Wbbb_1-w_{\epol})\subset \Qbbb_1$. Therefore, if we can bound $|\E[C^{-1}_{\xi}\Tcal'_{\gamma}(\hat w-w_{\epol})\tilde q ]|$, the convergence rate is obtained. This is bounded as in the proof of \cref{thm:generalization}.  Thus, with probability $1-\delta$, 
\begin{align*}
    \E[\{\Tcal'_{\gamma}(\hat w-w_{\epol})\}^2]=\E[\{\Bcal'\hat w-\hat w\}^2]\leq C^{-1}_{\xi}\mathrm{Err}_{\MIL,w}. 
\end{align*}

\textbf{Convergence rate of q-functions regarding $\|\Bcal \hat q-\hat q\|_2$}

Recall that estimators for q-functions in MIL are defined as 
\begin{align*}
   \hat q=\argmin_{q\in\Qbbb_2}\max_{w\in \Wbbb_2}|\E_n[f_{\MIL_2}(s,a,r,s';w,q)]|,\,f_{\MIL_2}(s,a,r,s';w,q)=w(s,a)\{r-q(s,a)+\gamma v(s')\}
\end{align*}
We write this solution to $\hat q$. In addition, we have 
\begin{align*}
    \E[ f_{\MIL_2}(s,a,r,s';w,\hat q)]=\E[w\Tcal_{\gamma}(\hat q-q_{\epol})]. 
\end{align*}

What we want to bound is $\E[\{\Bcal \hat q-\hat q\})^2]$. Since $\Bcal q-q=\Tcal_{\gamma}(q-q_{\epol})$, we bound $\E[\{\Tcal_{\gamma}(\hat q-q_{\epol})\}^2]$: 
\begin{align*}
    &\E[\{\Tcal_{\gamma}(\hat q-q_{\epol})\}^2]\leq C^{-1}_{\xi}\max_{w \in \Wbbb_2}|\E[w\Tcal_{\gamma}(\hat q-q_{\epol})]|=C^{-1}_{\xi}| \E[\tilde w\Tcal_{\gamma}(\hat q-q_{\epol})]|,
\end{align*}
where 
\begin{align*}
    \tilde w= \argmax_{w\in \Wbbb_2}|\E[w\Tcal_{\gamma}(\hat q-q_{\epol})]|.
\end{align*}
Here, we use $C_{\xi}\Tcal_{\gamma}(\Qbbb_2-q_{\epol})\subset \Wbbb_2$. Therefore, if we can bound $|\E[\tilde w\Tcal_{\gamma}(\hat q-q_{\epol})]|$, the convergence rate is obtained. This is bounded as in the proof of \cref{thm:generalization}.  Thus, with probability $1-\delta$, 
\begin{align*}
    \E[\{\Bcal \hat q-\hat q\}^2]\leq C^{-1}_{\xi}\mathrm{Err}_{\MIL,q}. 
\end{align*}

\textbf{Convergence rate of q-functions regarding $\|\hat q-q_{\epol}\|_2$}

We have 
\begin{align*}
    &\E[\{\hat q-q_{\epol}\}^2]\leq \max_{w \in \Wbbb_2}C^{-1}_{\xi}|\E[\Tcal'_{\gamma}w(\hat q-q_{\epol})]|=C^{-1}_{\xi}| \E[\tilde w\Tcal_{\gamma}(\hat q-q_{\epol})]|,
\end{align*}
where 
\begin{align*}
    \tilde w= \argmax_{w\in \Wbbb_2}|\E[w\Tcal_{\gamma}(\hat q-q_{\epol})]|.
\end{align*}
Here, we use $C_{\xi}(\Qbbb_2-q_{\epol})\subset \Tcal'_{\gamma}\Wbbb_2$. Therefore, if we can bound $|\E[\tilde w\Tcal_{\gamma}(\hat q-q_{\epol})]|$, the convergence rate is obtained. This is bounded as in the proof of \cref{thm:generalization}.  Thus, with probability $1-\delta$, 
\begin{align*}
    \E[\{\hat q-q_{\epol}\}^2]\leq C^{-1}_{\xi}\mathrm{Err}_{\MIL,q}. 
\end{align*}

\end{proof}

\begin{proof}[Proof of Lemma \ref{lem:q_completenss}]

~\\ 
\textbf{Statement \RNum{1}: $q_{\epol}\in \Qbbb$}\\ 
Recall 
\begin{align*}
    \Tcal f(s,a)=\int f(s',\epol)P_{S'|S,A}(s'|s,a)\rd\mu(s'). 
\end{align*}
Thus, for any $q\in \Qbbb$, $\Bcal q$ belongs to $\Qbbb$. Noting $\Bcal q_{\epol}=q_{\epol}$, $q_{\epol}$ also belongs to $\Qbbb$. We write $q_{\epol}=\{\theta^{*}\}^{\top}\phi$. 

~\textbf{Statement \RNum{2}: $\Tcal_{\gamma}(\Qbbb-q_{\epol})\subset \Qbbb$}

We show there exists $M_{\epol}$ s.t. $\Tcal(\beta^{\top}\phi)=(M_{\epol}\beta)^{\top}\phi$. This is proved by 
\begin{align*}\ts 
  \Tcal(\beta^{\top}\phi)=\beta^{\top}\{\int \phi(s',\epol)\vartheta(s')^{\top}\rd\mu(s')\}\phi. 
\end{align*}
Thus, for any $\theta \in \mathbb{R}^d$, 
\begin{align*}
   ( \gamma \Tcal-I)\{(\theta-\theta^{*})^{\top}\phi\}&=\gamma (\theta-\theta^{*})^{\top} M^{\top}_{\epol}\phi-(\theta-\theta^{*})^{\top}\phi\\
   &=\{(\gamma M_{\epol}-I)(\theta-\theta^{*}) \}^{\top}\phi. 
\end{align*}
This implies $\Tcal_{\gamma}(\Qbbb-q_{\epol})\subset \Qbbb$. Recall $\Qbbb=\{(s,a)\mapsto \theta^{\top} \phi(s,a);\theta \in \mathbb{R}^d\}$. 


\textbf{Statement \RNum{3}: $\Qbbb-q_{\epol} \subset \Tcal'_{\gamma}\Qbbb$}

From \cref{exa:ver1} and \eqref{eq:linear_ver1}, given $f_2=\theta \cdot \phi$, we have
\begin{align*}
\Tcal'_{\gamma} f_2=\{\E[\phi\phi^{\top}]^{-1}(\gamma M^{\top}_{\epol}-I)\E[\phi\phi^{\top}] \}\theta \cdot \phi. 
\end{align*}
Since $I-\gamma M^{\top}_{\epol}$ is invertible from the assumption $1/\gamma \notin \sigma(M_{\epol})$, for any $q-q_{\epol}=(\theta-\theta^{*})\cdot \phi\in \Qbbb-q_{\epol}$, we can take 
\begin{align*}
      q'=\{\E[\phi\phi^{\top}]^{-1}(\gamma M^{\top}_{\epol}-I)\E[\phi\phi^{\top}] \}^{-1}(\theta-\theta^*) \cdot \phi
\end{align*}
s.t. $\Tcal'_{\gamma} q'=q-q_{\epol}$
This implies $\Qbbb-q_{\epol} \subset \Tcal'_{\gamma}\Qbbb$

~\textbf{Statement \RNum{4}: Tabular case}

We use a matrix formulation following \cref{exa:tabular}. In a tabular case, $M_{\epol}=P_{\epol}$. The spectrum of $P_{\epol}$ is less than or equal to $1$. Thus, $\Qbbb-q_{\epol} \subset \Tcal'_{\gamma}\Qbbb$ holds from the statement \RNum{3}. Besides, $\Qbbb=\{\theta^{\top} \phi;\theta \in \mathbb{R}^d\}$ from the statement \RNum{2}. 

\end{proof}

\begin{proof}[Proof of Lemma \ref{lem:w_completenss}]

From Lemma \ref{lem:form}, recall 
\begin{align*}
 \Tcal' f(s,a)=\eta_{\epol}(s,a)\int f(s',a')P_{S,A|S'}(s',a'|s)\rd\mu(s',a'). 
\end{align*}
First, by using the assumption $P_{S,A|S'}(s',a'|\cdot)=\psi(\cdot)\cdot \vartheta(s',a')$, $ \Tcal' f$ belongs to 
\begin{align*}
   \Wbbb= \{(s,a)\mapsto \eta_{\epol}(s,a)\times (\psi(s)\cdot \beta);\beta \in \mathbb{R}^d\}. 
\end{align*}
Especially, $\Tcal'\{\beta^{\top}\eta\psi\}=(M'_{\epol}\beta)^{\top}\eta\psi$ holds by taking $M'^{\top}_{\epol}$ s.t. 
\begin{align*}
    \int \eta_{\epol}(s',a')\psi(s')\vartheta^{\top}(s',a')\rd\mu(s',a'). 
\end{align*}

~\textbf{Statement \RNum{1} : $w_{\epol}\in \Wbbb$}

Noting $\Bcal' w=\gamma \Tcal'w+(1-\gamma)\frac{d_0(s)\epol(a|s)}{\bpol(a|s)p_{S}(s)}$ and the assumption $d_0(s)/p_{S}(s)$ belongs to $\{\psi(s)\cdot \beta);\beta \in \mathbb{R}^d \}$, for any $w$, $\Bcal' w$ belongs to $\Wbbb$. Since $\Bcal' w_{\epol}=w_{\epol}$, $w_{\epol}$ also belongs to $\Wbbb$.  We define $\theta^{*}$ s.t. $w_{\epol}=\eta \{\theta^{*}\}^{\top}\psi$. 

~\textbf{Statement \RNum{2}: $\Tcal'_{\gamma}(\Wbbb-w_{\epol})\subset \Wbbb$ }

For any $\theta\in \mathbb{R}^d$, 
\begin{align*}
    (\gamma \Tcal'-I)\{\eta\theta^{\top}\psi-\eta \{\theta^{*}\}^{\top}\psi\}=  (\gamma M'_{\epol}-I)(\theta-\theta^{*}) \cdot \eta\psi. 
\end{align*}
Thus, $\Tcal'_{\gamma}(\Wbbb-w_{\epol})\subset \Wbbb$ is concluded. 

 ~\textbf{Statement \RNum{3}: $(\Wbbb-w_{\epol})\subset \Tcal_{\gamma}\Wbbb$}
 
 From \cref{exa:ver2} and \eqref{eq:linear_ver2}, for $f_2=\theta\cdot \eta\psi$ we have
 \begin{align*}
     \Tcal_{\gamma}f_2=\{\E[\psi\psi^{\top}]^{-1}(\gamma M'^{\top}_{\epol}-I)\E[\psi\psi^{\top}]\}\theta \cdot \eta\psi
 \end{align*}
Since $(I-\gamma M'^{\top}_{\epol})$ is non-singular from the assumption $1/\gamma \notin \sigma(M_{\epol})$, for any $w-w_{\epol}=(\theta-\theta^{*})\cdot \eta\psi$, we can take $w'$
\begin{align*}
    w'=\{\E[\psi\psi^{\top}]^{-1}(\gamma M'^{\top}_{\epol}-I)\E[\psi\psi^{\top}]\}^{-1}(\theta-\theta^{*})\cdot \eta\psi
\end{align*}
s.t. $w-w_{\epol}=\Tcal_{\gamma}w'$. This implies  $(\Wbbb-w_{\epol})\subset \Tcal_{\gamma}\Wbbb$. 

 

~\textbf{Statement \RNum{4}: Tabular case}
We use a matrix formulation following \cref{exa:tabular}. W-\comp, $\Tcal_{\gamma}(\Wbbb-w_{\epol}) \subset \Wbbb$ , obviously holds. Thus, we prove adjoint w-\comp. Recall $\E[\psi\psi^{\top}]$ is non-singular, and we have 
\begin{align*}
 [\Tcal'f]=( \E[\psi\psi^{\top}]^{-1}(P^{\epol})^{\top} \E[\psi\psi^{\top}])[f]
\end{align*}
and 
\begin{align*}
   [\Tcal_{\gamma}f]=[(\gamma \Tcal'-I)f]=(\E[\psi\psi^{\top}]^{-1}(\gamma P^{\epol}-I)^{\top} \E[\psi\psi^{\top}])[f]. 
\end{align*}

Thus, for given any $f\in \Wbbb-w_{\epol}$, we can take a vector $$f':=\{ (\E[\psi\psi^{\top}]^{-1}(\gamma P^{\epol}-I)^{\top} \E[\psi\psi^{\top}])^{-1} [f]\}\cdot \psi$$ s.t. 
\begin{align*}
    (\gamma \Tcal'-I)f'=f. 
\end{align*}
Note $\gamma P^{\epol}-I$ is non-singular. Therefore, the adjoint w-\comp, $(\Wbbb-w_{\epol}) \subset \Tcal_{\gamma}\Wbbb$ holds.

\end{proof}

\begin{proof}[Proof of Corollary \ref{cor:implication}]
~ \\
\textbf{First Statement }
Define $J^{\MIL}_{w}\coloneqq \E[\hat w r]$. Then, 
\begin{align*}
      |J^{\MIL}_{w}-J| &=|\E[(\hat w-w_{\epol})\Tcal_{\gamma}q_{\epol}]|=|\E[\Tcal'_{\gamma}(\hat w-w_{\epol})q_{\epol}]|\leq \|\Tcal'_{\gamma}(\hat w-w_{\epol})\|_2\|\|q_{\epol}\|_2,\\
        |J^{\MIL}_{w}-J|&=|\E[(\hat w-w_{\epol})\Tcal_{\gamma}q_{\epol}]|\leq \|(\hat w-w_{\epol})\|_2\|\|\Tcal_{\gamma}q_{\epol}\|_2. 
\end{align*}
Thus, when $w_{\epol}\in \Wbbb_1,C_{\xi}\Tcal'_{\gamma}(\Wbbb_1-w_{\epol})\subset \Qbbb_1$, with $1-\delta$, 
\begin{align*}
    |\hat J^{\MIL}_{w}-J|&\leq |\hat J^{\MIL}_{w}-J^{\MIL}_{w} |+| J^{\MIL}_{w}-J|\\
    &\leq  |\hat J^{\MIL}_{w}-J^{\MIL}_{w} |+\|\Tcal'_{\gamma}(\hat w-w_{\epol})\|_2\|\|q_{\epol}\|_2\\
    &=\Error'_{\MIL,w}+\sqrt{C^{-1}_{\xi}\Error_{\MIL,w}}\|q_{\epol}\|_2. 
\end{align*}
Similarly,  when $q_{\epol}\in \Qbbb_2,C_{\xi}(\Wbbb_1-w_{\epol})\subset \Tcal_{\gamma}\Qbbb_1$, with $1-\delta$, we have 
\begin{align*}
        |\hat J^{\MIL}_{w}-J|&\leq \Error'_{\MIL,w}+\sqrt{C^{-1}_{\xi}\Error_{\MIL,w}}\|\Tcal_{\gamma} q_{\epol}\|_2. 
\end{align*}

~
\textbf{Second Statement}
When $q_{\epol}\in \Qbbb_2,C_{\xi}\Tcal_{\gamma}(\Qbbb_2-q_{\epol})\subset \Wbbb_2$, with $1-\delta$, we have 
\begin{align*}
|\hat J^{\MIL}_{q}-J|&=|\E[w_{\epol}\Tcal_{\gamma}(\hat q-q_{\epol})]|\leq \|w_{\epol}\|_2\|\|\Tcal_{\gamma}(\hat q-q_{\epol})\|_2\\
&\leq \|w_{\epol}\|_2\sqrt{C^{-1}_{\xi}\Error_{\MIL,q}}. 
\end{align*}

When $q_{\epol}\in \Qbbb_2,C_{\xi}(\Qbbb_2-q_{\epol})\subset \Tcal'_{\gamma}\Wbbb_2$, with $1-\delta$, we have 
\begin{align*}
|\hat J^{\MIL}_{q}-J|&=|\E[\Tcal'_{\gamma}w_{\epol}(\hat q-q_{\epol})]|\leq \|\Tcal'_{\gamma}w_{\epol}\|_2\|\|(\hat q-q_{\epol})\|_2\\
&\leq \|\Tcal'_{\gamma}w_{\epol}\|_2\sqrt{C^{-1}_{\xi}\Error_{\MIL,q}}. 
\end{align*}
~
\textbf{Third Statement}

We define $$J^{\MIL}_{wq}\coloneqq \E[\hat w \{r-\hat q(s,a)+\gamma \hat q(s',\epol)\}+(1-\gamma)\E_{d_0}[\hat q(s_0,\epol)].$$ 
When we have $q_{\epol}\in \Qbbb_2, C_{\xi}\Tcal_{\gamma}(\Qbbb_2-q_{\epol})\subset \Wbbb_2 $, with $1-\delta$, 
\begin{align*}
|J^{\MIL}_{wq}-J|&=|\E[\{w_{\epol}-w\}\Tcal_{\gamma}(\hat q-q_{\epol})]|\leq \|w_{\epol}-\hat w\|_2\|\|\Tcal_{\gamma}(\hat q-q_{\epol})\|_2\\
&\leq \sup_{w\in \Wbbb_1}\|w-w_{\epol}\|_2 \|\Tcal_{\gamma}(\hat q-q_{\epol})\|_2\lesssim \sup_{w\in \Wbbb_1}\|w-w_{\epol}\|_2\sqrt{C^{-1}_{\xi}\Error_{\MIL,q}}
\end{align*}
When we have $w_{\epol}\in \Wbbb_1,C_{\xi}\Tcal'_{\gamma}(\Wbbb_1-w_{\epol})\subset \Qbbb_1 $,   with $1-\delta$, 
\begin{align*}
 |J^{\MIL}_{wq}-J|&=  |\E[\Tcal'_{\gamma}(\hat w-w_{\epol})(\hat q-q_{\epol})]|\leq \|\Tcal'_{\gamma}(w_{\epol}-\hat w)\|_2\|\|\hat q-q_{\epol}\|_2\\
 &\leq \sup_{q\in \Qbbb_1}\|q-q_{\epol}\|_2 \|\Tcal'_{\gamma}(\hat w-w_{\epol})\|_2\lesssim \sup_{q\in \Qbbb_2}\|q-q_{\epol}\|_2  \sqrt{C^{-1}_{\xi}\Error_{\MIL,w}}
\end{align*}
Noting 
$|\hat J^{\MIL}_{wq}-J^{\MIL}_{wq}|\lesssim \Error_{\MIL,wq}$, the statement is proved as 
\begin{align*}
    |\hat J^{\MIL}_{wq}-J|\leq \Error_{\MIL,wq}+C^{-1/2}_{\xi}\max(\sup_{w\in \Wbbb_1}\|w-w_{\epol}\|_2\sqrt{\Error_{\MIL,q}}, \sup_{q\in \Qbbb_2}\|q-q_{\epol}\|_2  \sqrt{\Error_{\MIL,w}} ). 
\end{align*}
This immediately implies 
\begin{align*}
    |\hat J^{\MIL}_{wq}-J|\leq \Error_{\MIL,wq}+C^{-1/2}_{\xi}\max(C_{\Wbbb}\sqrt{\Error_{\MIL,q}}, C_{\Qbbb}\sqrt{\Error_{\MIL,w}} ). 
\end{align*}
Regarding the statement related to adjoint completeness, the statement similarly holds. 

\end{proof}

\subsection{Proof of \cref{sec:fast_slow}}


In the following we will frequently invoke two key lemmas. First, recall that the empirical localized Rademacher complexity of the function class is $\Gcal:V\to [-c,c]$ for a given set of samples $S=\{v_i\}$ is defined as 
\begin{align*}
   \Rcal_n(\eta;\Gcal)=\E_{\epsilon}\bracks{\sup_{g\in \Gcal: \|g\|_{n}\leq \eta }\frac{1}{n}\sum_i \epsilon_i g(v_i)}.
\end{align*}
The critical radius of $\Gcal$ is defined as a solution to $ \Rcal_n(\eta;\Gcal)\leq \eta^2/c$. 

\begin{lemma}[Theorem 14.1, \cite{WainwrightMartinJ2019HS:A} ]\label{lem:support1}
Given a star-shaped and $b$-uniformly bounded function class $\Gcal$, let $\eta_n$ be any positive solution of the inequality $R_s(\eta;\Gcal)\leq \eta^2/b$. Then, for any $t\geq \eta_n$, we have 
\begin{align*}
    |\|g\|^2_n-\|g\|^2_2  |\leq 0.5\|g\|^2_2+0.5t^2, \forall g\in \Gcal. 
\end{align*}
\end{lemma}

Consider a function class $\Fcal:X\to \mathbb{R}$ with loss $l:\mathbb{R}\times Z\to \mathbb{R}$. 

\begin{lemma}[Lemma 11 \citep{FosterDylanJ.2019OSL}]\label{lem:support2}
Assume $\sup_{f\in \Fcal}\|f\|_{\infty}\leq c$ and pick any $f^{*}\in \Fcal$. We define $\eta$ be a solution to 
\begin{align*}
   \Rcal_n(\eta;\mathrm{star}(\Fcal-f^*))\leq \eta^2/c. 
\end{align*}
Moreover, assume that the loss $l(\cdot,\cdot)$ is $L$-Lipschitz in the first argument. Then, with $1-\delta$,
\begin{align*}
     |(\E_n[l(f(x),z)]-\E_n[l(f^{*}(x),z)])-(\E[l(f(x),z)]-\E[l(f^{*}(x),z)])|\leq L \eta_n (\|f-f^{*}\|_2+\eta_n),\,\forall f\in \Fcal, 
\end{align*}
where $\eta_n=\eta+c_0\sqrt{\log(c_1/\delta)/n}$ for some universal constants $c_0,c_1$. 
\end{lemma}


\begin{proof}[Proof of \cref{thm:fast_w}]
~ \\ 
\textbf{Preliminary}
Define 
\begin{align*}
    \Xi(w,q)  &=\E[w(s,a)\{-q(s,a)+\gamma q(s',\epol) \}+(1-\gamma)\E_{d_0}[q(s_0,\epol)] ],\\ 
    \Xi_n(w,q)&=\E_n[w(s,a)\{-q(s,a)+\gamma q(s',\epol) \}+(1-\gamma)\E_{d_0}[q(s_0,\epol)] ],\\
    \Xi^{\lambda}(w,q) &=    \Xi^{\lambda}(w,q)-\lambda \|q(s,a)-\gamma q(s',\epol) \|^2_2,\\
    \Xi^{\lambda}_n(w,q) &=    \Xi^{\lambda}_n(w,q)-\lambda\|q(s,a)-\gamma q(s',\epol) \|^2_{2,n}.
\end{align*}
From Lemma \ref{lem:support1}, we have
\begin{align}\label{eq:q_uni}
    \forall q\in \Qbbb, |\|-q(s,a)+\gamma q(s',\epol) \|^2_{2,n}-\|-q(s,a)+\gamma q(s',\epol)\|^2_2  |\leq 0.5\|-q(s,a)+\gamma q(s',\epol)\|^2_2+0.5\eta^2, 
\end{align}
for our choice of $\eta:=\eta_n+c_0\sqrt{\log(c_1/\delta)/n}$ noting $\eta_n$ upper bounds the critical of $\Gcal_{w1}$. 

\textbf{First Step}
By definition of $\hat w_{\MIL}$ and $w_{\epol}\in \Wbbb$, we have 
\begin{align}
    \sup_{q\in \Qbbb}\Xi^{\lambda}_n(\hat w_{\MIL}; q)\leq     \sup_{q\in \Qbbb}\Xi^{\lambda}_n(w_{\epol}; q). 
\end{align}
From Lemma \ref{lem:support2}, with probability $1-\delta$, we have 
\begin{align*}
   \forall q\in \Qbbb: |\Xi_n(w_{\epol};q)- \Xi(w_{\epol};q)|\lesssim C_{\Wbbb}\{\eta \|q(s,a)-\gamma q(s',\epol)\|_2 +\eta^2 \}. 
\end{align*}
noting $\eta$ upper-bounds the critical radius of $\Gcal_{w1}$. Here, we use $l(a_1,a_2) (a_1=-q'(s,a)+\gamma q(s',\epol),a_2=w(s,a))$ is $C_{\Wbbb}$-Lipschitz with respect to $a_1$ noting $w(s,a)$ is in $[-C_{\Wbbb}, C_{\Wbbb}]$. More specifically, 
\begin{align*}
    | l(a_1,a_2)- l(a'_1,a_2) |\leq |a_1-a'_1|C_{\Wbbb}. 
\end{align*}
Thus,
{\small 
\begin{align*}
    \sup_{q\in \Qbbb}  \Xi^{\lambda}_n(w_{\epol},q)&=  \sup_{q\in \Qbbb}\{  \Xi_n(w_{\epol},q)-\lambda \|-q(s,a)+\gamma q(s',\epol) \|^2_{2,n} \}\\
    &\leq   \sup_{q\in \Qbbb}\{  \Xi(w_{\epol},q)+C_{\Wbbb}\{\eta \|-q(s,a)+\gamma q(s',\epol)\|_2 +\eta^2 \}  -\lambda \|-q(s,a)+\gamma q(s',\epol) \|^2_{2,n} \}     \\
   &\leq   \sup_{q\in \Qbbb}\{  \Xi(w_{\epol},q)+C_{\Wbbb}\{\eta \|-q(s,a)+\gamma q(s',\epol)\|_2 +\eta^2 \}  -0.5\lambda \|-q(s,a)+\gamma q(s',\epol) \|^2_{2}+\lambda \eta^2\}     \\
     &\lesssim   \sup_{q\in \Qbbb}\{  \Xi(w_{\epol},q)+(\lambda +C^2_{\Wbbb}/\lambda+C_{\Wbbb})\eta^2\}.
\end{align*}
}
In the last line, we use a general inequality $a,b\geq 0$, 
\begin{align*}
    \sup_{f\in \Fcal}(a\|f\|-b\|f\|^2)\leq a^2/4b. 
\end{align*}
Moreover, 
{\small 
\begin{align*}
        &\sup_{q\in \Qbbb}  \Xi^{\lambda}_n(\hat w_{\MIL},q)\\
    &= \sup_{q\in \Qbbb}\{\Xi_n(\hat w_{\MIL},q)-\Xi_n(w_{\epol},q)+\Xi_n(w_{\epol},q)-\lambda \|-q(s,a)+\gamma q(s',\epol) \|^2_{2,n} \}\\
   &\geq \sup_{q\in \Qbbb}\{\Xi_n(\hat w_{\MIL},q)-\Xi_n(w_{\epol},q)-2\lambda \|-q(s,a)+\gamma q(s',\epol) \|^2_{2,n} \}+\inf_{q\in \Qbbb}\{\Xi_n(w_{\epol},q)+\lambda \|q(s,a)+\gamma q(s',\epol) \|^2_{2,n} \}\\
    &= \sup_{q\in \Qbbb}\{\Xi_n(\hat w_{\MIL},q)-\Xi_n(w_{\epol},q)-2\lambda \|-q(s,a)+\gamma q(s',\epol) \|^2_{2,n} \}+\inf_{-q\in \Qbbb}\{\Xi(w_{\epol},-q)+\lambda \|-q(s,a)+\gamma q(s',\epol) \|^2_{2,n} \}\\
   &= \sup_{q\in \Qbbb}\{\Xi_n(\hat w_{\MIL},q)-\Xi_n(w_{\epol},q)-2\lambda \|-q(s,a)+\gamma q(s',\epol) \|^2_{2,n} \}+\inf_{-q\in \Qbbb}\{-\Xi_n(w_{\epol},q)+\lambda \|-q(s,a)+\gamma q(s',\epol) \|^2_{2,n} \}\\
     &= \sup_{q\in \Qbbb}\{\Xi_n(\hat w_{\MIL},q)-\Xi_n(w_{\epol},q)-2\lambda \|-q(s,a)+\gamma q(s',\epol) \|^2_{2,n} \}-\sup_{-q\in \Qbbb}\{\Xi_n(w_{\epol},q)-\lambda \|-q(s,a)+\gamma q(s',\epol) \|^2_{2,n} \}\\
     &= \sup_{q\in \Qbbb}\{\Xi_n(\hat w_{\MIL},q)-\Xi_n(w_{\epol},q)-2\lambda \|-q(s,a)+\gamma q(s',\epol) \|^2_{2,n} \}-\sup_{q\in \Qbbb}\{     \Xi^{\lambda}_n(w_{\epol},q) \}. 
\end{align*}
}
Here, we use $\Qbbb$ is symmetric. Therefore, 
\begin{align*}
        \sup_{q\in \Qbbb}\{\Xi_n(\hat w_{\MIL};q)- \Xi(w_{\epol};q)-2\lambda \| -q(s,a)+\gamma q(s',\epol)\|^2_{2,n}\}& \leq \{\sup_{q\in \Qbbb}  \Xi^{\lambda}_n(\hat w_{\MIL},q)\}+\sup_{q\in \Qbbb}\{     \Xi^{\lambda}_n(w_{\epol},q) \}\\
        &\leq  2 \sup_{q\in \Qbbb}\{     \Xi^{\lambda}_n(w_{\epol},q) \}\\
         &\leq  2 \sup_{q\in \Qbbb}\{     \Xi(w_{\epol},q)+(\lambda +C^2_{\Wbbb}/\lambda+C_{\Wbbb})\eta^2) \}\\
         &= (\lambda +C^2_{\Wbbb}/\lambda+C_{\Wbbb})\eta^2. 
\end{align*}

\textbf{Second Step}

Define $q_{w}=C_{\xi}\Tcal'_{\gamma}(w-w_{\epol})$. Suppose $\|q_{\hat w_{\MIL}}\|\geq \eta$, and let $\kappa=\eta/\{2\|q_{\hat w_{\MIL}} \|_2\}\in [0,0.5]$. Then, noting $\Qbbb$ is star-convex and $C_{\xi}\Tcal'_{\gamma}(\Wbbb-w_{\epol})\subset \Qbbb$, 
\begin{align*}
    &\sup_{q\in \Qbbb}\{\Xi_n(\hat w_{\MIL};q)- \Xi(w_{\epol};q)-2\lambda\|q(s,a)-\gamma q(s',\epol)\|^2_{2,n}\}\\
    &\geq \kappa\{\Xi_n(\hat w_{\MIL};q_{\hat w_{\MIL}})- \Xi(w_{\epol};q_{\hat w_{\MIL}})\} -2\kappa^2\lambda\|q_{\hat w_{\MIL}}(s,a)-\gamma q_{\hat w_{\MIL}}(s',\epol)\|^2_{2,n}. 
\end{align*}
Here, we use $\kappa q_{\hat w_{\MIL}}\in \Qbbb$. Then, from \cref{eq:q_uni}, 
\begin{align*}
  \kappa^2\|-q_{\hat w_{\MIL}}(s,a)+\gamma q_{\hat w_{\MIL}}(s',\epol)\|^2_{2,n} & \lesssim \kappa^2\{ 1.5\|-q_{\hat w_{\MIL}}(s,a)+\gamma q_{\hat w_{\MIL}}(s',\epol)\|^2_{2}+\eta^2\} \\
    &\lesssim  (1+\gamma C_m C_{\eta})\eta^2. 
\end{align*}
Here, we use 
$$\|-q_{\hat w_{\MIL}}(s,a)+\gamma q_{\hat w_{\MIL}}(s',\epol)\|^2_{2}\leq (1+C_mC_{\eta})\|q_{\hat w_{\MIL}}\|^2_2,$$
which will be explained soon. Therefore, 
{\small 
\begin{align*}
    & \sup_{q\in \Qbbb}\{\Xi_n(\hat w_{\MIL};q)- \Xi(w_{\epol};q)-2 \lambda \{\| -q(s,a)+\gamma q(s',\epol)\|^2_{2,n}\}\}\\
     &\geq \kappa\{\Xi_n(\hat w_{\MIL};q_{\hat w_{\MIL}})- \Xi(w_{\epol};q_{\hat w_{\MIL}})\} -2\lambda (1+\gamma C_m C_{\eta})\eta^2. 
\end{align*}
}
Observe that 
\begin{align*}
    \Xi_n( w;q_w)- \Xi(w_{\epol};q_w)=\E_n[(w(s,a)-w_{\epol}(s,a))(-q_{w}(s,a)+\gamma q_{w}(s',\epol))  ]. 
\end{align*}

Therefore, from Lemma \ref{lem:support2}, noting $\eta$ upper bounds the critical radius of $\Gcal_{w2}$, 
regarding $l(a_1,a_2)=a_1,a_1=(w(s,a)-w_{\epol}(s,a))(q_{w}(s,a)-\gamma q_{w}(s',\epol)) $, for any $w\in \Wbbb$, we have 
\begin{align*}
  & |   \Xi_n( w;q)- \Xi_n(w_{\epol};q)-\{\Xi( w;q)- \Xi(w_{\epol};q)\}|\\
&\leq (\eta  \E[\{(w(s,a)-w_{\epol}(s,a))(q_{w}(s,a)-\gamma q_{w}(s',\epol))\}^{2}]^{1/2}+\eta^2)\\ 
&\lesssim (\eta  \E[\{q_{w}(s,a)-\gamma q_{w}(s',\epol)\}^2]^{1/2}C_{\Wbbb} +\eta^2)=(\eta  \|\{q_{w}(s,a)-\gamma q_{w}(s',\epol)\}\|_2C_{\Wbbb} +\eta^2)\\ 
&\leq (\eta  \{\|q_w(s,a)\|_2+\gamma\|q_w(s',\epol)\|_2 \}C_{\Wbbb}  +\eta^2)\\ 
&= (\eta  \{\|q_w(s,a)\|_2+\gamma \sqrt{C_mC_{\eta}}\|q_w(s,a)\|_2 \}C_{\Wbbb}  +\eta^2). 
\end{align*}
We use
\begin{align*}
    &\E[q^2_w(s',\epol)]=\E[\E_{a'\sim \epol(a'|s')}[q_w(s',a')]^2]\leq  \E[\E_{a'\sim \epol(a'|s')}[q^2_w(s',a')]] \tag{Jensen}\\
 &=\E_{s'\sim p_{S'}(s),a'\sim \epol(a'|s')}[q^2_w(s',a')]]\\
  &=\E_{s\sim p_{S}(s),a\sim \bpol(a|s)}[q^2_w(s',a')]]C_mC_{\eta}=\|q_w(s,a)\|^2_2C_mC_{\eta}. \tag{ $P_{S'}(s)/P_{S}(s)\leq C_m,\epol(a|s)/\bpol(a|s)\leq C_{\eta}$}
\end{align*}

Thus, 
\begin{align*}
    &\kappa\{  \Xi_n(\hat w_{\MIL};q_{\hat w_{\MIL}})- \Xi_n(w_{\epol};q_{\hat w_{\MIL}})\}\\
    & \geq \kappa\{\Xi(\hat w_{\MIL};q_{\hat w_{\MIL}})- \Xi(w_{\epol};q_{\hat w_{\MIL}})\}-\kappa(\eta  \|q_{\hat w_{\MIL}}\|_2 (1+\gamma \sqrt{C_mC_{\eta}})C_{\Wbbb} +\eta^2 )  && \\
    &\geq \kappa\{\Xi( \hat w_{\MIL};q_{\hat w_{\MIL}})- \Xi(w_{\epol};q_{\hat w_{\MIL}})\}-\kappa\eta  \|q_{\hat w_{\MIL}}\|_2  (1+\gamma \sqrt{C_mC_{\eta}})C_{\Wbbb} -0.5\eta^2 && \text{Use $\kappa\leq 0.5$} \\ 
    &\overset{(a)}{=} \kappa C^{-1}_{\xi}\|q_{\hat w_{\MIL}}\|^2_2    -\kappa\eta  \|q_{\hat w_{\MIL}}\|_2 (1+\gamma \sqrt{C_mC_{\eta}})C_{\Wbbb}  )-0.5\eta^2 && \\ 
    &=\frac{\eta}{2\|q_{\hat w_{\MIL}}\|_2 }\{ C^{-1}_{\xi}\|q_{\hat w_{\MIL}}\|^2_2 -(1+\gamma \sqrt{C_mC_{\eta}})C_{\Wbbb} \eta \|q_{\hat w_{\MIL}}\|_2    \}-0.5\eta^2 && \\ 
  &=0.5 \eta  C^{-1}_{\xi}\|q_{\hat w_{\MIL}}\|_2  -\{(0.5+0.5\gamma \sqrt{C_mC_{\eta}})C_{\Wbbb}+0.5\}\eta^2 . 
\end{align*}
Regarding (a), we use 
\begin{align*}
    \Xi( w;q_{ w})- \Xi(w_{\epol};q_{w})&=\E[(w(s,a)-w_{\epol}(s,a))(-q_{w}(s,a)+\gamma q_{w}(s',\epol))]\\
    &=\E[(w(s,a)-w_{\epol}(s,a))\Tcal_{\gamma}q_{w}]=\E[\Tcal'_{\gamma}(w(s,a)-w_{\epol}(s,a))q_{w}]\\
   &=C^{-1}_{\xi}\|q_w\|^2_2.
\end{align*}
Therefore,
\begin{align*}
    &\sup_{q\in \Qbbb}\{\Xi_n(\hat w_{\MIL};q)- \Xi(w_{\epol};q)-2\lambda\|q(s,a)-\gamma q(s',\epol)\|^2_{2,n}\}\\
    &\geq \kappa\{\Xi_n(\hat w_{\MIL};q_{\hat w_{\MIL}})- \Xi(w_{\epol};q_{\hat w_{\MIL}})\} -\lambda (1+\gamma C_mC_{\eta})\eta^2 \\  
    &\geq 0.5 \eta  C^{-1}_{\xi}\|q_{\hat w_{\MIL}}\|_2  -\{(0.5+0.5\gamma \sqrt{C_mC_{\eta}})C_{\Wbbb}+0.5\}\eta^2 -\lambda (1+\gamma C_mC_{\eta})\eta^2. 
\end{align*}

\textbf{Combining all results}

Thus, $C_{\xi}\|\Tcal'_{\gamma}(\hat w_{\MIL}-w_{\epol})\|_2=\|q_{\hat w_{\MIL}}\|_2\leq \eta$ or 
\begin{align*}
    C_{\xi}\eta  \|\Tcal'_{\gamma}(\hat w_{\MIL}-w_{\epol})\|_2-\{(1+\gamma \sqrt{C_mC_{\eta}})C_{\Wbbb}+1\}\eta^2 \lesssim  (\lambda+C^2_{\Wbbb}/\lambda+C_{\Wbbb})\eta^2. 
\end{align*}
Therefore, 
\begin{align*}
     \|\Tcal'_{\gamma}(\hat w_{\MIL}-w_{\epol})\|_2\leq \braces{1+\lambda(1+\gamma  C_mC_{\eta})+C^2_{\Wbbb}/\lambda+C_{\Wbbb}+\gamma \sqrt{C_mC_{\eta}}C_{\Wbbb}+1/C_{\xi}}\eta. 
\end{align*}
\end{proof}

\begin{proof}[Proof of \cref{thm:fast_q}]
~
\\\\
\textbf{Preliminary}
Define 
\begin{align*}
    \Phi(q;w) &=\E[\{r-q(s,a)+\gamma q(s',\epol)\}w(s,a)],\\
  \Phi_n(q;w) &=\E_n[\{r-q(s,a)+\gamma q(s',\epol)\}w(s,a)],\\
     \Phi^{\lambda}_n(q;w) &= \Phi_n(q;w)-\lambda \|w\|^2_{2,n},\\
         \Phi^{\lambda}(q;w) &=\Phi(q;w)-\lambda\|w\|^2_{2}. 
\end{align*}

From Lemma \ref{lem:support1}, we have 
\begin{align}\label{eq:martin_q}
    \forall w\in \Wbbb,\,|\|w\|_{n}^2-\|w\|^2_2|\leq 0.5\|w\|^2_2+\eta^2
\end{align}
for our choice of $\eta:=\eta_n+\sqrt{c_0\log(c_1/\delta)/n}$ noting $\eta_n$ upper bounds the critical radius of $\Wbbb$.

\textbf{First Step}

By definition of $\hat q_{\MIL}$ and $q_{\epol}\in \Qbbb$, we have 
\begin{align}
    \sup_{w\in \Wbbb}   \Phi^{\lambda}_n(\hat q_{\MIL};w) \leq  \sup_{w\in \Wbbb}   \Phi^{\lambda}_n(q_{\epol};w) 
\end{align}
From Lemma \ref{lem:support2}, with probability $1-\delta$,  we have 
\begin{align}\label{eq:martin2_q}
    \forall w\in \Wbbb: |  \Phi_n(q_{\epol};w) -\Phi(q_{\epol};w)   |\lesssim C_1\{\eta \|w\|_2+\eta^2\}. 
\end{align}
Here, we use $l(a_1,a_2):=a_1a_2,a_1=w(s,a),a_2=r-q(s,a)+\gamma q(s',\epol) $ is $C_1$-Lipschitz with respect to $a_1$ by defining $C_1=R_{\max}+(1+\gamma)C_{\Qbbb}$.  More specifically, 
\begin{align*}
    |l(a_1,a_2)-l(a'_1,a_2)|\lesssim C_1|a_1-a'_1|.  
\end{align*}

Thus, 
\begin{align} 
    \sup_{w \in \Wbbb } \Phi^{\lambda}_n(q_{\epol};w) &=    \sup_{w \in \Wbbb } \braces{ \Phi_n(q_{\epol};w) -\lambda \|w\|^2_{2,n}} && \text{Definition}\nonumber \\
&\leq   \sup_{w \in \Wbbb } \braces{ \Phi(q_{\epol};w)+C_1\eta \|w\|_2+ C_1\eta^2 -\lambda\|w\|^2_{2,n}} && \text{From \cref{eq:martin2_q}}\nonumber \\
&\leq    \sup_{w \in \Wbbb } \braces{ \Phi(q_{\epol};w)+C_1\eta \|w\|_2+C_1\eta^2 -0.5\lambda\|w\|^2_2+\lambda\eta^2} && \text{From \cref{eq:martin_q}} \nonumber \\
&\lesssim   \sup_{w \in \Wbbb }  \braces{\Phi(q_{\epol};w)+(C^2_1/\lambda+C_1+\lambda) \eta^2}.  
\end{align}
In the last line, we use a general inequality, $a,b>0$:
\begin{align*}
    \sup_{f \in F}(a\|f\|-b\|f\|^2 )\leq a^2/4b. 
\end{align*}
Moreover,
{\small 
\begin{align*}
    \sup_{w\in \Wbbb}\{ \Phi^{\lambda}_n(\hat q_{\MIL};w) \}&=     \sup_{w\in \Wbbb}\{ \Phi_n(\hat q_{\MIL};w)- \Phi_n(q_{\epol};w)+ \Phi_n(q_{\epol};w) -\lambda \|w\|^2_{2,n}\}\\
    &\geq     \sup_{w\in \Wbbb}\{ \Phi_n(\hat q_{\MIL};w)- \Phi_n(q_{\epol};w) - 2\lambda \|w\|^2_{2,n}\}+  \inf_{w\in \Wbbb}\{ \Phi_n(q_{\epol};w) +\lambda  \|w\|^2_{2,n}\} \\
     &=     \sup_{w\in \Wbbb}\{ \Phi_n(\hat q_{\MIL};w)- \Phi_n(q_{\epol};w) - 2\lambda \|w\|^2_{2,n}\}+  \inf_{-w\in \Wbbb}\{ \Phi_n(q_{\epol};-w) + \lambda \|w\|^2_{2,n}\} \\
         &=     \sup_{w\in \Wbbb}\{ \Phi_n(\hat q_{\MIL};w)- \Phi_n(q_{\epol};w) - 2\lambda \|w\|^2_{2,n}\}+  \inf_{-w\in \Wbbb}\{-\Phi_n(q_{\epol};w) + \lambda \|w\|^2_{2,n}\} \\
         &=     \sup_{w\in \Wbbb}\{ \Phi_n(\hat q_{\MIL};w)- \Phi_n(q_{\epol};w) - 2\lambda \|w\|^2_{2,n}\}-  \sup_{-w\in \Wbbb}\{\Phi_n(q_{\epol};w)- \lambda \|w\|^2_{2,n}\} \\
    & =     \sup_{w\in \Wbbb}\{ \Phi_n(\hat q_{\MIL};w)- \Phi_n(q_{\epol};w)-2\lambda \|w\|^2_{2,n}\}-\sup_{w\in \Wbbb}\Phi^{\lambda}_n(q_{\epol};w).
\end{align*}
}
Here, we use $\Wbbb$ is symmetric.  Therefore, 
\begin{align*}
    \sup_{w\in \Wbbb}\{ \Phi_n(\hat q_{\MIL};w)- \Phi_n(q_{\epol};w)-2\lambda \|w\|^2_{2,n} \}&\leq \sup_{w\in \Wbbb}\braces{\Phi^{\lambda}_n(q_{\epol};w)}+    \sup_{w\in \Wbbb}\{ \Phi^{\lambda}_n(\hat q_{\MIL};w) \}\\
    &\leq 2\sup_{w\in \Wbbb}\Phi^{\lambda}_n(q_{\epol};w)\\
    & \lesssim   \sup_{w \in \Wbbb }  \braces{\Phi(q_{\epol};w)+(C^2_1/\lambda+C_1+\lambda)\eta^2}\\
    &= (C^2_1/\lambda+C_1+\lambda)\eta^2. 
\end{align*}

\textbf{Second Step }
Define 
\begin{align*}
    w_q=C_{\xi}\Tcal_{\gamma}(q-q_{\epol}). 
\end{align*}
Suppose $\|w_{\hat q_{\MIL}}\|_2\geq \eta$, and let $\kappa\coloneqq \eta/\{2\|w_{\hat q_{\MIL}}\|_2\}\in [0,0.5]$. Then, noting $\Wbbb$ is star-convex,
\begin{align*}
  \sup_{w\in \Wbbb}\{ \Phi_n(\hat q_{\MIL};w)- \Phi_n(q_{\epol};w)-2\lambda \|w\|^2_{2,n} \}\geq \kappa\{\Phi_n(\hat q_{\MIL},w_{\hat q_{\MIL}})-\Phi_n(q_{\epol},w_{\hat q_{\MIL}})\}-2\kappa^2\lambda\|w_{\hat q_{\MIL}}\|^2_{2,n}. 
\end{align*}
Here, we use $\kappa w_{\hat q_{\MIL}}\in \Wbbb$. 
Then,
\begin{align*}
    \kappa^2\|w_{\hat q_{\MIL}}\|^2_{2,n}& \lesssim \kappa^2\{1.5\|w_{\hat q_{\MIL}}\|^2_{2}+0.5\eta^2\}  \tag{ \cref{eq:martin_q}} \\
    & \lesssim \eta^2.  \tag{Definition of $\kappa$}
\end{align*}
Therefore, 
\begin{align*}
     \sup_{w\in \Wbbb}\{ \Phi_n(\hat q_{\MIL};w)- \Phi_n(q_{\epol};w)-2\|w\|^2_{2,n} \}\geq \kappa\{\Phi_n(\hat q_{\MIL},w_{\hat q_{\MIL}})-\Phi_n(q_{\epol},w_{\hat q_{\MIL}})\}-2\lambda \eta^2. 
\end{align*}
Observe that
\begin{align*}
    \Phi_n(q,w_q)-  \Phi_n(q_{\epol},w_q)&=-\E_n[\{q(s,a)-q_{\epol}(s,a)-\gamma q(s',\epol)+ \gamma q_{\epol}(s',\epol)\}w_q(s,a)].
\end{align*}
Therefore, from Lemma \ref{lem:support2} noting $\eta$ upper-bounds the critical radius of $\Gcal_q$, for any $ q\in \Qbbb$, 
\begin{align*}
    &|\Phi_n(q,w_q)-  \Phi_n(q_{\epol},w_q)-\{\Phi(q,w_q)-  \Phi(q_{\epol},w_q)\}|\\
    &=|\E_n[\{-q(s,a)+q_{\epol}(s,a)+\gamma q(s',\epol)- \gamma q_{\epol}(s',\epol)\}w_q(s,a)  ]\\ &-\E[\{-q(s,a)+q_{\epol}(s,a)+\gamma q(s',\epol)- \gamma q_{\epol}(s',\epol)\}w_q(s,a))  ]  |\\ 
    &\leq (\eta\E[\{-q(s,a)+q_{\epol}(s,a)+\gamma q(s',\epol)- \gamma q_{\epol}(s',\epol)\}^2w^2_q(s,a)]^{1/2}+\eta^2)\\
 &\lesssim (\eta C_1\|w_q\|_2+\eta^2). 
\end{align*}
Here, we invoke Lemma \ref{lem:support2} by treating $l(a_1,a_2)=a_1,\,a_1= \{q(s,a)-q_{\epol}(s,a)-\gamma q(s',\epol)+ \gamma q_{\epol}(s',\epol)\}w_q(s,a)\}$. 

Thus, 
\begin{align*}
    \kappa\{\Phi_n(\hat q_{\MIL},w_{\hat q_{\MIL}})-  \Phi_n(q_{\epol},w_{\hat q_{\MIL}})\}& \geq  \kappa\{\Phi(\hat q_{\MIL},w_{\hat q_{\MIL}})-  \Phi(q_{\epol},w_{\hat q_{\MIL}})\}-\kappa(C_1\eta\|w_{\hat q_{\MIL}}\|_2+\eta^2). \tag{$\kappa\leq 0.5$} \\
    & \geq\kappa\{\Phi(\hat q_{\MIL},w_{\hat q_{\MIL}})-  \Phi(q_{\epol},w_{\hat q_{\MIL}})\}-\kappa(C_1\eta\|w_{\hat q_{\MIL}}\|_2)-0.5\eta^2 \\
    &\overset{\text{(a)}}= \kappa\E[w_{\hat q_{\MIL}}(s,a)\Tcal_{\gamma}(\hat q_{\MIL}-q_{\epol})(s,a)]-\kappa(C_1\eta\|w_{\hat q_{\MIL}}\|_2)-0.5\eta^2 \\
  &= \frac{\eta}{2\|w_{\hat q_{\MIL}}\|}\{C^{-1}_{\xi}\|w_{\hat q_{\MIL}}\|^2_2-C_1\eta\|w_{\hat q_{\MIL}}\|_2\}-0.5\eta^2 \\
    &\geq 0.5\eta \|\Tcal_{\gamma}(\hat q_{\MIL}-q_{\epol})\|_2-(0.5+C_1)\eta^2. 
\end{align*}
For (a), we use 
\begin{align*}
    \Phi(\hat q_{\MIL},w_{\hat q_{\MIL}})-  \Phi(q_{\epol},w_{\hat q_{\MIL}})&=\E[w_{\hat q_{\MIL}}(s,a)\{-\hat q_{\MIL}(s,a)+q_{\epol}(s,a)+\gamma  \hat q_{\MIL}(s',\epol)- \gamma q_{\epol}(s',\epol)\}]\\
    &= \E[w_{\hat q_{\MIL}}(s,a)\Tcal_{\gamma}(\hat q_{\MIL}-q_{\epol})(s,a)].
\end{align*}
Finally,
\begin{align*}
     \sup_{w\in \Wbbb}\{ \Phi_n(\hat q_{\MIL};w)- \Phi_n(q_{\epol};w)-2\|w\|^2_{2,n} \}&\geq \kappa\{\Phi_n(\hat q_{\MIL},w_{\hat q_{\MIL}})-\Phi_n(q_{\epol},w_{\hat q_{\MIL}})\}-2\lambda \eta^2\\
      & \geq 0.5\eta \|\Tcal_{\gamma}(\hat q_{\MIL}-q_{\epol})\|_2-(0.5+C_1)\eta^2-2\lambda \eta^2. 
\end{align*}

\textbf{Combining all results}

Thus, $\|w_{\hat q_{\MIL}}\|\leq \eta$ or 
\begin{align*}
\eta \|\Tcal_{\gamma}(\hat q_{\MIL}- q_{\epol})\|_2-(1+C_1+\lambda )\eta^2\lesssim (C_1+C_1^2/\lambda+\lambda)\eta^2. 
\end{align*}
Therefore, we have 
\begin{align*}
        \|\Tcal_{\gamma}(\hat q_{\MIL}- q_{\epol})\|_2\leq C^{-1}_{\zeta}\eta
\end{align*}
or 
\begin{align*}
    \|\Tcal_{\gamma}(\hat q_{\MIL}- q_{\epol})\|_2\lesssim (1+C_1+\lambda+C^2_1/\lambda)  \eta. 
\end{align*}
This concludes 
\begin{align*}
    \|\Tcal_{\gamma}(\hat q_{\MIL}- q_{\epol})\|_2\lesssim \prns{1+C_1+\lambda+C^2_1/\lambda+C^{-1}_{\xi}} \eta. 
\end{align*}
\end{proof}

\begin{proof}[Proof of Corollary \ref{cor:vc_critical}]

We use Lemma \ref{lem:dudley} (Dudley integral). In order to show this, we calculate 
\begin{align*}
      \frac{1}{\sqrt{n}}\int^{\eta}_{0}\sqrt{\log  \Ncal(t;\Gcal_q),\|\cdot\|_n)}\rd t.
\end{align*}
This is upper-bounded by 
\begin{align*}
     O\prns{ \frac{\log(n)}{\sqrt{n}}\int^{\eta}_{0}\sqrt{\max(V(\Wbbb),V(\Qbbb))\log(1/t) }\rd t }. 
\end{align*}
To prove this, we calculate 
\begin{align*}
      \frac{1}{\sqrt{n}}\int^{\eta}_{0}\sqrt{\log  \Ncal(t;\Gcal(\Wbbb,\Qbbb)),\|\cdot\|_n)}\rd t.
\end{align*}

Recall the empirical $L^2$-norm as $\|f(\cdot)\|_n= \{1/n\sum_if(s^i,a^i)^2\}^{1/2}$, $L^{\infty}$-norm as $\|\cdot\|_{n,\infty}=\max_{i}|f(s^i,a^i)|$. In addition, we define $\|f(\cdot)\|'_{n,\infty}\coloneqq \max_i| f(s'^i)|$, $\|f(\cdot)\|'_{n}\coloneqq \{ 1/n\sum_i f(s'^i)^2\}^{1/2}$. Then, 
\begin{align*}
    &  \ts \int^{\eta}_0\sqrt{\frac{\log \Ncal(t,\Gcal(\Wbbb,\Qbbb),\|\cdot\|_n)}{n}}\rd\mu(t)\\
    &\ts \leq \int^{\eta}_0\sqrt{\frac{\log \Ncal(t,\Gcal(\Wbbb,\Qbbb),\|\cdot\|_{n,\infty})}{n}}\rd\mu(t)\\
     &\ts \leq\int^{\eta}_0\sqrt{\frac{\log \Ncal(t/\{(3(1+\gamma)C_{\Qbbb}\},\Wbbb,\|\cdot\|_{n,\infty})}{n}}+\sqrt{\frac{\log \Ncal(t/\{3C_{\Wbbb}\},\Qbbb,\|\cdot\|_{n,\infty})}{n}}+\sqrt{\frac{\log \Ncal(t/\{3\gamma C_{\Wbbb}\},\Vbbb,\|\cdot\|'_{n,\infty})}{n}} \rd\mu(t) \\ 
          &\ts \leq \int^{\eta}_0\sqrt{\frac{\log \Ncal(t/\{3(1+\gamma)C_{\Qbbb}\},\Wbbb,\sqrt{n}\|\cdot\|_{n})}{n}}+\sqrt{\frac{\log \Ncal(t/\{(3 C_{\Wbbb}\},\Qbbb,\sqrt{n}\|\cdot\|_{n})}{n}}+\sqrt{\frac{\log \Ncal(t/\{3\gamma C_{\Wbbb}\},\Vbbb,\sqrt{n}\|\cdot\|'_{n})}{n}} \rd\mu(t).
\end{align*}
Here, $\Vbbb=\{ q(s,\epol) ; q\in \Qbbb\}$. 
From the second line to the third line, we use the same argument as the proof of Lemma \ref{lem:covering_G}: 
\begin{align*}
    &\ts \log \Ncal(t,\Gcal(\Wbbb,\Qbbb),\|\cdot\|_{n,\infty})\\
    &\ts \leq \log \Ncal(t/\{3(1+\gamma)C_{\Qbbb}\},\Wbbb,\|\cdot\|_{n,\infty})+ \log \Ncal(t/\{3 C_{\Wbbb}\},\Qbbb,\|\cdot\|_{n,\infty})\\
    &+\log \Ncal(t/\{3\gamma C_{\Wbbb}\},\Vbbb,\|\cdot\|'_{n,\infty}), 
\end{align*}
and $\sqrt{a+b+c}\leq \sqrt{a}+\sqrt{b}+\sqrt{c}$. From the third line to the fourth line, we use $\|\cdot\|_{n,\infty}\leq \sqrt{n}\|\cdot\|_{n} $. 
Then, 
\begin{align*}
    \ts  &\frac{1}{\sqrt{n}}\int^{\eta}_{0}\sqrt{\log  \Ncal(t;\Gcal(\Wbbb,\Qbbb)),\|\cdot\|_n)}\rd t\\
 &\ts O\prns{ \frac{\log(n)}{\sqrt{n}}\int^{\eta}_{0}\sqrt{\max(V(\Wbbb),V(\Qbbb))\log(1/t) }\rd t } \\ 
     &=O\prns{ \sqrt{\max(V(\Wbbb),V(\Qbbb))}\frac{\log(n)}{\sqrt{n}}\eta \log(1/\eta)}.  
\end{align*}
Thus, 
\begin{align*}
    \ts  \frac{1}{\sqrt{n}}\int^{\eta}_{0}\sqrt{\log  \Ncal(t;\Gcal_q),\|\cdot\|_n)}\rd t.
\end{align*}
is similarly upper-bounded.

Finally, from Lemma \ref{lem:critical_basic}, the critical inequality becomes 
\begin{align*}
    \sqrt{\max(V(\Wbbb),V(\Qbbb))}\frac{\log(n)}{\sqrt{n}}\eta \log(1/\eta)\leq \eta^2. 
\end{align*}
This inequality is satisfied with $\eta=\tilde O(\sqrt{\max(V(\Wbbb),V(\Qbbb))/n})$. 
\end{proof}

\begin{proof}[Proof of Corollary \ref{cor:nonpara_critical}]
We use Lemma \ref{lem:critical_basic}. We consider the critical radius of $\Fcal$ s.t. 
\begin{align*}
    \log \Ncal(\tau;\Fcal,\|\cdot\|_{\infty})=(1/\tau)^{\alpha}. 
\end{align*}
When $\alpha\leq 2$, the critical inequality becomes
\begin{align*}
    \frac{2}{2-\alpha}n^{-1/2}\eta^{1-0.5\alpha}\leq \eta^2. 
\end{align*}
When $\eta=\Theta(n^{-1/(2+\alpha)})$, this is satisfied. 

When $\alpha=2$,  the critical inequality becomes
\begin{align*}
   \frac{1}{\sqrt{n}}\{\log (\eta)-\log(\eta^2/2\|\Fcal\|_{\infty})\}\leq \eta^2. 
\end{align*}
This is satisfied when $\eta=\Theta(n^{-1/4}\log n )$.

When $\alpha\geq 2$,  the critical inequality becomes
\begin{align*}
        \frac{2}{\alpha-2}n^{-1/2}\eta^{-\alpha+2}\leq \eta^2. 
\end{align*}
This is satisfied when $\eta=\Theta(n^{-1/(2\alpha)})$.

We apply the above analysis to our situation using the fact that the covering number of $G_q$ is upper-bounded as in the Lemma \ref{lem:covering_G}. 
Then, the proof is immediately concluded. 
\end{proof}

\begin{proof}[Proof of \cref{thm:sieve_rate}]
We consider $q_0 \in \Qbbb$ which will be specified later. From \cref{thm:agnostic_fast_q}, we obtain
\begin{align}
    \|\Tcal_\gamma ( \hat{q} - q_{\pi^e})\|_2 = O\prns{\eta + \epsilon_n + \frac{\|\Tcal_\gamma(q_{\pi^e} - q_0)\|_2^2}{\eta} + \|\Tcal_\gamma(q_{\pi^e} - q_0)\|_2}, \label{ineq:basic_sieve}
\end{align}
where $\eta$ is the maximum of critical radiuses of $\Wbbb$ and $\Gcal_q$, and $\epsilon_n = \sup_{q \in \Qbbb} \inf_{w \in \Wbbb} \|w - \Tcal_\gamma (q-q_0)\|_2$. We will bound each term in the right hand side of \eqref{ineq:basic_sieve}.

First, from Corollary \ref{cor:vc_critical}, $\eta=\tilde O(\sqrt{k_n/n})$. Next, we take $q_0$ s.t.
\begin{align*}
     \|\Tcal_\gamma(q_{\pi^e} - q_0)\|_2\leq (1+\gamma)\|q_{\pi^e} - q_0\|_{\infty})=O(k^{-p/d}_n). 
\end{align*}
from the assumption. 

Next, we show $\epsilon_n=O(k^{-p/d}_n)$. To do that, we check 
\begin{align*}
    \gamma \E_{s'\sim P_{S'|S,A}(s'|\cdot)}[-q(s',\epol)+q_0(s',\epol)]\in \Lambda^{\alpha}_{2\gamma C_{\Qbbb}K}(I^d). 
\end{align*}
Let $\|\cdot\|_b$ be a \Holder\,norm. By the assumption on the density $P_{S'|S,A}(s'|\cdot)$, we can bound the \Holder\,norm of $ \mathbb{E}_{s' \sim P_{S'|S,A}(s'|\cdot)}[q(s', \pi^e)]$
\begin{align*}
    \|\E_{s'\sim P_{S'|S,A}(s'|\cdot)}[q(s',\epol)]\|_{b}&=\left\|\int_{[0,1]^d}q(s',\epol)P_{S'|S,A}(s'|\cdot)\rd s'\right\|_b \\
    &\leq C_{\Qbbb}\left\|\int_{[0,1]^d} P_{S'|S,A}(s'|\cdot)\rd s'\right\|_b\leq C_{\Qbbb}\int_{[0,1]^d} \|P_{S'|S,A}(s'|\cdot)\|_b\rd s'\\
    &\leq C_{\Qbbb}\int_{[0,1]^d} K\rd s'=C_{\Qbbb}K. 
\end{align*}
Then, for any $q\in \Qbbb$, we can take $w'\in W'$ s.t. 
\begin{align*}
    \|w''- \gamma \E_{s'\sim P_{S'|S,A}}[-q(s',\epol)+q_0(s',\epol)]\|_{\infty}= O(k^{-p/d}_n). 
\end{align*}
Thus, by defining $w'=w''-(q-q_0)$, we can take $w'\in \Wbbb$ s.t. 
\begin{align*}
    \|w'-\Tcal_{\gamma}(q-q_0)\|_2\leq     \|w'-\Tcal'_{\gamma}(q-q_0)\|_{\infty}=O(k^{-p/d}_n). 
\end{align*}
We use  $2\Qbbb_2+\Wbbb'_2\subset \Wbbb_2$. Finally, the error is 
\begin{align*}
         \|\Tcal_\gamma ( \hat{q} - q_{\pi^e})\|_2&= O\prns{\sqrt{k_n/n}+k^{-p/d}_n+\frac{k^{-2p/d}_n}{\sqrt{k_n/n}}  }\\
         &=\tilde O(n^{-p/(2p+d)}). \tag{Taking balance of both terms}
\end{align*}

\end{proof}

\begin{proof}[Proof of Theorem \ref{cor:neural_policy_rate}]

We consider $q_0 \in \Qbbb$ which will be specified later. From \cref{thm:agnostic_fast_q}, we obtain
\begin{align}
    \|\Tcal_\gamma ( \hat{q} - q_{\pi^e})\|_2 = O\prns{\eta + \epsilon_n + \frac{\|\Tcal_\gamma(q_{\pi^e} - q_0)\|_2^2}{\eta} + \|\Tcal_\gamma(q_{\pi^e} - q_0)\|_2}, \label{ineq:basic_nn}
\end{align}
where $\eta$ is the maximum of critical radiuses of $\Wbbb$ and $\Gcal_q$, and $\epsilon_n = \sup_{q \in \Qbbb} \inf_{w \in \Wbbb} \|w - \Tcal_\gamma (q-q_0)\|_2$. We will bound each term in the right hand side of \eqref{ineq:basic_nn}.

For the third and fourth terms, we consider an approximation power by fixing $q_0$. We set $q_0 = \argmin_{q \in \Qbbb} \|q - q_{\pi^e}\|_\infty$. By Lemma \ref{lem:yaro} and boundedness of $\Tcal_\gamma$, we obtain
\begin{align*}
    \|\Tcal_\gamma(q_{\pi^e} - q_0)\|_2  \leq (1+\gamma )\|q_{\pi^e} - q_0\|_\infty = \tilde{O}(\Xi^{-\alpha/d}).
\end{align*}
By this result, we bound the third and fourth terms in \eqref{ineq:basic_nn} by
\begin{align*}
    \tilde{O}(\eta^{-1}\Xi^{-2\alpha/d} + \Xi^{-\alpha/d}).
\end{align*}

About the second term $\epsilon_n$ in \eqref{ineq:basic_nn}, we fix $q \in \Qbbb$ arbitrary and consider an explicit form of $\Tcal_\gamma (q-q_0)$ as
\begin{align*}
    \Tcal_\gamma (q-q_0) = \gamma \mathbb{E}_{s' \sim P_{S'|S,A}(s'|\cdot)}[q(s', \pi^e) - q_0(s', \pi^e)] -(q-q_0).
\end{align*}
We define $\Wbbb'$ as a set of neural network with $L$ layers and $W$ parameters.
By the assumption, $\Wbbb' \supseteq \Qbbb$ holds, hence $q,q_0 \in \Wbbb'$ is satisfied.
Let $\|\cdot\|_b$ be a Sobolev norm. By the assumption on the density $P_{S'|S,A}(s'|\cdot)$, we can bound the Sobolev norm of $ \mathbb{E}_{s' \sim P_{S'|S,A}(s'|\cdot)}[q(s', \pi^e)]$ as
\begin{align*}
    \|E_{s' \sim P_{S'|S,A}(s'|\cdot)}[q(s',\epol)]\|_{b}&=\left\|\int_{[0,1]^d}q(s',\epol)P_{S'|S,A}(s'|\cdot)\rd s'\right\|_b \\
    &\leq C_{\Qbbb}\left\|\int_{[0,1]^d} P_{S'|S,A}(s'|\cdot)\rd s'\right\|_b\leq C_{\Qbbb}\int_{[0,1]^d} \|P_{S'|S,A}(s'|\cdot)\|_b\rd s'\\
    &\leq C_{\Qbbb}\int_{[0,1]^d} K\rd s'=C_{\Qbbb}K. 
\end{align*}
Thus, 
$\gamma \mathbb{E}_{s' \sim P_{S'|S,A}(s'|\cdot)}[q(s', \pi^e) - q_0(s', \pi^e)] \in \Wcal_{2\gamma C_{\Qbbb}K}^\alpha(I_d)$
Hence, we pick $w' \in \Wbbb'$ as
\begin{align*}
    \|w' -  \gamma \mathbb{E}_{s' \sim P_{S'|S,A}(s'|\cdot)}[q(s', \pi^e) - q_0(s', \pi^e)]\|_\infty = \tilde{O}(\Xi^{-\alpha/d}),
\end{align*}
by Lemma \ref{lem:yaro}.
We set $w \in \Wbbb$ as $w=w' - q - q_0$ and obtain
\begin{align*}
    \|w - \Tcal_\gamma(q-q_0)\|_2 &= \|w' - \gamma \mathbb{E}_{s' \sim P_{S'|S,A}(s'|\cdot)}[q(s', \pi^e) - q_0(s', \pi^e)]\|_2 \\
    &\leq \|w' -  \gamma \mathbb{E}_{s' \sim P_{S'|S,A}(s'|\cdot)}[q(s', \pi^e) - q_0(s', \pi^e)]\|_\infty = \tilde{O}(\Xi^{-\alpha/d}).
\end{align*}
Hence, we have
\begin{align*}
    \epsilon_n = \tilde{O}(\Xi^{-\alpha/d}).
\end{align*}
Combining the results with \eqref{ineq:basic_nn}, we obtain
\begin{align}
    \|\Tcal_\gamma ( \hat{q} - q_{\pi^e})\|_2 \lesssim (1+C_1)^2 \eta +\tilde{O}(\eta^{-1}\Xi^{-2\alpha/d} + \Xi^{-\alpha/d}). \label{ineq:basic_nn2}
\end{align}

Next, We derive an explicit order of $\eta$ by deriving a critical radius of $\Wbbb$.
By Lemma \ref{lem:covering_neural}, we have
\begin{align*}
    \log \Ncal(\varepsilon, \Wbbb, \|\cdot\|_\infty) \leq \Xi\log \left( \frac{ 2(L+1)B^{(L+1)} (3W+1)^{(L+1)}}{\varepsilon} \right) \lesssim  \Xi L \log\left( \frac{LB\Xi}{ \varepsilon} \right).
\end{align*}
To the end, we derive the following local Rademacher complexity of the star-hull of $\Wbbb$ as Lemma \ref{lem:support2},
\begin{align*}
    &\frac{1}{\sqrt{n}}\int_{0}^\delta \sqrt{\log \Ncal( t, \mathrm{star}(\Wbbb), \|\cdot\|_n)} dt \\
    &\lesssim \frac{1}{\sqrt{n}}\int_{0}^\delta \sqrt{\log \Ncal( t, \mathrm{star}(\Wbbb), \|\cdot\|_\infty)} dt \\
    & \leq \frac{1}{\sqrt{n}}\int_{0}^\delta \sqrt{ \log (1/t) + \log \Ncal( t/2, \Wbbb, \|\cdot\|_\infty)} dt\\
    &\lesssim \frac{1}{\sqrt{n}}\int_{0}^\delta \sqrt{\log (1/t)} dt +  \sqrt{\frac{\Xi L}{n}} \int_{0}^\delta \sqrt{\log (LB\Xi/t)} dt\\
    &=\bigO\prns{ \sqrt{\frac{\Xi L}{n}}\delta\log(LB\Xi/\delta))},
\end{align*}
for any $\delta > 0$.
The second inequality follows Lemma 4.5 in \cite{mendelson2002improving} for a covering number of a star-hull of sets.
We note that the lemma in \cite{mendelson2002improving} holds regardless of the choice of norms.
By setting
\begin{align*}
    \eta = \tilde O\prns{\sqrt{\frac{\Xi L}{n}}},
\end{align*}
the  critical inequality of $\Wbbb$ is satisfied. Similarly, the  critical inequality of $\Gcal_q$ is satisfied for this $\eta$ utilizing Lemma \ref{lem:covering_G}. 

We substitute $\eta$ into \eqref{ineq:basic_nn2} and achieve
\begin{align*}
    \|\Tcal_\gamma ( \hat{q} - q_{\pi^e})\|_2 \lesssim \tilde O\prns{\sqrt{\frac{\Xi L}{n}}} +\tilde{O}(n^{1/2} L^{1/2} \Xi^{-2\alpha/d - 1/2} + \Xi^{-\alpha/d})
\end{align*}
Then, by setting $L = \Theta(\log n)$ and $W=\Theta(n^{d/(2\alpha + d)})$, the statement is concluded. 
\end{proof}

\begin{proof}[Proof of Corollary \ref{cor:implication2}]
The proof is done as in the proof of Corollary \ref{cor:implication}. 
\end{proof}

\subsection{Proof of \cref{sec:efficiency}}


\begin{proof}[Proof of \cref{thm:efficiency}]

Recall we use sample splitting. The estimator is defined as
\begin{align*}
     \E_{n_1}[\phi(\hat w^{\MIL}_0,\hat q^{\MIL}_0)]+ \E_{n_0}[\phi(\hat w^{\MIL}_1,\hat q^{\MIL}_1)], 
\end{align*}
where $\E_{n_1}$ is the empirical average over $\Dcal_1$ and $\E_{n_0}$ is the empirical average over $\Dcal_0$. We define $\phi(s,a,r,s';w,q)=w(s,a)(r-q(s,a)+\gamma q(s',\epol))+\E_{d_0}[q(s_0,\epol)]$. Then, we have 
\begin{align}
    |\E_{n_1}[\phi(\hat w^{\MIL}_0,\hat q^{\MIL}_0)]-J|&\leq |(\E_{n_1}-\E)[\phi(\hat w^{\MIL}_0,\hat q^{\MIL}_0)|\Dcal_0]-(\E_{n_1}-\E)[\phi(w_{\epol},q_{\epol})|\Dcal_0]| \label{eq:first}\\
    &+|\E[\phi(\hat w^{\MIL}_0,\hat q^{\MIL}_0)-\phi(w_{\epol},q_{\epol})|\Dcal_0]| \label{eq:second}\\
    &+|\E_{n_1}[\phi(w_{\epol},q_{\epol})]-J|. \label{eq:third}
\end{align}
Recall with probability $1-2\delta'$, we have 
\begin{align*}
    &\|\hat w^{\MIL}_0-w_{\epol}\|_2\lesssim \eta_{w,n}\times C_{\iota,n}C^{*}_{w,n},\,\|\Tcal'_{\gamma}(\hat w^{\MIL}_0-w_{\epol})\|_2\lesssim \eta_{w,n}\times C^{*}_{w,n},\\
    &\|\hat q^{\MIL}_0-q_{\epol}\|_2\lesssim \eta_{q,n}\times C_{\iota,n}C^{*}_{q,n},\,\|\Tcal_{\gamma}(\hat q^{\MIL}_0-q_{\epol})\|_2\lesssim \eta_{q,n}\times C^{*}_{q,n}. 
\end{align*}

\textbf{First Term}

We analyze the first term \eqref{eq:first}. For simplicity, we write $\hat w^{\MIL}_0,\hat q^{\MIL}_0$ as  $\hat w,\hat q$.
\begin{align}
    &(\E_{n_1}-\E)[\phi(\hat w,\hat q)-\phi(w_{\epol},q_{\epol})|\Dcal_0] \nonumber\\
    &=(\E_{n_1}-\E)[(\hat w(s,a)-w_{\epol}(s,a))(-\hat q(s,a)+\gamma \hat v(s')+q_{\epol}(s,a)-\gamma v_{\epol}(s') )|\Dcal_0] \label{eq:first_1}\\
    &+(\E_{n_1}-\E)[(\hat w(s,a)-w_{\epol}(s,a))(r-q_{\epol}(s,a)+\gamma v_{\epol}(s')) |\Dcal_0] \label{eq:first_2} \\
    &+(\E_{n_1}-\E)[w_{\epol}(s,a)\{-\hat q(s,a)+q_{\epol}(s,a)+\gamma \hat v(s')-\gamma v_{\epol}(s')\}|\Dcal_0]\label{eq:first_3}. 
\end{align}
From Berstein's inequality, with probability $1-\delta'$, regarding the term \cref{eq:first_1} (we can do since $\hat w$ only depend on $\Dcal_0$), we have
\begin{align*}
    &(\E_{n_1}-\E)[(\hat w(s,a)-w_{\epol}(s,a))(-\hat q(s,a)+\gamma \hat v(s')+q_{\epol}(s,a)-\gamma v_{\epol}(s') )|\Dcal_0]\\
    &\leq \sqrt{2\frac{\mathrm{var}[(\hat w(s,a)-w_{\epol}(s,a))(-\hat q(s,a)+\gamma \hat v(s')+q_{\epol}(s,a)-\gamma v_{\epol}(s') )]\log(1/\delta')}{n_1}}+\frac{2C_{\Wbbb_1}C_{\Qbbb_2}\log(1/\delta')}{3n_1}\\
        &\lesssim \sqrt{2\frac{\|\hat w-w_{\epol}\|^2_2\{C_{\Qbbb_2}\}^2 \log(1/\delta')}{n_1}}+\frac{2C_{\Wbbb_1}C_{\Qbbb_2}\log(1/\delta')}{3n_1}. 
\end{align*}
The second term \eqref{eq:first_2} is similarly upper-bounded. With probability $1-\delta'$, The third term \eqref{eq:first_3} is upper-bounded as follows:
\begin{align*}
    &(\E_{n_1}-\E)[w_{\epol}(s,a)\{-\hat q(s,a)+q_{\epol}(s,a)+\gamma \hat v(s')-\gamma v_{\epol}(s')\}|\Dcal_0]\\
    &\leq \sqrt{\frac{2\var[w_{\epol}(s,a)\{-\hat q(s,a)+q_{\epol}(s,a)+\gamma \hat v(s')-\gamma v_{\epol}(s')\}]\log(1/\delta')}{n_1}}+\frac{2C_{\Wbbb_1}C_{\Qbbb_2}\log(1/\delta')}{3n_1}\\
        &\leq \sqrt{\frac{2\|\hat q-q_{\epol}\|^2_2C^2_{\Wbbb_1}(1+C_{\eta}C_m)\log(1/\delta')}{n_1}}+\frac{2C_{\Wbbb_1}C_{\Qbbb_2}\log(1/\delta')}{3n_1}. 
\end{align*}
We use 
\begin{align*}
    &\var[w_{\epol}(s,a)\{-\hat q(s,a)+q_{\epol}(s,a)+\gamma \hat v(s')-\gamma v_{\epol}(s')\}|\Dcal_0]\\
    &\lesssim C^2_{\Wbbb_1}\{\var[ -\hat q(s,a)+q_{\epol}(s,a) ]+\var[\gamma \hat v(s')-\gamma v_{\epol}(s')  ]\}\\
      &\lesssim C^2_{\Wbbb_1}\{\var[ -\hat q(s,a)+q_{\epol}(s,a) ]+C_{\eta}C_m\var[ -\hat q(s,a)+q_{\epol}(s,a) ]\}.   
\end{align*}
Finally, by combining the all arguments, with probability $1-3\delta'$, 
\begin{align*}
      &|(\E_{n_1}-\E)[\phi(\hat w,\hat q)-\phi(w_{\epol},q_{\epol})|\Dcal_0]| \nonumber\\
 &\lesssim \sqrt{\frac{\{\|\hat w-w_{\epol}\|^2_2\{C_{\Qbbb_2}\}^2 +\|\hat q-q_{\epol}\|^2_2C^2_{\Wbbb_1}(1+C_{\eta}C_m)  \} \log(1/\delta')}{n_1}}+\frac{2C_{\Wbbb_1}C_{\Qbbb_2}\log(1/\delta')}{3n_1}\\
 &\lesssim  (C^{*}_{w,n}C_{\Qbbb_2}C_{\iota,n}\eta_w+C^{*}_{q,n}C_{\Wbbb_1}C_{\iota,n}\eta_{q,n}\sqrt{1+C_{\eta}C_m})\sqrt{\frac{\log(1/\delta')}{n_1}}+\frac{C_{\Wbbb_1}C_{\Qbbb_2}\log(1/\delta')}{n_1}
\end{align*}

\textbf{Second Term}

We analyze the second term \eqref{eq:second}. For simplicity, we write $\hat w^{\MIL}_0,\hat q^{\MIL}_0$ as  $\hat w,\hat q$. Then, 
\begin{align*}
   & |\E[\phi(\hat w,\hat q)-\phi(w_{\epol},q_{\epol})]|=|\E[(\hat w-w_{\epol})\Tcal_{\gamma}(\hat q-q_{\epol})]| \\
    &\leq \|\hat w-w_{\epol}\|_2\|\Tcal_{\gamma}(\hat q-q_{\epol})\|_2 \leq \eta_{w,n}\eta_{q,n}C^{*}_{w,n}C^{*}_{q,n}C_{\iota,n}. 
\end{align*}
Here, we use 
\begin{align*}
&    \E[\phi(\hat w,\hat q)-\phi(w_{\epol},q_{\epol})]\\
&=\E[\hat w(s,a)\{r-\hat q(s,a)+\gamma \hat q(s',\epol)\}+(1-\gamma) \E_{d_0}[\hat q(s_0,\epol)  ]]\\
&-\E[w_{\epol}(s,a)\{r-q_{\epol}(s,a)+\gamma \hat q(s',\epol)\}+(1-\gamma) \E_{d_0}[q_{\epol}(s_0,\epol)  ] ]\\
&=\E[(\hat w-w_{\epol})\Tcal_{\gamma}(\hat q-q_{\epol})]. 
\end{align*}
 

\textbf{Combining All terms}
With probability $1-(\delta+10\delta')$ (From $2\times 2\delta'+2\times 3\delta'+\delta$), 
\begin{align*}
     & \E_{n_1}[\phi(\hat w^{\MIL}_0,\hat q^{\MIL}_0)]+ \E_{n_0}[\phi(\hat w^{\MIL}_1,\hat q^{\MIL}_1)]\\
     &\leq  \E_{n}[\phi(w_{\epol},q_{\epol})]-\E[\phi(w_{\epol},q_{\epol})]\\
     &+ c_0\braces{\eta_{w,n}\eta_{q,n}C^{*}_{w,n}C^{*}_{q,n}C_{\iota,n}+(C^{*}_{w,n}\eta_w+C^{*}_{q,n}\eta_{q,n}\sqrt{1+C_{\eta}C_m})C_{\iota,n}\sqrt{\frac{\log(1/\delta')}{n}}+\frac{C_{\Wbbb_1}C_{\Qbbb_2}\log(1/\delta')}{n}}\\
     &\leq \sqrt{2\frac{\mathrm{var}[\phi(w_{\epol},q_{\epol})]\log(1/\delta)}{n}}\\
       &+ c_0\{\eta_{w,n}\eta_{q,n}C^{*}_{w,n}C^{*}_{q,n}C_{\iota,n}+(C^{*}_{w,n}\eta_w+C^{*}_{q,n}\eta_{q,n}\sqrt{1+C_{\eta}C_m})C_{\iota,n}\sqrt{\frac{\log(1/\delta')}{n}}\\
     &+\frac{C_{\Wbbb_1}C_{\Qbbb_2}(\log(1/\delta')+\log(1/\delta))}{n}\}. 
\end{align*}
for  some universal constant $c_0$. 

\end{proof}

\begin{proof}[Proof of \cref{thm:operator_calculation}]
~ \\ 
\textbf{Q-functions}
Since $q_{\epol}\in \Qbbb_2$, for any $q\in \Qbbb_2$, there exists $\theta$ s.t. $q-q_{\epol}=\theta^{\top}\phi$. Then, $$\ts
\Bcal q-q=\Tcal_{\gamma}(q-q_{\epol})=\theta^{\top}(\gamma \Tcal-I)\phi=\theta^{\top}(\gamma M^{\top}_{\epol}-I)\phi.$$ 
Thus, 
\begin{align*}
   \frac{\|q-q_{\epol}\|^2_2}{\|\Bcal q-q\|^2_2}&\leq \sup_{\theta \in \mathbb{R}^d}\frac{\theta^{\top}\E[\phi(s,a)\phi^{\top}(s,a)]\theta}{\theta^{\top}(\gamma M^{\top}_{\epol}-I)\E[\phi(s,a)\phi^{\top}(s,a)](\gamma M^{\top}_{\epol}-I)^{\top}\theta}\\
   &=\sup_{\theta\in \mathbb{R}^d}\frac{\theta^{\top}X\theta}{\theta^{\top}(\gamma M_{\epol}-I)X(\gamma M_{\epol}-I)^{\top}\theta}\\
     &=\sup_{\theta\in \mathbb{R}^d}\frac{\theta^{\top}X\theta}{\theta^{\top}\bar M_{\epol}X (\bar M_{\epol})^{\top}\theta}\\
     &\leq \sigma_{\max}(\{\bar M_{\epol}X (\bar M_{\epol})^{\top}\}^{-1/2}X\{M_{\epol}X (\bar M_{\epol})^{\top}\}^{-1/2}). 
\end{align*}

\textbf{W-functions}

Since $w_{\epol}\in \Wbbb_1$, for any $w\in \Wbbb_1$, there exists $\theta$ s.t. $w-w_{\epol}=\theta^{\top}\phi$. Then, $$\ts
\Bcal'w-w=\Tcal'_{\gamma}(w-w_{\epol})=\theta^{\top}(\gamma \Tcal'-I)\phi=\theta^{\top}(\gamma M'^{\top}_{\epol}-I)\phi.$$ 
Thus, 
\begin{align*}
   \frac{\|w-w_{\epol}\|^2_2}{\|\Bcal w-w\|^2_2}&\leq \sup_{\theta\in \mathbb{R}^d}\frac{\theta^{\top}X\theta}{\theta^{\top}(\gamma M'^{\top}_{\epol}-I)X(\gamma M'^{\top}_{\epol}-I)^{\top}\theta}\\
     &=\sup_{\theta\in \mathbb{R}^d}\frac{\theta^{\top}X\theta}{\theta^{\top}\bar M'_{\epol} X (\bar M'_{\epol})^{\top}\theta}\\
     &\leq \sigma_{\max}(\{\bar M'_{\epol} X (\bar M'_{\epol})^{\top}\}^{-1/2}X\{\bar M'_{\epol} X (\bar M'_{\epol})^{\top}\}^{-1/2}). 
\end{align*}

\end{proof}

\begin{proof}[Proof of Corollary \ref{cor:tabular_recovery}] 
Recall \cref{exa:tabular}.  The matrix $\E[\phi\phi^{\top}]$, i.e., $X$, is defined as a matrix, whose $(i,i)$-th entry corresponding $x_i\in \Xcal$, is $\E[I((s,a)=x_i)]$, i.e., $d_{\bpol}(x_i)$. Besides, 
\begin{align*}
     \bar M_{\epol}&=(\gamma P^{\epol}-I)^{\top},\\ 
     \bar M'_{\epol}   &=\E[\phi\phi^{\top}](\gamma P^{\epol}-I) \E[\phi\phi^{\top}]^{-1}.
\end{align*}
All of matrices are non-singular. Thus, the assumptions in \cref{thm:operator_calculation} are satisfied. 

\end{proof}

\begin{proof}[Proof of Corollary \ref{cor:final_efficiency}]
Recall we use sample splitting. The estimator is defined as
\begin{align*}
     \E_{n_1}[\phi(\hat w,\hat q)]+ \E_{n_0}[\phi(\hat w,\hat q)]. 
\end{align*}
We define $\phi(s,a,r,s';w,q)=w(r-q(s,a)+\gamma q(s',\epol))+\E_{d_0}[q(s_0,\epol)]$. Then, we have 
\begin{align}
    |\E_{n_1}[\phi(\hat w,\hat q)]-J|&\leq |(\E_{n_1}-\E)[\phi(\hat w,\hat q)|\Dcal_0]-(\E_{n_1}-\E)[\phi(w_{\epol},q_{\epol})|\Dcal_0]| \label{eq:first_final}\\
    &+|\E[\phi(\hat w,\hat q)-\phi(w_{\epol},q_{\epol})|\Dcal_0]| \label{eq:second_final}\\
    &+|\E_{n_1}[\phi(w_{\epol},q_{\epol})]-J|. \label{eq:third_final}
\end{align}
As in the proof of \cref{thm:efficiency}, it is proved that \cref{eq:first_final,eq:second_final} are $\op(1)$. From CLT, the third term weakly converges to a normal distribution 
\begin{align*}\ts 
\Ncal(0,\E[w^2_{\epol}(s,a)\{r-q(s,a)+q(s',\epol)\}^2]). 
\end{align*}
From Slutsky's theorem, the final statement is concluded. 
\end{proof}

\subsection{Proof of \cref{sec:extension}}


\begin{proof}[Proof of \cref{thm:generalization_s}]
The proof is similarly done as the proof of \cref{thm:generalization}.Thus, we omit the proof.  
\end{proof}

\begin{proof}[Proof of \cref{thm:recovery1_s}]
The proof is similarly done as the proof of \cref{thm:recovery1}. Thus, we omit the proof.  
\end{proof}

\begin{proof}[Proof of \cref{thm:mabo}]

The main part is already proved in \citep{XieTengyang2020QASf}. For \comp, we prove the main part again.

\textbf{Preliminary}\\

Note that estimators for q-functions in MIL are defined as 
\begin{align*}
   \hat q=\argmin_{q'\in\Qbbb}\max_{w'\in \Wbbb}|\E_n[f^{*}_q(s,a,r,s;w',q')]|,\,f^{*}_q(s,a,r,s;w',q')=w'(s,a)\{r-q'(s,a)+\argmax_a q'(s',a)\}
\end{align*}
We write this solution as $(\hat q,\hat w)$, i.e., 
\begin{align*}
  \hat w=\argmax_{w'\in \Wbbb}|\E_n[f^{*}_q(s,a,r,s;w',\hat q)]|. 
\end{align*}
In addition,  we have 
\begin{align*}
    \E[ f^{*}_q(s,a,r,s;w',q')]=\E[w'\Tcal^{*}_{\gamma}(\hat q-q^{*})], 
\end{align*}
where $\Tcal^{*}_{\gamma}:q\to \gamma\E[\argmax_{a}q(s',a)|a]-q$.

Let $\pi^{*}=\argmax_{\pi \in \Pi_q}J(\pi)$. Then, from \citep[Theorem 8]{XieTengyang2020QASf}, 
\begin{align*}
    \max_{\pi \in \Pi_q}J(\pi)-J(\pi_{\hat q})\leq 2(1-\gamma)^{-1}\max_{\pi \in \Pi_Q}|\E[w_{\pi}\{r+\gamma \max_a \hat q(s',a)-\hat q(s,a)\}] |.
\end{align*}
Thus, 
\begin{align*}
&|\E[w_{\pi}\{r+\gamma \max_a \hat q(s',a)-\hat q(s,a)\}] |\\
&\leq (\max_{w'\in \Wbbb}|\E[w\{r+\gamma \max_a \hat q(s',a)-\hat q(s,a)\}] |)+ \min_{w'\in \Wbbb}|\E[(w_{\pi}-w')\{r+\gamma \max_a \hat q(s',a)-\hat q(s,a)\}] |\\
&\leq (\max_{w'\in \Wbbb}|\E[w\{r+\gamma \max_a \hat q(s',a)-\hat q(s,a)\}] |)+ \max_{q'\in \Qbbb}\min_{w'\in \Wbbb}|\E[(w_{\pi}-w')\{r+\gamma \max_a q'(s',a)-q'(s,a)\}] |\\
\end{align*}
Here, the second term is $0$ as follows since $w_{\pi}\in \Wbbb$. Thus,
\begin{align*}
    |\E[w_{\pi}\{r+\gamma \max_a \hat q(s',a)-\hat q(s,a)\}] | \leq \max_{w' \in \Wbbb}|\E[w'\Tcal^{*}_{\gamma}(\hat q-q^{*})]|= \E[\tilde w\Tcal^{*}_{\gamma}(\hat q-q^{*})]|,
\end{align*}
where 
\begin{align*}
    \tilde w= \argmax_{w'\in \Wbbb}|\E[w\Tcal^{*}_{\gamma}(\hat q-q^{*})]|. 
\end{align*}

As a first step, we calculate $|\E[\hat w\Tcal^{*}_{\gamma}(\hat q-q^{*})]|$: 
\begin{align*}
   & |\E[\hat w\Tcal^{*}_{\gamma}(\hat q-q^{*})]|=|\E[f^{*}_q(\hat w,\hat q)]|=|\E[f^{*}_q(\hat w,\hat q)]|-|\E[f^{*}_q(\hat w,q^{*})]|\\
    &=|\E[f^{*}_q(\hat w,\hat q)]|-|\E_n[f^{*}_q(\hat w,\hat q)]|+|\E_n[f^{*}_q(\hat w,\hat q)]|-|\E_n[f^{*}_q(\hat w,q^{*})]|+|\E_n[f^{*}_q(\hat w,q^{*})]|-|\E[f^{*}_q(\hat w,q^{*})]|\\
    &=|\E[f^{*}_q(\hat w,\hat q)]|-|\E_n[f^{*}_q(\hat w,\hat q)]|+|\E_n[f^{*}_q(\hat w,q^{*})]|-|\E[f^{*}_q(\hat w,q^{*})]|\\
    &\leq 2\sup_{w'\in\Wbbb,q'\in \Qbbb}(\E-\E_n)[|f^{*}_q(q',w')| ] 
\end{align*}
Here, we use $|\E_n[f^{*}_q(\hat w,\hat q)]|\leq |\E_n[f^{*}_q(\hat w,q^{*})]|$ and $|\E[f^{*}_q(\hat w,q^{*})]|$=0

Second, we bound the difference of $ |\E[\hat w\Tcal^{*}_{\gamma}(\hat q-q^{*})]|$ and $ |\E[\tilde w\Tcal^{*}_{\gamma}(\hat q-q^{*})]|$
\begin{align*}
    &|\E[\tilde w\Tcal_{\gamma}(\hat q-q^{*})]|-|\E[\hat w\Tcal_{\gamma}(\hat q-q^{*})]|\\
    & =|\E[f^{*}_q(\tilde w,\hat q)]|-|\E[f^{*}_q(\hat w,\hat q)]|\\
    & =|\E[f^{*}_q(\tilde w,\hat q)]|-|\E_n[f^{*}_q(\tilde w,\hat q)]|+|\E_n[f^{*}_q(\tilde w,\hat q)]|-|\E_n[f^{*}_q(\hat w,\hat q)]|+|\E_n[f^{*}_q(\hat w,\hat q)]| -|\E[f^{*}_q(\hat w,\hat q)]|\\
  & \leq |\E[f^{*}_q(\tilde w,\hat q)]|-|\E_n[f^{*}_q(\tilde w,\hat q)]|+|\E_n[f^{*}_q(\hat w,\hat q)]| -|\E[f^{*}_q(\hat w,\hat q)]|\\
    &  \leq 2\sup_{w'\in\Wbbb,q'\in \Qbbb}(\E-\E_n)[|f^{*}_q(q',w')| ] .
\end{align*}

Here, we use $|\E_n[f^{*}_q(\tilde w,\hat q)]|\leq |\E_n[f^{*}_q(\hat w,\hat q)]|$. In the end, 
\begin{align*}
 |\E[\tilde w\Tcal^{*}_{\gamma}(\hat q-q^{*})]||\leq 4\sup_{w'\in\Wbbb,q'\in \Qbbb}(\E-\E_n)[|f^{*}_q(q',w')| ]. 
\end{align*}

\begin{proof}[Proof of \cref{thm:fast_q_opti}]

The proof is similarly conducted as \cref{thm:fast_q}. Thus, we omit the proof. 
\end{proof}

\begin{proof}[Proof of Corollary \ref{cor:policy_optimize}]
From \citep[Theorem 8]{XieTengyang2020QASf}, 
\begin{align*}
    \max_{\pi \in \Pi_q}J(\pi)-J(\pi_{\hat q})&\leq 2(1-\gamma)^{-1}\max_{\pi \in \Pi_Q}|\E[w_{\pi}\{r+\gamma \max_a \hat q(s',a)-\hat q(s,a)\}] |. 
 \end{align*}
Then, this implies 
 \begin{align*}
    \max_{\pi \in \Pi_q}J(\pi)-J(\pi_{\hat q})    &\leq 2(1-\gamma)^{-1}\max_{\pi \in \Pi_Q}|\E[w_{\pi}(\Bcal^{*}\hat q-\hat q ) ] |\\
       &\leq 2(1-\gamma)^{-1}\max_{\pi \in \Pi_Q}\|w_{\pi}\|_2 \|\Bcal^{*}\hat q-\hat q \|_2. \tag{CS inequality}
\end{align*}
By the combining the statement in \cref{thm:mabo}, the statement is proved. 

\end{proof}

\end{proof}


%



\subsection{Proof of \cref{ape:misspecification} }


\begin{proof}[Proof of \cref{thm:allow_mis1}]
We prove the case of $\hat J^{\MIL}_q$. The other case is similarly proved. We use the notation in \cref{thm:fast_q}. Recall the proof of \cref{thm:fast_q}. We have 
\begin{align*}
|\hat J^{\MIL}_q-J|&\leq  \max_{w\in \Wbbb_2}|\E[w(s,a)\Tcal_{\gamma}(\hat q-q_{\epol})]|+ \min_{w\in \Wbbb_2}|\E[(w-w_{\epol})\Tcal_{\gamma}(\hat q-q_{\epol})]|\\ 
&\leq \max_{w\in \Wbbb_2}|\E[w(s,a)\Tcal_{\gamma}(\hat q-q_{\epol})]|+\min_{w\in \Wbbb_2}\max_{q\in \Qbbb_2}|\E[(w-w_{\epol})\Tcal_{\gamma}(q-q_{\epol})]|.
\end{align*}
and 
\begin{align*}\ts 
    \E[ f_q(s,a,r,s';w,q)]=\E[w\Tcal_{\gamma}( q-q_{\epol})]. 
\end{align*}
Thus, 
\begin{align*}
      |\hat J^{\MIL}_q-J|\leq \E[\tilde w\Tcal_{\gamma}(\hat q-q_{\epol})]|+\min_{w\in \Wbbb_2}\max_{q\in \Qbbb_2}|\E[(w-w_{\epol})\Tcal_{\gamma}(q-q_{\epol})]|,
\end{align*}
where 
\begin{align*}\ts 
    \tilde w= \argmax_{w\in \Wbbb_2}|\E[w\Tcal_{\gamma}(\hat q-q_{\epol})]|.
\end{align*}
Next, we define 
\begin{align*}
   q^{\dagger}=\argmin_{q\in \Qbbb_2}\max_{w\in \Wbbb_2}|\E[w\Tcal_{\gamma}(q-q_{\epol})]|. 
\end{align*}

As a first step, we bound $|\E[\hat w\Tcal_{\gamma}(\hat q-q^{\dagger})]|-|\E[f_q(\hat w,q^{\dagger})]|$: 
\begin{align*}
   & |\E[\hat w\Tcal_{\gamma}(\hat q-q_{\epol})]|-|\E[f_q(\hat w,q^{\dagger})]|=|\E[f_q(\hat w,\hat q)]|-|\E[f_q(\hat w,q^{\dagger})]|=|\E[f_q(\hat w,\hat q)]|-|\E[f_q(\hat w,q^{\dagger})]|\\
    &=|\E[f_q(\hat w,\hat q)]|-|\E_n[f_q(\hat w,\hat q)]|+|\E_n[f_q(\hat w,\hat q)]|-|\E_n[f_q(\hat w,q^{\dagger})]|+|\E_n[f_q(\hat w,q^{\dagger})]|-|\E[f_q(\hat w,q^{\dagger})]|\\
    &=|\E[f_q(\hat w,\hat q)]|-|\E_n[f_q(\hat w,\hat q)]|+|\E_n[f_q(\hat w,q^{\dagger})]|-|\E[f_q(\hat w,q^{\dagger})]|\\
    &\leq 2|\sup_{w\in\Wbbb_2,q\in \Qbbb_2}(\E-\E_n)[f_q(q,w)]| 
\end{align*}
Here, we use $|\E_n[f_q(\hat w,\hat q)]|\leq |\E_n[f_q(\hat w,q^{\dagger})]|$. 

Second, we bound the difference of $ |\E[\hat w\Tcal_{\gamma}(\hat q-q_{\epol})]|$ and $ |\E[\tilde w\Tcal_{\gamma}(\hat q-q_{\epol})]|$: 
\begin{align*}
    &|\E[\tilde w\Tcal_{\gamma}(\hat q-q_{\epol})]|-|\E[\hat w\Tcal_{\gamma}(\hat q-q_{\epol})]|\\
    & =|\E[f_q(\tilde w,\hat q)]|-|\E[f_q(\hat w,\hat q)]|\\
    & =|\E[f_q(\tilde w,\hat q)]|-|\E_n[f_q(\tilde w,\hat q)]|+|\E_n[f_q(\tilde w,\hat q)]|-|\E_n[f_q(\hat w,\hat q)]|+|\E_n[f_q(\hat w,\hat q)]| -|\E[f_q(\hat w,\hat q)]|\\
  & \leq |\E[f_q(\tilde w,\hat q)]|-|\E_n[f_q(\tilde w,\hat q)]|+|\E_n[f_q(\hat w,\hat q)]| -|\E[f_q(\hat w,\hat q)]|\\
    &  \leq \sup_{w\in\Wbbb_2,q\in \Qbbb_2}2|(\E-\E_n)[f_q(q,w) ]| .
\end{align*}
Here, we use $|\E_n[f_q(\tilde w,\hat q)]|\leq |\E_n[f_q(\hat w,\hat q)]|$. In the end, 
\begin{align*}
        &\E[\tilde w\Tcal_{\gamma}(\hat q-q_{\epol})]|\\
        &\leq |\E[\tilde w\Tcal_{\gamma}(\hat q-q_{\epol})]|- |\E[\hat w\Tcal_{\gamma}(\hat q-q_{\epol})]|+|\E[\hat w\Tcal_{\gamma}(\hat q-q_{\epol})]|-|\E[f_q(\hat w,q^{\dagger})]|+|\E[f_q(\hat w,q^{\dagger})]|\\ 
        &\leq  \sup_{w\in\Wbbb_2,q\in \Qbbb_2}4|(\E-\E_n)[f_q(q,w)]|\\
        &+\min_{q\in \Qbbb_2}\max_{w\in \Wbbb_2}|\E[(w-w_{\epol})\Tcal_{\gamma}(q-q_{\epol})]|.
\end{align*}
Noting the first term is upper-bounded by $
\Rcal_n(\Wbbb_2)R_{\max}+\Rcal_n(\Gcal(\Wbbb_2,\Qbbb_2))+C_{\Wbbb_2}\{R_{\max}+C_{\Qbbb_2}\}\sqrt{\log(2/\delta)/n}$, the error is upper-bounded by  
\begin{align*}
|\hat J^{\MIL}_q-J|&\leq 4\sup_{w\in\Wbbb_2,q\in \Qbbb_2}|(\E-\E_n)[f_q(q,w)]|+|\E[f_q(\hat w,q^{\dagger})]|+\min_{w\in \Wbbb_2}\max_{q\in \Qbbb_2}|\E[(w-w_{\epol})\Tcal_{\gamma}(q-q_{\epol})]| \\ 
    &\lesssim \Rcal_n(\Wbbb_2)R_{\max}+\Rcal_n(\Gcal(\Wbbb_2,\Qbbb_2))+C_{\Wbbb_2}\{R_{\max}+C_{\Qbbb_2}\}\sqrt{\log(2/\delta)/n}+\\
    &+\min_{q\in \Qbbb_2}\max_{w\in \Wbbb_2}|\E[w\Tcal_{\gamma}(q-q_{\epol})]|+\min_{w\in \Wbbb_2}\max_{q\in \Qbbb_2}|\E[(w-w_{\epol})\Tcal_{\gamma}(q-q_{\epol})]|. 
\end{align*}

\end{proof}

\begin{proof}[Proof of \cref{thm:agnostic_fast_q}]
~\\
\textbf{Preliminary}\\

Define 
\begin{align*}
    \Phi(q;w) &=\E[\{r-q(s,a)+\gamma q(s',\epol)\}w(s,a)],\\
  \Phi_n(q;w) &=\E_n[\{r-q(s,a)+\gamma q(s',\epol)\}w(s,a)],\\
     \Phi^{\lambda}_n(q;w) &= \Phi_n(q;w)-\lambda \|w\|^2_{n},\\
         \Phi^{\lambda}(q;w) &=\Phi(q;w)-\lambda\|w\|^2_{2}. 
\end{align*}

From Lemma \ref{lem:support1}, we have 
\begin{align}\label{eq:martin}
    \forall w\in \Wbbb,\,|\|w\|_{n}^2-\|w\|^2_2|\leq 0.5\|w\|^2_2+\eta^2
\end{align}
for our choice of $\eta:=\eta_n+\sqrt{c_0\log(c_1/\delta)/n}$ noting $\eta_n$ upper bounds the critical radius of $\Wcal$.

\textbf{First Step}

By definition of $\hat q$ and $q_0\in \Qbbb$, we have 
\begin{align}\label{eq:obvious_def}
    \sup_{w\in \Wbbb}   \Phi^{\lambda}_n(\hat q;w) \leq  \sup_{w\in \Wbbb}   \Phi^{\lambda}_n(q_0;w) 
\end{align}
From Lemma \ref{lem:support2}, with probability $1-\delta$,  we have 
\begin{align}\label{eq:martin22}
    \forall w\in \Wbbb: |  \Phi_n(q_0;w) -\Phi(q_0;w)   |\lesssim (C_1+1)\{\eta \|w\|_2+\eta^2\}. 
\end{align}
Here, we use $l(a_1,a_2):=a_1a_2,a_1=w(s,a),a_2=r-q(s,a)+\gamma q(s',\epol) $ is $C_1$-Lipschitz with respect to $a_1$ by defining $C_1=R_{\max}+(1+\gamma)C_{\Qbbb}$.  More specifically, 
\begin{align*}
    |l(a_1,a_2)-l(a'_1,a_2)|\lesssim C_1|a_1-a'_1|.  
\end{align*}

Thus, 
\begin{align} 
    \sup_{w \in \Wbbb } \Phi^{\lambda}_n(q_0;w) &=    \sup_{w \in \Wbbb } \braces{ \Phi_n(q_0;w) -\lambda \|w\|^2_{n}} && \text{definition}\nonumber \\
&\leq   \sup_{w \in \Wbbb } \braces{ \Phi(q_0;w)+C_1\eta \|w\|_2+ C_1\eta^2 -\lambda\|w\|^2_{n}} && \text{From \cref{eq:martin22}}\nonumber \\
&\leq    \sup_{w \in \Wbbb } \braces{ \Phi(q_0;w)+C_1\eta \|w\|_2+C_1\eta^2 -0.5\lambda\|w\|^2_2+\lambda\eta^2} && \text{From \cref{eq:martin}} \nonumber \\
&\lesssim   \sup_{w \in \Wbbb }  \braces{\Phi(q_0;w)-0.25\lambda\|w\|^2_2+(C^2_1/\lambda+C_1+\lambda) \eta^2}.  \label{eq:help3}
\end{align}
In the last line, we use a general inequality, $a,b>0$:
\begin{align*}
    \sup_{f \in F}(a\|f\|-b\|f\|^2 )\leq a^2/4b. 
\end{align*}
Moreover,
\begin{align*}
    \sup_{w\in \Wbbb}\{ \Phi^{\lambda}_n(\hat q;w) \}&=     \sup_{w\in \Wbbb}\{ \Phi_n(\hat q;w)- \Phi_n(q_0;w)+ \Phi_n(q_0;w) -\lambda \|w\|^2_{n}\}\\
    &\geq     \sup_{w\in \Wbbb}\{ \Phi_n(\hat q;w)- \Phi_n(q_0;w) - 2\lambda \|w\|^2_{n}\}+  \inf_{w\in \Wbbb}\{ \Phi_n(q_0;w) +\lambda  \|w\|^2_{n}\} \\
     &=     \sup_{w\in \Wbbb}\{ \Phi_n(\hat q;w)- \Phi_n(q_0;w) - 2\lambda \|w\|^2_{n}\}+  \inf_{-w\in \Wbbb}\{ \Phi_n(q_0;-w) + \lambda \|w\|^2_{n}\} \\
         &=     \sup_{w\in \Wbbb}\{ \Phi_n(\hat q;w)- \Phi_n(q_0;w) - 2\lambda \|w\|^2_{n}\}+  \inf_{-w\in \Wbbb}\{-\Phi_n(q_0;w) + \lambda \|w\|^2_{n}\} \\
         &=     \sup_{w\in \Wbbb}\{ \Phi_n(\hat q;w)- \Phi_n(q_0;w) - 2\lambda \|w\|^2_{n}\}-  \sup_{-w\in \Wbbb}\{\Phi_n(q_0;w)- \lambda \|w\|^2_{n}\} \\
    & =     \sup_{w\in \Wbbb}\{ \Phi_n(\hat q;w)- \Phi_n(q_0;w)-2\lambda \|w\|^2_{n}\}-\sup_{w\in \Wbbb}\Phi^{\lambda}_n(q_0;w).
\end{align*}
Here, we use $\Wbbb$ is symmetric.  

Therefore, 
\begin{align*}
    \sup_{w\in \Wbbb}\{ \Phi_n(\hat q;w)- \Phi_n(q_0;w)-2\lambda \|w\|^2_{n} \}&\leq \sup_{w\in \Wbbb}\braces{\Phi^{\lambda}_n(q_0;w)}+    \sup_{w\in \Wbbb}\{ \Phi^{\lambda}_n(\hat q;w) \}\\
    &\leq 2\sup_{w\in \Wbbb}\Phi^{\lambda}_n(q_0;w) \tag{\cref{eq:obvious_def}}\\
    & \lesssim   \sup_{w \in \Wbbb }  \braces{\Phi(q_0;w)-0.25\lambda \|w\|^2_2+(C^2_1/\lambda+C_1+\lambda)\eta^2} \tag{ \cref{eq:help3}}\\ 
        & \lesssim     \sup_{w \in \Wbbb }  \braces{ \|\Tcal_{\gamma}(q_{\epol}-q_0)\|_2 \|w\|_2-0.25\lambda \|w\|^2_2 +(C^2_1/\lambda+C_1+\lambda)\eta^2}\\
   & \overset{(a)}\lesssim     \|\Tcal_{\gamma}(q_{\epol}-q_0)\|^2_2/\lambda +(C^2_1/\lambda+C_1+\lambda)\eta^2.  
\end{align*}
For (a), we use
\begin{align*}
    \Phi(q_0;w)= |\E[\Tcal_{\gamma}(q_{\epol}-q_0)w]|\leq \|\Tcal_{\gamma}(q_{\epol}-q_0)\|_2 C_{\Wbbb}. 
\end{align*}

\textbf{Second Step }

\begin{align*}
    w_q=\argmin_{w\in \Wbbb}\|w-\Tcal_{\gamma}(q-q_0)\|^2_2.  
\end{align*}
Suppose $\|w_{\hat q}\|_2\geq \eta$, and let $\kappa\coloneqq \eta/\{2\|w_{\hat q}\|_2\}\in [0,0.5]$. Then, noting $\Wbbb$ is star-convex,
\begin{align*}
  \sup_{w\in \Wbbb}\{ \Phi_n(\hat q;w)- \Phi_n(q_0;w)-2\lambda \|w\|^2_{n} \}\geq \kappa\{\Phi_n(\hat q,w_{\hat q})-\Phi_n(q_0,w_{\hat q})\}-2\kappa^2\lambda\|w_{\hat q}\|^2_{n}. 
\end{align*}
Here, we use $\kappa w_{\hat q}\in \Wbbb$. 
Then,
\begin{align*}
    \kappa^2\|w_{\hat q}\|^2_{n}& \lesssim \kappa^2\{1.5\|w_{\hat q}\|^2_{2}+0.5\eta^2\} && \text{From \cref{eq:martin}} \\
    & \lesssim \eta^2.  && \text{From definition of $\kappa$}
\end{align*}
Therefore, 
\begin{align*}
     \sup_{w\in \Wbbb}\{ \Phi_n(\hat q;w)- \Phi_n(q_0;w)-2\|w\|^2_{n} \}\geq \kappa\{\Phi_n(\hat q,w_{\hat q})-\Phi_n(q_0,w_{\hat q})\}-2\lambda \eta^2. 
\end{align*}
Observe that
\begin{align*}
    \Phi_n(q,w_q)-  \Phi_n(q_0,w_q)&=-\E_n[\{q(s,a)-q_0(s,a)-\gamma q(s',\epol)+ \gamma q_0(s',\epol)\}w_q(s,a)].
\end{align*}
Therefore, from Lemma \ref{lem:support2} noting $\eta$ upper-bounds the critical radius of $\Gcal_q$, for any $ q\in \Qbbb$, 
\begin{align*}
    &|\Phi_n(q,w_q)-  \Phi_n(q_0,w_q)-\{\Phi(q,w_q)-  \Phi(q_0,w_q)\}|\\
    &=|\E_n[\{-q(s,a)+q_0(s,a)+\gamma q(s',\epol)- \gamma q_0(s',\epol)\}w_q(s,a)  ]\\ &-\E[\{-q(s,a)+q_0(s,a)+\gamma q(s',\epol)- \gamma q_0(s',\epol)\}w_q(s,a))  ]  |\\ 
    &\leq (\eta\E[\{-q(s,a)+q_0(s,a)+\gamma q(s',\epol)- \gamma q_0(s',\epol)\}^2w^2_q(s,a)]^{1/2}+\eta^2)\\
 &\lesssim (\eta C_1\|w_q\|_2+\eta^2). 
\end{align*}
Here, we invoke Lemma \ref{lem:support2} by treating $l(a_1,a_2)=a_1,\,a_1= \{q(s,a)-q_0(s,a)-\gamma q(s',\epol)+ \gamma q_0(s',\epol)\}\{\Tcal_{\gamma}(q-q_0)(s,a)\}$. 

Thus, 
\begin{align*}
    \kappa\{\Phi_n(\hat q,w_{\hat q})-  \Phi_n(q_0,w_{\hat q})\}& \geq  \kappa\{\Phi(\hat q,w_{\hat q})-  \Phi(q_0,w_{\hat q})\}-\kappa(C_1\eta\|w_{\hat q}\|_2+\eta^2). \\
    &\geq \kappa\{\Phi(\hat q,w_{\hat q})-  \Phi(q_0,w_{\hat q})\}-\kappa(C_1\eta\|w_{\hat q}\|_2)-0.5\eta^2 \\
    &\overset{(a)}{=} \kappa\E[w_{\hat q}(s,a)\Tcal_{\gamma}(\hat q-q_0)(s,a)]-\kappa(C_1\eta\|w_{\hat q}\|_2)-0.5\eta^2 \\
  &= \frac{\eta}{2\|w_{\hat q}\|}\{\E[w_{\hat q}(s,a)\Tcal_{\gamma}(\hat q-q_0)(s,a)]-C_1\eta\|w_{\hat q}\|_2\}-0.5\eta^2 \\
    &\overset{(b)}{\geq}0.5\eta\{\|\Tcal_{\gamma}(\hat q-q_0)\|_2-2\epsilon_n\}-(0.5+C_1)\eta^2. 
\end{align*}
For (a),  we use 
\begin{align*}
    \Phi(\hat q,w_{\hat q})-  \Phi(q_0,w_{\hat q})&=\E[w_{\hat q}(s,a)\{-\hat q(s,a)+q_0(s,a)+\gamma  \hat q(s',\epol)- \gamma q_0(s',\epol)\}]\\
    &= \E[w_{\hat q}(s,a)\Tcal_{\gamma}(\hat q-q_0)(s,a)].
\end{align*}
For (b), we use 
\begin{align*}
\frac{\E[w_{\hat q}(s,a)\Tcal_{\gamma}(\hat q-q_0)(s,a)]}{\|w_{\hat q}\|_2 }&= \frac{\E[w_{\hat q}(s,a)\{-w_{\hat q}+w_{\hat q}+\Tcal_{\gamma}(\hat q-q_0)\}(s,a)]}{\|w_{\hat q}\|_2 }\\ 
&\geq \frac{\|w_{\hat q}\|^2_2-\|w_{\hat q}\|_2\|w_{\hat q}-\Tcal_{\gamma}(\hat q-q_0)\|_2}{\|w_{\hat q}\|_2 }\\ 
&=\|w_{\hat q}\|_2-\|w_{\hat q}-\Tcal_{\gamma}(\hat q-q_0)\|_2\\
&=\|w_{\hat q}\|_2-\epsilon_n\geq \|w_{\hat q}-\Tcal_{\gamma}(\hat q-q_0)+\Tcal_{\gamma}(\hat q-q_0)\|_2-\epsilon_n \\
&\geq \|\Tcal_{\gamma}(\hat q-q_0)\|_2-2\epsilon_n. 
\end{align*}

\textbf{Combining all results}

Thus, $\|w_{\hat q}\|\leq \eta$ or 
\begin{align*}
\eta\{\|\Tcal_{\gamma}(\hat q- q_0)\|_2-\epsilon_n\}-(1+C_1+\lambda )\eta^2\lesssim \frac{\|\Tcal_{\gamma}(q_{\epol}- q_0)\|^2_2}{\lambda}+(C_1+C_1^2/\lambda+\lambda)\eta^2. 
\end{align*}
Therefore, we have 
\begin{align*}
        \|\Tcal_{\gamma}(\hat q- q_0)\|_2&\leq   \|w_{\hat q}-\Tcal_{\gamma}(\hat q- q_0)\|_2+ \|w_{\hat q}\| \\
          &\leq  \eta+\epsilon_n. 
\end{align*}
or 
\begin{align*}
    \|\Tcal_{\gamma}(\hat q- q_0)\|_2\lesssim (1+C_1+\lambda+C^2_1/\lambda)  \eta+\epsilon_n+\frac{\|\Tcal_{\gamma}(q_{\epol}- q_0)\|^2_2 }{\lambda \eta}. 
\end{align*}
In the end, from a triangle inequality, 
\begin{align*}
    \|\Tcal_{\gamma}(\hat q- q_{\epol})\|_2\lesssim  (1+C_1+\lambda+C^2_1/\lambda)  \eta+\epsilon_n+\frac{\|\Tcal_{\gamma}(q_{\epol}- q_0)\|^2_2 }{\lambda \eta}+\|\Tcal_{\gamma}(q_{\epol}- q_0)\|_2. 
\end{align*}

\end{proof}

\subsection{Proof of \cref{ape:l2_error} }


\begin{proof}[Proof of \cref{thm:fast_q_l2}]
~\\
\textbf{Preliminary}
Define 
\begin{align*}
    \Phi(q;w) &=\E[\{r-q(s,a)+\gamma q(s',\epol)\}w(s,a)],\\
  \Phi_n(q;w) &=\E_n[\{r-q(s,a)+\gamma q(s',\epol)\}w(s,a)],\\
     \Phi^{\lambda}_n(q;w) &= \Phi_n(q;w)-\lambda \|w\|^2_{n},\\
         \Phi^{\lambda}(q;w) &=\Phi(q;w)-\lambda\|w\|^2_{2}. 
\end{align*}

From Lemma \ref{lem:support1}, we have 
\begin{align}\label{eq:martin}
    \forall w\in \Wbbb,\,|\|w\|_{n}^2-\|w\|^2_2|\leq 0.5\|w\|^2_2+\eta^2
\end{align}
for our choice of $\eta:=\eta_n+\sqrt{c_0\log(c_1/\delta)/n}$ noting $\eta_n$ upper bounds the critical radius of $\Wbbb$.

\textbf{First Step}

First step is conducted as in the proof of \cref{thm:fast_q}. 
Therefore, 
\begin{align*}
    \sup_{w\in \Wbbb}\{ \Phi_n(\hat q;w)- \Phi_n(q_{\epol};w)-2\lambda \|w\|^2_{n} \}&\lesssim (C^2_1/\lambda+C_1+\lambda)\eta^2. 
\end{align*}

\textbf{Second Step }

Given $q$,we define $w_q$ s.t. 
\begin{align*}
   \Tcal'_{\gamma} w_q=C_{\xi}(q-q_{\epol}). 
\end{align*}
Suppose $\|\hat q-q_{\epol}\|_2\geq C_{\xi}\eta$, and let $\kappa\coloneqq \eta/\{2C_{\xi}\|\hat q-q_{\epol}\|_2\}\in [0,0.5]$. Then, noting $\Wbbb$ is star-convex,
\begin{align*}
  \sup_{w\in \Wbbb}\{ \Phi_n(\hat q;w)- \Phi_n(q_{\epol};w)-2\lambda \|w\|^2_{n} \}\geq \kappa\{\Phi_n(\hat q,w_{\hat q})-\Phi_n(q_{\epol},w_{\hat q})\}-2\kappa^2\lambda\|w_{\hat q}\|^2_{n}. 
\end{align*}
Here, we use $\kappa w_{\hat q}\in \Wbbb$. 
Then,
\begin{align*}
    \kappa^2\|w_{\hat q}\|^2_{n}& \lesssim \kappa^2\{1.5\|w_{\hat q}\|^2_{2}+0.5\eta^2\}  \tag{\cref{eq:martin}} \\
    & \lesssim \eta^2 \frac{\|w_{\hat q}\|^2_{2}}{\|\Tcal'_{\gamma}w_{\hat q}\|^2_{2}}+\eta^2.   \tag{Definition of $\kappa$}\\
    &  \lesssim \eta^2 C^2_{\iota}+\eta^2.  \tag{Definition of $C_{\iota}$}
\end{align*}
Therefore, 
\begin{align*}
     \sup_{w\in \Wbbb}\{ \Phi_n(\hat q;w)- \Phi_n(q_{\epol};w)-2\|w\|^2_{n} \}\geq \kappa\{\Phi_n(\hat q,w_{\hat q})-\Phi_n(q_{\epol},w_{\hat q})\}-2\lambda\{1+C^2_{\iota}\} \eta^2. 
\end{align*}
Observe that
\begin{align*}
    \Phi_n(q,w_q)-  \Phi_n(q_{\epol},w_q)&=-\E_n[\{q(s,a)-q_{\epol}(s,a)-\gamma q(s',\epol)+ \gamma q_{\epol}(s',\epol)\}w_q(s,a)].
\end{align*}
Therefore, from Lemma \ref{lem:support2} noting $\eta$ upper-bounds the critical radius of $\Gcal_q$, for any $ q\in \Qbbb$, 
\begin{align*}
    &|\Phi_n(q,w_q)-  \Phi_n(q_{\epol},w_q)-\{\Phi(q,w_q)-  \Phi(q_{\epol},w_q)\}|\\
    &=|\E_n[\{-q(s,a)+q_{\epol}(s,a)+\gamma q(s',\epol)- \gamma q_{\epol}(s',\epol)\}w_q(s,a)  ]\\ &-\E[\{-q(s,a)+q_{\epol}(s,a)+\gamma q(s',\epol)- \gamma q_{\epol}(s',\epol)\}w_q(s,a))  ]  |\\ 
    &\leq (\eta\E[\{-q(s,a)+q_{\epol}(s,a)+\gamma q(s',\epol)- \gamma q_{\epol}(s',\epol)\}^2w^2_q(s,a)]^{1/2}+\eta^2)\\
 &\lesssim \eta C_{\Wbbb}(1+\gamma \sqrt{C_mC_{\eta}})\|q-q_{\epol}\|_2+\eta^2. 
\end{align*}
Here, we invoke Lemma \ref{lem:support2} by treating $l(a_1,a_2)=a_1,\,a_1= \{q(s,a)-q_{\epol}(s,a)-\gamma q(s',\epol)+ \gamma q_{\epol}(s',\epol)\}w_q(s,a)\}$. 

Thus, 
\begin{align*}
    \kappa\{\Phi_n(\hat q,w_{\hat q})-  \Phi_n(q_{\epol},w_{\hat q})\}& \geq  \kappa\{\Phi(\hat q,w_{\hat q})-  \Phi(q_{\epol},w_{\hat q})\}-\kappa(\eta C_{\Wbbb}(1+\gamma \sqrt{C_mC_{\eta}})\|q-q_{\epol}\|_2+\eta^2). \\
    &\geq \kappa\{\Phi(\hat q,w_{\hat q})-  \Phi(q_{\epol},w_{\hat q})\}-\kappa  \eta C_{\Wbbb}(1+\gamma \sqrt{C_mC_{\eta}})\|q-q_{\epol}\|_2-0.5\eta^2 \\
    &= \kappa\E[w_{\hat q}(s,a)\Tcal_{\gamma}(\hat q-q_{\epol})(s,a)]-\kappa  \eta C_{\Wbbb}(1+\gamma \sqrt{C_mC_{\eta}})\|q-q_{\epol}\|_2-0.5\eta^2 \\
  &= \frac{\eta}{2C_{\xi}\|(\hat q-q_{\epol})\|_2}\{C_{\xi}\|(\hat q-q_{\epol})\|^2_2-\eta C_{\Wbbb}(1+\gamma \sqrt{C_mC_{\eta}})\|q-q_{\epol}\|_2\}-0.5\eta^2 \\
    &\geq 0.5\eta \|(\hat q-q_{\epol})\|_2-(0.5+ C_{\Wbbb}(1+\gamma \sqrt{C_mC_{\eta}}))\eta^2. 
\end{align*}
From the second line to the third line, we use 
\begin{align*}
    \Phi(\hat q,w_{\hat q})-  \Phi(q_{\epol},w_{\hat q})&=\E[w_{\hat q}(s,a)\{-\hat q(s,a)+q_{\epol}(s,a)+\gamma  \hat q(s',\epol)- \gamma q_{\epol}(s',\epol)\}]\\
    &= \E[w_{\hat q}(s,a)\Tcal_{\gamma}(\hat q-q_{\epol})(s,a)]\\
    &=\E[\Tcal'_{\gamma}w_{\hat q}(s,a)(\hat q-q_{\epol})(s,a)]=C_{\xi}\|\hat q-q_{\epol}\|^2_2. 
\end{align*}
\textbf{Combining all results}

Thus, $C_{\xi}\|\hat q-q_{\epol}\|_2\leq \eta$ or 
\begin{align*}
\eta \|\hat q- q_{\epol}\|_2-(1+C_{\Wbbb}(1+\gamma\sqrt{C_mC_{\eta}}))\eta^2-\lambda(1+C^2_{\iota} )\eta^2\lesssim (C_1+C_1^2/\lambda+\lambda)\eta^2. 
\end{align*}
In the end, 
\begin{align*}
    \|\hat q- q_{\epol}\|_2&\lesssim  (1+C_1+C_1^2/\lambda+\lambda\{1+C^2_{\iota}\}+C_{\Wbbb}(1+\gamma \sqrt{C_mC_{\eta}})+1/C_{\xi} )  \eta\,\\
    C_1 &= C_{\Qbbb}+ (1-\gamma^{-1})R_{\max}. 
\end{align*}

\end{proof}

\begin{proof}[Proof of \cref{thm:fast_w_l2}]
~ \\
\textbf{Preliminary}\\

Define 
\begin{align*}
    \Xi(w,q)  &=\E[w(s,a)\{-q(s,a)+\gamma q(s',\epol) \}+(1-\gamma)\E_{d_0}[q(s_0,\epol)] ],\\ 
    \Xi_n(w,q)&=\E_n[w(s,a)\{-q(s,a)+\gamma q(s',\epol) \}+(1-\gamma)\E_{d_0}[q(s_0,\epol)] ],\\
    \Xi^{\lambda}(w,q) &=    \Xi^{\lambda}(w,q)-\lambda \|q(s,a)-\gamma q(s',\epol) \|^2_2,\\
    \Xi^{\lambda}_n(w,q) &=    \Xi^{\lambda}_n(w,q)-\lambda \|q(s,a)-\gamma q(s',\epol) \|^2_{n}.
\end{align*}
From Lemma \ref{lem:support1}, we have
\begin{align}\label{eq:help4}
    \forall q\in \Qbbb, |\|-q(s,a)+\gamma q(s',\epol) \|^2_{n}-\|-q(s,a)+\gamma q(s',\epol)\|^2_2  |\leq 0.5\|-q(s,a)+\gamma q(s',\epol)\|^2_2+0.5\eta^2. 
\end{align}
for our choice of $\eta:=\eta_n+c_0\sqrt{\log(c_1/\delta)/n}$ noting $\eta_n$ upper bounds the critical of $\Gcal_{w1}$. 

\textbf{First Step }\\

As in the proof of \cref{thm:fast_w}, we have 
\begin{align*}
        \sup_{q\in \Qbbb}\{\Xi_n(\hat w;q)- \Xi(w_{\epol};q)-2\| -q(s,a)+\gamma q(s',\epol)\|^2_{n}\}
         &\leq (\lambda +C^2_{\Wbbb}/\lambda +C_{\Wbbb})\eta^2. 
\end{align*}

\textbf{Second Step}\\ 

Given $w$, we define $q_w\in \Qbbb$ s.t. $\Tcal_{\gamma}q_{w}\coloneqq C_{\xi}(w-w_{\epol})$ noting $C_{\xi}\{\Wbbb-w_{\epol}\} \subset \Tcal_{\gamma}\Qbbb$. Suppose $C_{\xi}\|\hat w-w_{\epol}\|_2\geq \eta$, and let $\kappa\coloneqq \eta/\{2C_{\xi}\|\hat w-w_{\epol} \|_2\}\in [0,0.5]$. Then, noting $\Qbbb$ is star-convex,
\begin{align*}
    &\sup_{q\in \Qbbb}\{\Xi_n(\hat w;q)- \Xi(w_{\epol};q)-2\lambda \|q(s,a)-\gamma q(s',\epol)\|^2_{n}\}\\
    &\geq \kappa\{\Xi_n(\hat w;q_{\hat w})- \Xi(w_{\epol};q_{\hat w})\} -2\kappa^2\lambda \|q_{\hat w}(s,a)-\gamma q_{\hat w}(s',\epol)\|^2_{n}. 
\end{align*}
Here, we use $\kappa q_{\hat w}\in \Qbbb$, and take the value of $\kappa q_{\hat w}$. Then, from \cref{eq:help4},  
\begin{align*}
   \kappa^2 \|-q_{\hat w}(s,a)+\gamma q_{\hat w}(s',\epol)\|^2_{n} & \lesssim\kappa^2\{1.5 \|-q_{\hat w}(s,a)+\gamma q_{\hat w}(s',\epol)\|^2_{2}+\eta^2\}\\
    &\lesssim \kappa^2(1+ C_mC_{\eta})\|q_{\hat w}\|^2_2 +\eta^2 \\ 
      &\lesssim \frac{\eta^2}{\|\Tcal_{\gamma} q_{\hat w}\|^2_2} (1+ C_mC_{\eta})\|q_{\hat w}\|^2_2 +\eta^2 \\ 
       &\lesssim \eta^2 (1+ C_mC_{\eta})C^2_{\iota} +\eta^2  \tag{Definition of $C_{\iota}$}\\ 
\end{align*}
Therefore, 
\begin{align*}
     \sup_{q\in \Qbbb}\{\Xi_n(\hat w;q)- \Xi(w_{\epol};q)-2\lambda \| -q(s,a)+\gamma q(s',\epol)\|^2_{n}\}\geq \kappa(\Xi_n(\hat w;q_{\hat w})- \Xi(w_{\epol};q_{\hat w})) -2\lambda \eta^2. 
\end{align*}
Observe that 
\begin{align*}
    \Xi_n( w;q_w)- \Xi_n(w_{\epol};q_w)=\E_n[(w(s,a)-w_{\epol}(s,a))(-q_{w}(s,a)+\gamma q_{w}(s',\epol))  ]. 
\end{align*}
Therefore, from Lemma \ref{lem:support2}, noting $\eta$ upper bounds the critical radius of $\Gcal_{w2}$, 
regarding $l(a_1,a_2)=a_1,a_1=(w(s,a)-w_{\epol}(s,a))(q_{w}(s,a)-\gamma q_{w}(s',\epol)) $, $\forall w\in \Wbbb$,
\begin{align*}
  & |   \Xi_n( w;q)- \Xi_n(w_{\epol};q)-\{\Xi( w;q)- \Xi(w_{\epol};q)\}|\\
&\leq (\eta  \E[\{(w(s,a)-w_{\epol}(s,a))(q_{w}(s,a)-\gamma q_{w}(s',\epol))\}^{2}]^{1/2}+\eta^2)\\ 
&\leq (\eta  \E[(w(s,a)-w_{\epol}(s,a))^2]^{1/2}C_{\Qbbb}+\eta^2). 
\end{align*}
Here, we use $\|q_{w}(s,a)-\gamma q_{w}(s',\epol))\|_{\infty}\lesssim C_{\Qbbb}$. 

Thus, 
\begin{align*}
    &\kappa\{  \Xi_n(\hat w;q_{\hat w})- \Xi_n(w_{\epol};q_{\hat w})\}\\
    & \geq \kappa\{\Xi(\hat w;q_{\hat w})- \Xi(w_{\epol};q_{\hat w})\}-\kappa(\eta  \|\hat w-w_{\epol}\|_2 C_{\Qbbb} +\eta^2 )  && \\
    &\geq \kappa\{\Xi( \hat w;q_{\hat w})- \Xi(w_{\epol};q_{\hat w})\}-\kappa(\eta  \|\hat w-w_{\epol}\|_2 C_{\Qbbb} )-0.5\eta^2 && \text{Use $\kappa \leq 0.5$} \\ 
    &=r C_{\xi}\|\hat w-w_{\epol}\|^2_2    -\kappa(\eta  \|\hat w-w_{\epol}\|_2 C_{\Qbbb})-0.5\eta^2 && \\ 
    &=\frac{\eta}{2 C_{\xi}\|\hat w-w_{\epol}\|_2 }\{C_{\xi}\|\hat w-w_{\epol}\|^2_2 -C_{\Qbbb}\eta \|\hat w-w_{\epol}\|_2\}-0.5\eta^2 && \\ 
  &=0.5\eta \|\hat w-w_{\epol}\|_2  -(0.5+0.5C_{\Qbbb})\eta^2 . 
\end{align*}
 From the third line to the fourth line, we use 
\begin{align*}
    \Xi( w;q_{ w})- \Xi(w_{\epol};q_{w})&=\E[(w(s,a)-w_{\epol}(s,a))(-q_{w}(s,a)+\gamma q_{w}(s',\epol))]\\
    &=\E[(w(s,a)-w_{\epol}(s,a))\Tcal_{\gamma}q_w]\\
   &=C_{\xi}\E[(w(s,a)-w_{\epol}(s,a))(w(s,a)-w_{\epol}(s,a))]= C_{\xi}\|w-w_{\epol}\|^2_2.
\end{align*}

\textbf{Combining all results}

Thus, $ C_{\xi}\|\hat w-w_{\epol}\|_2 \leq \eta$ or 
\begin{align*}
\eta  \|\hat w-w_{\epol}\|_2-(1+C_{\Qbbb})\eta^2-\lambda \eta^2C^2_{\iota}(1+C_mC_{\eta}) \lesssim (\lambda +C^2_{\Wbbb}/\lambda+C_{\Wbbb})\eta^2. 
\end{align*}
Therefore, 
\begin{align*}
     \|\hat w-w_{\epol}\|_2\leq \braces{1+C^2_{\Wbbb}/\lambda+C_{\Wbbb}+\lambda\{1+C^2_{\iota}(1+C_mC_{\eta})\}+C_{\Qbbb}+1/ C_{\xi}}\eta. 
\end{align*}
\end{proof}

%

\subsection{Proof of \cref{ape:fqi_analysis}}


\begin{proof}[Proof of Lemma \ref{lem:fast_rate_reg}]
We define 
\begin{align*}
    \Phi_n(q)=\E_{\zeta_k}[\{r+\gamma f_{k-1}(s')-q(s,a)\}^2-\{r+\gamma f_{k-1}(s')-\Bcal f_{k-1}(s,a)\}^2 ]. 
\end{align*}
From the definition of the estimator and the assumption $\Bcal f_{k-1}\in \Qbbb$ using $\Bcal \Qbbb\subset \Qbbb$,  
\begin{align*}
      \Phi_n(f_k)\leq 0. 
\end{align*}
Note
\begin{align*}
    &\E_{\zeta_k}[\{r+\gamma f_{k-1}(s')-q(s,a)\}^2-\{r+\gamma f_{k-1}(s')-\Bcal f_{k-1}(s,a)\}^2 ]\\
    &=\E_{\zeta_k}[\{-q(s,a)+\Bcal f_{k-1}(s,a)\}\{2r+2\gamma f_{k-1}(s')-\Bcal f_{k-1}(s,a)-q(s,a)\}]. 
\end{align*}
By Lemma \ref{lem:support2}, $   \forall q\in \Qbbb$, 
\begin{align*}
    &|\Phi_n(q)-\Phi(q)\}\\
    &\lesssim \eta(\{\E[\{-q(s,a)+\Bcal f_{k-1}(s,a)\}^2\{2r+2\gamma f_{k-1}(s')-\Bcal f_{k-1}(s,a)-q(s,a)\}^2]\}^{1/2}+\eta)\\
   &\lesssim \eta (\{R_{\max}+C_{\Qbbb}\}\|q-\Bcal f_{k-1}\|_2+\eta).
\end{align*}
We have
\begin{align*}
   0\geq \Phi_n(f_k)&= \Phi_n(f_k)-\Phi(f_k)+\Phi(f_k)\\ 
    &\geq \Phi(f_k)-\eta (\{R_{\max}+C_{\Qbbb}\}\|f_k-\Bcal f_{k-1}\|_2+\eta) \\ 
    &=\|f_k-\Bcal f_{k-1}\|^2_2- \eta (\{R_{\max}+C_{\Qbbb}\}\|f_k-\Bcal f_{k-1}\|_2+\eta). 
\end{align*}
This implies $\|f_k-\Bcal f_{k-1}\|_2\lesssim (R_{\max}+C_{\Qbbb})\eta$. 
\end{proof}

\begin{proof}[Proof of Lemma \ref{lem:fqipop}]
We define $\{\E_{\mu}[f(\cdot)^q]\}^{1/q}=\|\cdot\|_{q,\mu}$. 
\begin{align}
|\E_{s\sim d_0}[f_K(s, \epol)] - \E_{s\sim d_0}[q_{\epol}(s, \epol)]| = &~ \|f_K - q_{\epol}\|_{1,d_{{\epol},1}} \tag{$d_{{\epol},1} \coloneqq d_0 \times \epol$}
\\
\leq &~ \|f_K - q_{\epol}\|_{2,d_{{\epol},1}}  \tag{Jensen}
\\
\leq &~ \|f_K - \Bcal f_{K - 1}\|_{2,d_{{\epol},1}} + \|\Bcal f_{K - 1} - q_{\epol}\|_{2,d_{{\epol},1}} \nonumber
\\
= &~ \|f_K - \Bcal f_{K - 1}\|_{2,d_{{\epol},1}} + \|\Bcal f_{K - 1} - \Bcal q_{\epol}\|_{2,d_{{\epol},1}} \nonumber
\\
\overset{\text{(a)}}{\leq} &~ \|f_K - \Bcal f_{K - 1}\|_{2,d_{{\epol},1}} + \gamma \|f_{K - 1} - q_{\epol}\|_{2,d_{{\epol},2}} \nonumber
\\
\label{eq:stareq}
\overset{\text{(b)}}{\leq} &~ \|f_K - \Bcal f_{K - 1}\|_{2,d_{{\epol},1}} + \sqrt{\gamma} \|f_{K - 1} - q_{\epol}\|_{2,\gamma d_{{\epol},2}}, \tag{$\star$}
\end{align}
where (a) is from the following argument and (b) follows from the first part of \citep[Fact 5.1]{agarwal2019reinforcement},
\begin{align*}
\|\Bcal f_{K - 1} - \Bcal q_{\epol}\|_{2,d_{{\epol},1}}^2 = &~ \E_{(s,a) \sim d_{{\epol},1}}\left[\big((\Bcal f_{K - 1})(s,a) - ( \Bcal q_{\epol})(s,a)\big)^2\right]
\\
= &~ \gamma^2 \E_{(s,a) \sim d_{{\epol},1}}\left[\left(\E_{(s',a')\sim \Pr(\cdot,\cdot|s,a,\epol)}[f_{K - 1}(s',a') - q_{\epol}(s',a')]\right)^2\right]
\\
\leq &~ \gamma^2 \E_{(s,a) \sim d_{{\epol},1}}\left[\E_{(s',a')\sim \Pr(\cdot,\cdot|s,a,\epol)}\left[\left(f_{K - 1}(s',a') - q_{\epol}(s',a')\right)^2\right]\right] \tag{Jensen}
\\
= &~ \gamma^2 \E_{(s',a') \sim d_{{\epol},2}}\left[\left(f_{K - 1}(s',a') - q_{\epol}(s',a')\right)^2\right]
\\
= &~ \gamma^2 \|f_{K - 1} - q_{\epol}\|_{2,d_{{\epol},2}}^2,
\end{align*}
where $\Pr(\cdot,\cdot|s,a,\epol)$ denotes the distribution of next the state-action pair from $(s,a)$ when following policy $\epol$.

By applying \eqref{eq:stareq} inductively, we have
\begin{align}
&~ |\E_{s\sim d_0}[f_K(s, \epol)] - \E_{s\sim d_0}[q_{\epol}(s, \epol)]|
\\
= &~ \|f_K - q_{\epol}\|_{1,d_{{\epol},1}} \nonumber
\\
\leq &~ \|f_K - \Bcal f_{K - 1}\|_{2,d_{{\epol},1}} + \sqrt{\gamma} \|f_{K - 1} - q_{\epol}\|_{2,\gamma d_{{\epol},2}} \nonumber
\\
\leq &~ \|f_K - \Bcal f_{K - 1}\|_{2,d_{{\epol},1}} + \sqrt{\gamma} \|f_{K - 1} - \Bcal f_{K - 2}\|_{2,\gamma d_{{\epol},2}} + \gamma \|f_{K - 2} - q_{\epol}\|_{2,\gamma^2 d_3^{\epol}} \nonumber
\\
\vdots \nonumber
\\
\leq &~ \sum_{k = 0}^{K-1} \gamma^{k/2} \|f_{K - k} - \Bcal f_{K - 1 - k}\|_{2,\gamma^k d_{k + 1}^{\epol}} + \gamma^{{K}/{2}} \|f_{0} - q_{\epol}\|_{2,d_{{\epol},{K+1}}}. \nonumber
\end{align}

By the definition of $d_{\epol}$, we know $\gamma^k d_{\epol,k+1} \leq \frac{d_{\epol,\gamma}}{1 - \gamma}$ for any $k \geq 1$. Therefore, using the second part of \citep[Fact 5.1]{agarwal2019reinforcement}, we obtain
\begin{align*}
&~ |\E_{s\sim d_0}[f_K(s, \epol)] - \E_{s\sim d_0}[q_{\epol}(s, \epol)]|
\\
\leq &~ \sum_{k = 0}^{K-1}\frac{1}{\sqrt{1-\gamma}} \gamma^{k/2} \|f_{K - k} - \Bcal f_{K - 1 - k}\|_{2,d_{\epol,\gamma}} + \gamma^{{K}/{2}} (C_{\Qbbb}+(1-\gamma)^{-1}R_{\max})
\\
\leq &~ \sum_{k = 0}^{K-1} \frac{1}{\sqrt{1-\gamma}}\gamma^{k/2} \|w_{\epol}\|_2 \|f_{K - k} - \Bcal f_{K - 1 - k}\|_{2} + \gamma^{{K}/{2}} (C_{\Qbbb}+(1-\gamma)^{-1}R_{\max})
\\
\leq &~ \|w_{\epol}\|_2\frac{1}{\sqrt{1-\gamma}}\varepsilon \sum_{k = 0}^{K-1} \gamma^{k/2} + \gamma^{{K}/{2}} (C_{\Qbbb}+(1-\gamma)^{-1}R_{\max})
\\
\leq &~ \frac{1 - \gamma^{K/2}}{\sqrt{1-\gamma}(1 - \gamma^{1/2})}\|w_{\epol}\|_2 \varepsilon + \gamma^{{K}/{2}} (C_{\Qbbb}+(1-\gamma)^{-1}R_{\max}). 
\end{align*}
\end{proof}

\subsection{Proof of auxiliary lemmas}


\begin{lemma}\label{lem:minimax}
\begin{align*}
&\E[w(s,a)(r-q(s,a)+\gamma q(s',\epol))+(1-\gamma)\E_{d_0}[q(s_0,\epol)]]\\
&=\E[\{ w(s,a)-w_{\epol}(s,a)\}\{-q(s,a)+q_{\epol}(s,a)+\gamma v(s')-\gamma v_{\epol}(s)\}]+J(\epol).
\end{align*}
\end{lemma}

\begin{proof}[Proof of Lemma~\ref{lem:minimax}]
\begin{align*}
    &\E[w(s,a)\{r-q(s,a)+\gamma v(s')\}]+(1-\gamma)\E_{d_0}[v(s_0)]-J(\epol)+J(\epol)\\
   &=\E[w(s,a)\{r-q(s,a)+\gamma v(s')\}]+(1-\gamma)\E_{d_0}[v(s_0)] \\ 
   &-\{\E[w_{\epol}(s,a)\{r-q(s,a)+\gamma v(s')\}]+(1-\gamma)\E_{d_0}[v(s_0)]\}+J(\epol)\\
    &=\E[w(s,a)\{r-q(s,a)+\gamma v(s')\}]-\E[w_{\epol}(s,a)\{r-q(s,a)+\gamma v(s')\}]+J(\epol)\\ 
     &=\E[\{w(s,a)-w_{\epol}(s,a)\}\{-q(s,a)+q_{\epol}(s,a)+\gamma v(s')-\gamma v_{\epol}(s')\}]+J(\epol). 
\end{align*}
\end{proof}

\begin{lemma}[Dudley integral]\label{lem:dudley}
\begin{align*}
    \Rcal_n(\Fcal)\lesssim \inf_{\tau\geq 0}\braces{\tau+\int_{\tau}^{\sup_{f\in \Fcal}\sqrt{\E_n[f^2]}}\sqrt{\frac{\log \Ncal(\tau,\Fcal,\|\cdot\|_n)}{n}}}. 
\end{align*}
\end{lemma}
Note $\sup_{f\in \Fcal}\sqrt{\hP[f^2]}$ is upper bounded by the envelope $\|\Fcal\|_{\infty}$.
\begin{lemma}{\citep[Corollary 14.3]{WainwrightMartinJ2019HS:A}}\label{lem:critical_basic}
        Let $\Ncal(\tau;\mathbb{B}(\delta;\Fcal),\|\cdot\|_n)$ denote the $\tau$-covering number of $\mathbb{B}_n(\delta;\Fcal):=\{f| \|f\|_n\leq \delta\}$. Then, the critical inequality of the empirical version is satisfied for any $\delta>0$ s.t.
        \begin{align*}
            \frac{1}{\sqrt{n}}\int^{\delta}_{\delta^2/(2\|\Fcal\|_{\infty})}\sqrt{\log  \Ncal(t,\mathbb{B}_n(\delta;\Fcal),\|\cdot\|_n)}\rd t\leq \frac{\delta^2}{\|\Fcal\|_{\infty}}.
        \end{align*}
\end{lemma}

We analyze the main error terms in MDL and MIL. We assume $\Wbbb_2=\Qbbb_2=\Wbbb_1=\Wbbb_2$ for simplicity. We consider a joint hypothesis set and investigate its covering number:
\begin{align*}
    \Gcal(\Wbbb,\Qbbb) \ts &= \ts \{(s,a,s')\mapsto w(s,a)\{-q(s,a)+\gamma  q(s',\epol)\}; w\in \Wbbb,q\in \Qbbb \}.
\end{align*}
To analyze above, we often use the following lemma.

\begin{lemma}[Joint covering number] \label{lem:covering_G}
	\begin{align*}
		&\log \bN(\tau, \Gcal(\Wbbb,\Qbbb), \|\cdot\|_\infty) \\
		&\leq \log \bN(0.5\tau/((1+\gamma) C_\Qbbb), \Wbbb, \|\cdot\|_\infty) + \log \bN(0.5\tau/(C_\Wbbb (1+\gamma)), \Qbbb, \|\cdot\|_\infty).
	\end{align*}
\end{lemma}

\begin{proof}[Proof of Lemma \ref{lem:covering_G}]
Fix $\varepsilon_1 > 0$ and $\varepsilon_2>0$ and set $N_Q = \bN(\varepsilon_1, \Qbbb, \|\cdot\|_\infty)$ and $N_W = \bN(\varepsilon_2, \Wbbb, \|\cdot\|_\infty)$.
We pick an $\varepsilon_1$-covering set $\{q_i\}_{i=1}^{N_Q} \subset \Qbbb$, that is, for any $q \in \Qbbb$, there exists $q_i \in \{q_i\}_{i=1}^{N_Q} $ such as $\|q - q_i\|_\infty \leq \varepsilon_1$.
Similarly, we pick an $\varepsilon_2$-covering set $\{w_j\}_{j=1}^{N_W} \subset \Wbbb$.

We consider a functional $G: \Wbbb \times \Qbbb \to \Gcal$ such as
\begin{align*}
	G[w,q](s,a,r,s') := w(s,a)\{-q(s,a)+\gamma  \rE_{a' \sim \pi_e(a'|s')}[q(s',a')]\}.
\end{align*}
Simply, we have $\Gcal = \{G[w,q] ; w\in \Wbbb,q\in \Qbbb \}$.
Using the picked covering set, we consider a finite set $\tilde{\Gcal} := \{G[w_j,q_i] ; i = 1,...,N_Q, j = 1,...,N_W \} \subset \Gcal$.
Obviously, we have $|\tilde{\Gcal} |\leq N_W N_Q$.

We will show that $\tilde{\Gcal}$ is a covering set of $\Gcal$.
Pick $G[w,q] \in \Gcal$ arbitrary, then we can pick $w_j$ and $q_i$ such as $\|q-q_i\|_\infty \leq \varepsilon_1$ and $\|w-w_j\|_\infty \leq \varepsilon_2$.
Then, we evaluate $G[w_j,q_i] (s,a,r,s')- G[w,q](s,a,r,s')$ as
\begin{align*}
	&|G[w_j,q_i] (s,a,r,s')- G[w,q](s,a,r,s') |\\
	& = |w_j(s,a)\{-q_i(s,a)+\gamma  \rE_{a' \sim \pi_e(a'|s')}[q_i(s',a')]\} - w(s,a)\{q(s,a)+\gamma  \rE_{a' \sim \pi_e(a'|s')}[q(s',a')]\} | \\
	&=|(w_j(s,a) - w(s,a))\{-q_i(s,a)+\gamma  \rE_{a' \sim \pi_e(a'|s')}[q_i(s',a')]\} \\
	&  \quad + w(s,a)[-(q(s,a) - q_i(s,a)) + \gamma ( \rE_{a' \sim \pi_e(a'|s')}[q_i(s',a')] -  \rE_{a' \sim \pi_e(a'|s')}[q(s',a')])]| \\
	& \leq |w_j(s,a) - w(s,a)| |-q_i(s,a)+\gamma  \rE_{a' \sim \pi_e(a'|s')}[q_i(s',a')]| \\
	& \quad + |w(s,a)| |-(q(s,a) - q_i(s,a)) + \gamma ( \rE_{a' \sim \pi_e(a'|s')}[q_i(s',a') - q(s',a')])| \\
	& =: I + II.
\end{align*}
We bound the terms $I$ and $II$ separately.
For the first term $I$, we have
\begin{align*}
	I&\leq |w_j(s,a) - w(s,a)| (|q_i(s,a)| + |\gamma  \rE_{a' \sim \pi_e(a'|s')}[q_i(s',a')]||) \\
	& \leq \|w_j - w\|_\infty ( \|q_i\|_\infty + \gamma \|q_i\|_\infty) \leq \varepsilon_2  (1+\gamma) C_\Qbbb.
\end{align*}
For the second term $II$, we have
\begin{align*}
	II &\leq \|w\|_\infty (|q(s,a) - q_i(s,a)| + \gamma | \rE_{a' \sim \pi_e(a'|s')}[q_i(s',a') - q(s',a')]| ) \\
	&\leq C_\Wbbb (\|q-q_i\|_\infty + \gamma \|q-q_i\|_\infty) \leq C_\Wbbb (1+\gamma) \varepsilon_1.
\end{align*}

Combining the bound, we find that $\tilde{\Gcal}$ is a $(\varepsilon_2  (1+\gamma) C_\Qbbb + C_\Wbbb (1+\gamma) \varepsilon_1)$-covering set of $\Gcal$.
That is, we have
\begin{align*}
	\bN(\varepsilon_2 ( (1+\gamma) C_\Qbbb) + C_\Wbbb (1+\gamma) \varepsilon_1, \Gcal, \|\cdot\|_\infty) \leq N_Q N_W = \bN(\varepsilon_1, \Qbbb, \|\cdot\|_\infty)\bN(\varepsilon_2, \Wbbb, \|\cdot\|_\infty).
\end{align*}
Substituting $\varepsilon_2 = 0.5\varepsilon/ ((1+\gamma) C_\Qbbb)$ and $\varepsilon_1 = 0.5\varepsilon/(C_\Wbbb (1+\gamma))$ yields the statement.
\end{proof}

We consider using neutral networks with an ReLU activation function. The covering number of these neural networks is calculated as follows. 
\begin{lemma}[Neural network; Lemma 21 in \cite{JMLR:v21:20-002}] \label{lem:covering_neural} 
    Let $\Fcal_{NN}$ be a set of functions by a neural network with $L$ layers, $\Omega$ weights in $[-B,B]$, and $1-$Lipschitz continuous activation function 
    Then, for $\tau \in (0,1]$, we have
    \begin{align*}
        \log \Ncal(\tau, \Fcal_{NN}, \|\cdot\|_\infty) \leq \Omega \log \left( \frac{ 2LB^L (\Omega+1)^L}{\tau} \right). 
    \end{align*}
\end{lemma}
Besides, it is known that these  neural networks approximate the Sobolev ball. 
\begin{lemma}[Theorem 1 in \cite{yarotsky2017error}] \label{lem:yaro}
    Suppose $q_0$ belongs to the Sobolev ball with an order $p$ with its input is $d$-dimensional.
    Then, there exists the neural network with a ReLU activation function, which has at most $L=O(\log \Omega)$ layers and $\Omega$ parameters s.t. 
    \begin{align*}
        \|q^* - q_0\|_\infty =\tilde{O}( \Omega^{-p/ d}).
    \end{align*}
\end{lemma}

For a matrix $A$, let $\|A\|_F:= \sqrt{\sum_{i,j} A_{i,j}^2}$ be its Frobenius norm. 
\begin{lemma}[Theorem 1 in \cite{golowich2018size}] \label{lem:golowich}
    Let $\mathcal{F}_{NN}$ be a set of functions by a neural network with $L$ layers and activation function which is $1$-Lipschitz continuous and positive-homogeneous, that is, $\sigma(az) = a \sigma(z)$ holds with $a > 0$.
    Further, let $A_\ell$ be a weight matrix of an $\ell$-th layers for $\ell = 1,...,L$ and assume that $\|A_\ell\|_{F} \leq M_F(\ell)$ with a bound $M_F(\ell)$.
    Then, we have $\mathcal{R}_n(\mathcal{F}_{NN}) \leq \max_x\|x\|_2 (2\sqrt{\log L} + 1) \prod_{\ell=1}^L M_F(\ell)/\sqrt{n}$.
\end{lemma}



\end{document}